%% file: neurips_2024.tex
\documentclass{article}

    \usepackage[final]{neurips_2024}

\usepackage[utf8]{inputenc} 
\usepackage[T1]{fontenc}    
\usepackage{hyperref}       
\usepackage{url}            
\usepackage{booktabs}       
\usepackage{amsfonts}       
\usepackage{nicefrac}       
\usepackage{microtype}      
\usepackage{xcolor}         

\usepackage{amsmath}
\usepackage{amssymb}
\usepackage{mathtools}
\usepackage{amsthm}
\usepackage{algorithm}
\usepackage{algorithmic}
\usepackage{wrapfig}

\newtheorem{theorem}{Theorem}[section]

\newtheorem{lemma}[theorem]{Lemma}

\newtheorem{assumption}[theorem]{Assumption}

\newcommand{\adagrad}{{AdaGrad}}
\newcommand{\adam}{{Adam}}
\newcommand{\adamw}{{AdamW}}
\newcommand{\adamax}{{Adamax}}
\newcommand{\rmsprop}{{RMSprop}}
\newcommand{\amsgrad}{{AMSGrad}}
\newcommand{\adashift}{{AdaShift}}
\newcommand{\adopt}{{ADOPT}}

\input{math_commands}

\title{ADOPT: Modified Adam Can Converge with Any $\beta_2$\\with the Optimal Rate}

\author{%
  Shohei Taniguchi \\
  The University of Tokyo \\
  \texttt{taniguchi@weblab.t.u-tokyo.ac.jp} \\
  \And
  Keno Harada \\
  The University of Tokyo \\
  \texttt{keno.harada@weblab.t.u-tokyo.ac.jp} \\
  \And
  Gouki Minegishi \\
  The University of Tokyo \\
  \texttt{minegishi@weblab.t.u-tokyo.ac.jp} \\
  \And
  Yuta Oshima \\
  The University of Tokyo \\
  \texttt{yuta.oshima@weblab.t.u-tokyo.ac.jp} \\
  \And
  Seong Cheol Jeong \\
  The University of Tokyo \\
  \texttt{jeong@weblab.t.u-tokyo.ac.jp} \\
  \And
  Go Nagahara \\
  The University of Tokyo \\
  \texttt{nagaharago@weblab.t.u-tokyo.ac.jp} \\
  \And
  Tomoshi Iiyama \\
  The University of Tokyo \\
  \texttt{iiyama@weblab.t.u-tokyo.ac.jp} \\
  \And
  Masahiro Suzuki \\
  The University of Tokyo \\
  \texttt{masa@weblab.t.u-tokyo.ac.jp} \\
  \And
  Yusuke Iwasawa \\
  The University of Tokyo \\
  \texttt{iwasawa@weblab.t.u-tokyo.ac.jp} \\
  \And
  Yutaka Matsuo \\
  The University of Tokyo \\
  \texttt{matsuo@weblab.t.u-tokyo.ac.jp} \\
}

\begin{document}

\maketitle

\begin{abstract}
  Adam is one of the most popular optimization algorithms in deep learning. However, it is known that Adam does not converge in theory unless choosing a hyperparameter, i.e., $\beta_2$, in a problem-dependent manner. There have been many attempts to fix the non-convergence (e.g., AMSGrad), but they require an impractical assumption that the gradient noise is uniformly bounded. In this paper, we propose a new adaptive gradient method named ADOPT, which achieves the optimal convergence rate of $\mathcal{O} ( 1 / \sqrt{T} )$ with any choice of $\beta_2$ without depending on the bounded noise assumption. ADOPT addresses the non-convergence issue of Adam by removing the current gradient from the second moment estimate and changing the order of the momentum update and the normalization by the second moment estimate. We also conduct intensive numerical experiments, and verify that our ADOPT achieves superior results compared to Adam and its variants across a wide range of tasks, including image classification, generative modeling, natural language processing, and deep reinforcement learning. The implementation is available at \url{https://github.com/iShohei220/adopt}.
\end{abstract}

\input{1_introduction}
\input{2_preliminary}
\input{3_analysis}
\input{4_method}
\input{5_experiment}
\input{6_conclusion}

\bibliographystyle{unsrtnat}
\bibliography{neurips_2024}

\newpage
\appendix

\input{A_related_work}
\input{B_replacement}
\input{C_another}
\input{D_hyperparam}
\input{E_theorem}
\input{F_proof}
\input{G_setup}


\newpage
\section*{NeurIPS Paper Checklist}

\begin{enumerate}

\item {\bf Claims}
    \item[] Question: Do the main claims made in the abstract and introduction accurately reflect the paper's contributions and scope?
    \item[] Answer: \answerYes{} 
    \item[] Justification: Our main contribution is to demystify the cause of non-convergence of Adam, which is clearly written in the abstract and introduction.
    \item[] Guidelines:
    \begin{itemize}
        \item The answer NA means that the abstract and introduction do not include the claims made in the paper.
        \item The abstract and/or introduction should clearly state the claims made, including the contributions made in the paper and important assumptions and limitations. A No or NA answer to this question will not be perceived well by the reviewers. 
        \item The claims made should match theoretical and experimental results, and reflect how much the results can be expected to generalize to other settings. 
        \item It is fine to include aspirational goals as motivation as long as it is clear that these goals are not attained by the paper. 
    \end{itemize}

\item {\bf Limitations}
    \item[] Question: Does the paper discuss the limitations of the work performed by the authors?
    \item[] Answer: \answerYes{} 
    \item[] Justification: Limitations are described in the last section.
    \item[] Guidelines:
    \begin{itemize}
        \item The answer NA means that the paper has no limitation while the answer No means that the paper has limitations, but those are not discussed in the paper. 
        \item The authors are encouraged to create a separate "Limitations" section in their paper.
        \item The paper should point out any strong assumptions and how robust the results are to violations of these assumptions (e.g., independence assumptions, noiseless settings, model well-specification, asymptotic approximations only holding locally). The authors should reflect on how these assumptions might be violated in practice and what the implications would be.
        \item The authors should reflect on the scope of the claims made, e.g., if the approach was only tested on a few datasets or with a few runs. In general, empirical results often depend on implicit assumptions, which should be articulated.
        \item The authors should reflect on the factors that influence the performance of the approach. For example, a facial recognition algorithm may perform poorly when image resolution is low or images are taken in low lighting. Or a speech-to-text system might not be used reliably to provide closed captions for online lectures because it fails to handle technical jargon.
        \item The authors should discuss the computational efficiency of the proposed algorithms and how they scale with dataset size.
        \item If applicable, the authors should discuss possible limitations of their approach to address problems of privacy and fairness.
        \item While the authors might fear that complete honesty about limitations might be used by reviewers as grounds for rejection, a worse outcome might be that reviewers discover limitations that aren't acknowledged in the paper. The authors should use their best judgment and recognize that individual actions in favor of transparency play an important role in developing norms that preserve the integrity of the community. Reviewers will be specifically instructed to not penalize honesty concerning limitations.
    \end{itemize}

\item {\bf Theory Assumptions and Proofs}
    \item[] Question: For each theoretical result, does the paper provide the full set of assumptions and a complete (and correct) proof?
    \item[] Answer: \answerYes{} 
    \item[] Justification: Assumptions and proofs are provided in Section 2 and the appendix.
    \item[] Guidelines:
    \begin{itemize}
        \item The answer NA means that the paper does not include theoretical results. 
        \item All the theorems, formulas, and proofs in the paper should be numbered and cross-referenced.
        \item All assumptions should be clearly stated or referenced in the statement of any theorems.
        \item The proofs can either appear in the main paper or the supplemental material, but if they appear in the supplemental material, the authors are encouraged to provide a short proof sketch to provide intuition. 
        \item Inversely, any informal proof provided in the core of the paper should be complemented by formal proofs provided in appendix or supplemental material.
        \item Theorems and Lemmas that the proof relies upon should be properly referenced. 
    \end{itemize}

    \item {\bf Experimental Result Reproducibility}
    \item[] Question: Does the paper fully disclose all the information needed to reproduce the main experimental results of the paper to the extent that it affects the main claims and/or conclusions of the paper (regardless of whether the code and data are provided or not)?
    \item[] Answer: \answerYes{} 
    \item[] Justification: Experimental settings are provided in Section 5 and the appendix. We also share the implementation of the experiment.
    \item[] Guidelines:
    \begin{itemize}
        \item The answer NA means that the paper does not include experiments.
        \item If the paper includes experiments, a No answer to this question will not be perceived well by the reviewers: Making the paper reproducible is important, regardless of whether the code and data are provided or not.
        \item If the contribution is a dataset and/or model, the authors should describe the steps taken to make their results reproducible or verifiable. 
        \item Depending on the contribution, reproducibility can be accomplished in various ways. For example, if the contribution is a novel architecture, describing the architecture fully might suffice, or if the contribution is a specific model and empirical evaluation, it may be necessary to either make it possible for others to replicate the model with the same dataset, or provide access to the model. In general. releasing code and data is often one good way to accomplish this, but reproducibility can also be provided via detailed instructions for how to replicate the results, access to a hosted model (e.g., in the case of a large language model), releasing of a model checkpoint, or other means that are appropriate to the research performed.
        \item While NeurIPS does not require releasing code, the conference does require all submissions to provide some reasonable avenue for reproducibility, which may depend on the nature of the contribution. For example
        \begin{enumerate}
            \item If the contribution is primarily a new algorithm, the paper should make it clear how to reproduce that algorithm.
            \item If the contribution is primarily a new model architecture, the paper should describe the architecture clearly and fully.
            \item If the contribution is a new model (e.g., a large language model), then there should either be a way to access this model for reproducing the results or a way to reproduce the model (e.g., with an open-source dataset or instructions for how to construct the dataset).
            \item We recognize that reproducibility may be tricky in some cases, in which case authors are welcome to describe the particular way they provide for reproducibility. In the case of closed-source models, it may be that access to the model is limited in some way (e.g., to registered users), but it should be possible for other researchers to have some path to reproducing or verifying the results.
        \end{enumerate}
    \end{itemize}

\item {\bf Open access to data and code}
    \item[] Question: Does the paper provide open access to the data and code, with sufficient instructions to faithfully reproduce the main experimental results, as described in supplemental material?
    \item[] Answer: \answerYes{} 
    \item[] Justification: Data and code are provided in the appendix.
    \item[] Guidelines:
    \begin{itemize}
        \item The answer NA means that paper does not include experiments requiring code.
        \item Please see the NeurIPS code and data submission guidelines (\url{https://nips.cc/public/guides/CodeSubmissionPolicy}) for more details.
        \item While we encourage the release of code and data, we understand that this might not be possible, so “No” is an acceptable answer. Papers cannot be rejected simply for not including code, unless this is central to the contribution (e.g., for a new open-source benchmark).
        \item The instructions should contain the exact command and environment needed to run to reproduce the results. See the NeurIPS code and data submission guidelines (\url{https://nips.cc/public/guides/CodeSubmissionPolicy}) for more details.
        \item The authors should provide instructions on data access and preparation, including how to access the raw data, preprocessed data, intermediate data, and generated data, etc.
        \item The authors should provide scripts to reproduce all experimental results for the new proposed method and baselines. If only a subset of experiments are reproducible, they should state which ones are omitted from the script and why.
        \item At submission time, to preserve anonymity, the authors should release anonymized versions (if applicable).
        \item Providing as much information as possible in supplemental material (appended to the paper) is recommended, but including URLs to data and code is permitted.
    \end{itemize}

\item {\bf Experimental Setting/Details}
    \item[] Question: Does the paper specify all the training and test details (e.g., data splits, hyperparameters, how they were chosen, type of optimizer, etc.) necessary to understand the results?
    \item[] Answer: \answerYes{} 
    \item[] Justification: Detailed experimental settings are provided in Section 5 and the appendix.
    \item[] Guidelines:
    \begin{itemize}
        \item The answer NA means that the paper does not include experiments.
        \item The experimental setting should be presented in the core of the paper to a level of detail that is necessary to appreciate the results and make sense of them.
        \item The full details can be provided either with the code, in appendix, or as supplemental material.
    \end{itemize}

\item {\bf Experiment Statistical Significance}
    \item[] Question: Does the paper report error bars suitably and correctly defined or other appropriate information about the statistical significance of the experiments?
    \item[] Answer: \answerYes{} 
    \item[] Justification: Error bars are reported in all the figures and tables.
    \item[] Guidelines:
    \begin{itemize}
        \item The answer NA means that the paper does not include experiments.
        \item The authors should answer "Yes" if the results are accompanied by error bars, confidence intervals, or statistical significance tests, at least for the experiments that support the main claims of the paper.
        \item The factors of variability that the error bars are capturing should be clearly stated (for example, train/test split, initialization, random drawing of some parameter, or overall run with given experimental conditions).
        \item The method for calculating the error bars should be explained (closed form formula, call to a library function, bootstrap, etc.)
        \item The assumptions made should be given (e.g., Normally distributed errors).
        \item It should be clear whether the error bar is the standard deviation or the standard error of the mean.
        \item It is OK to report 1-sigma error bars, but one should state it. The authors should preferably report a 2-sigma error bar than state that they have a 96\% CI, if the hypothesis of Normality of errors is not verified.
        \item For asymmetric distributions, the authors should be careful not to show in tables or figures symmetric error bars that would yield results that are out of range (e.g. negative error rates).
        \item If error bars are reported in tables or plots, The authors should explain in the text how they were calculated and reference the corresponding figures or tables in the text.
    \end{itemize}

\item {\bf Experiments Compute Resources}
    \item[] Question: For each experiment, does the paper provide sufficient information on the computer resources (type of compute workers, memory, time of execution) needed to reproduce the experiments?
    \item[] Answer: \answerYes{} 
    \item[] Justification: Computational resources used in our experiments are reported in the appendix.
    \item[] Guidelines:
    \begin{itemize}
        \item The answer NA means that the paper does not include experiments.
        \item The paper should indicate the type of compute workers CPU or GPU, internal cluster, or cloud provider, including relevant memory and storage.
        \item The paper should provide the amount of compute required for each of the individual experimental runs as well as estimate the total compute. 
        \item The paper should disclose whether the full research project required more compute than the experiments reported in the paper (e.g., preliminary or failed experiments that didn't make it into the paper). 
    \end{itemize}
    
\item {\bf Code Of Ethics}
    \item[] Question: Does the research conducted in the paper conform, in every respect, with the NeurIPS Code of Ethics \url{https://neurips.cc/public/EthicsGuidelines}?
    \item[] Answer: \answerYes{} 
    \item[] Justification: We confirmed that our research conforms with the Code of Ethics.
    \item[] Guidelines:
    \begin{itemize}
        \item The answer NA means that the authors have not reviewed the NeurIPS Code of Ethics.
        \item If the authors answer No, they should explain the special circumstances that require a deviation from the Code of Ethics.
        \item The authors should make sure to preserve anonymity (e.g., if there is a special consideration due to laws or regulations in their jurisdiction).
    \end{itemize}

\item {\bf Broader Impacts}
    \item[] Question: Does the paper discuss both potential positive societal impacts and negative societal impacts of the work performed?
    \item[] Answer: \answerYes{} 
    \item[] Justification: We discussed them in the last section of the paper.
    \item[] Guidelines:
    \begin{itemize}
        \item The answer NA means that there is no societal impact of the work performed.
        \item If the authors answer NA or No, they should explain why their work has no societal impact or why the paper does not address societal impact.
        \item Examples of negative societal impacts include potential malicious or unintended uses (e.g., disinformation, generating fake profiles, surveillance), fairness considerations (e.g., deployment of technologies that could make decisions that unfairly impact specific groups), privacy considerations, and security considerations.
        \item The conference expects that many papers will be foundational research and not tied to particular applications, let alone deployments. However, if there is a direct path to any negative applications, the authors should point it out. For example, it is legitimate to point out that an improvement in the quality of generative models could be used to generate deepfakes for disinformation. On the other hand, it is not needed to point out that a generic algorithm for optimizing neural networks could enable people to train models that generate Deepfakes faster.
        \item The authors should consider possible harms that could arise when the technology is being used as intended and functioning correctly, harms that could arise when the technology is being used as intended but gives incorrect results, and harms following from (intentional or unintentional) misuse of the technology.
        \item If there are negative societal impacts, the authors could also discuss possible mitigation strategies (e.g., gated release of models, providing defenses in addition to attacks, mechanisms for monitoring misuse, mechanisms to monitor how a system learns from feedback over time, improving the efficiency and accessibility of ML).
    \end{itemize}
    
\item {\bf Safeguards}
    \item[] Question: Does the paper describe safeguards that have been put in place for responsible release of data or models that have a high risk for misuse (e.g., pretrained language models, image generators, or scraped datasets)?
    \item[] Answer: \answerNA{} 
    \item[] Justification:
    \item[] Guidelines:
    \begin{itemize}
        \item The answer NA means that the paper poses no such risks.
        \item Released models that have a high risk for misuse or dual-use should be released with necessary safeguards to allow for controlled use of the model, for example by requiring that users adhere to usage guidelines or restrictions to access the model or implementing safety filters. 
        \item Datasets that have been scraped from the Internet could pose safety risks. The authors should describe how they avoided releasing unsafe images.
        \item We recognize that providing effective safeguards is challenging, and many papers do not require this, but we encourage authors to take this into account and make a best faith effort.
    \end{itemize}

\item {\bf Licenses for existing assets}
    \item[] Question: Are the creators or original owners of assets (e.g., code, data, models), used in the paper, properly credited and are the license and terms of use explicitly mentioned and properly respected?
    \item[] Answer: \answerYes{} 
    \item[] Justification: They are provided both in the main paper and the appendix.
    \item[] Guidelines:
    \begin{itemize}
        \item The answer NA means that the paper does not use existing assets.
        \item The authors should cite the original paper that produced the code package or dataset.
        \item The authors should state which version of the asset is used and, if possible, include a URL.
        \item The name of the license (e.g., CC-BY 4.0) should be included for each asset.
        \item For scraped data from a particular source (e.g., website), the copyright and terms of service of that source should be provided.
        \item If assets are released, the license, copyright information, and terms of use in the package should be provided. For popular datasets, \url{paperswithcode.com/datasets} has curated licenses for some datasets. Their licensing guide can help determine the license of a dataset.
        \item For existing datasets that are re-packaged, both the original license and the license of the derived asset (if it has changed) should be provided.
        \item If this information is not available online, the authors are encouraged to reach out to the asset's creators.
    \end{itemize}

\item {\bf New Assets}
    \item[] Question: Are new assets introduced in the paper well documented and is the documentation provided alongside the assets?
    \item[] Answer: \answerNA{} 
    \item[] Justification:
    \item[] Guidelines:
    \begin{itemize}
        \item The answer NA means that the paper does not release new assets.
        \item Researchers should communicate the details of the dataset/code/model as part of their submissions via structured templates. This includes details about training, license, limitations, etc. 
        \item The paper should discuss whether and how consent was obtained from people whose asset is used.
        \item At submission time, remember to anonymize your assets (if applicable). You can either create an anonymized URL or include an anonymized zip file.
    \end{itemize}

\item {\bf Crowdsourcing and Research with Human Subjects}
    \item[] Question: For crowdsourcing experiments and research with human subjects, does the paper include the full text of instructions given to participants and screenshots, if applicable, as well as details about compensation (if any)? 
    \item[] Answer: \answerNA{} 
    \item[] Justification:
    \item[] Guidelines:
    \begin{itemize}
        \item The answer NA means that the paper does not involve crowdsourcing nor research with human subjects.
        \item Including this information in the supplemental material is fine, but if the main contribution of the paper involves human subjects, then as much detail as possible should be included in the main paper. 
        \item According to the NeurIPS Code of Ethics, workers involved in data collection, curation, or other labor should be paid at least the minimum wage in the country of the data collector. 
    \end{itemize}

\item {\bf Institutional Review Board (IRB) Approvals or Equivalent for Research with Human Subjects}
    \item[] Question: Does the paper describe potential risks incurred by study participants, whether such risks were disclosed to the subjects, and whether Institutional Review Board (IRB) approvals (or an equivalent approval/review based on the requirements of your country or institution) were obtained?
    \item[] Answer: \answerNA{} 
    \item[] Justification:
    \item[] Guidelines:
    \begin{itemize}
        \item The answer NA means that the paper does not involve crowdsourcing nor research with human subjects.
        \item Depending on the country in which research is conducted, IRB approval (or equivalent) may be required for any human subjects research. If you obtained IRB approval, you should clearly state this in the paper. 
        \item We recognize that the procedures for this may vary significantly between institutions and locations, and we expect authors to adhere to the NeurIPS Code of Ethics and the guidelines for their institution. 
        \item For initial submissions, do not include any information that would break anonymity (if applicable), such as the institution conducting the review.
    \end{itemize}

\end{enumerate}

\end{document}

%% file: math_commands.tex
\usepackage{amsmath,amsfonts,bm}

\def\eqref#1{equation~\ref{#1}}

\def\1{\bm{1}}

\def\eps{{\epsilon}}

\def\vzero{{\bm{0}}}
\def\vone{{\bm{1}}}

\def\vtheta{{\bm{\theta}}}
\def\va{{\bm{a}}}

\def\vg{{\bm{g}}}

\def\vm{{\bm{m}}}

\def\vv{{\bm{v}}}

\def\vx{{\bm{x}}}
\def\vy{{\bm{y}}}

\def\vphi{{\bm{\phi}}}

\def\eva{{a}}

\DeclareMathAlphabet{\mathsfit}{\encodingdefault}{\sfdefault}{m}{sl}
\SetMathAlphabet{\mathsfit}{bold}{\encodingdefault}{\sfdefault}{bx}{n}



%% file: 1_introduction.tex
\section{Introduction}
\label{sec:intro}

Stochastic optimization algorithms, such as stochastic gradient descent (SGD), play a central role in deep learning.
In particular, adaptive gradient methods based on exponential moving averages, such as \adam{} \citep{kingma2014adam}, are widely used in practice.
Despite the empirical success, it is known that \adam{} does not converge in theory in general cases.
For example, \citet{j.2018on} show that Adam fails to converge to a correct solution in a simple example where the objective function at time $t$ is given as:
\begin{align}
    f_t \left( \theta \right) = 
    \begin{cases}
        C \theta, & \mathrm{for} \ t \ \mathrm{mod} \ 3 = 1 \\
        - \theta, & \mathrm{otherwise},
    \end{cases} \label{eq:online_toy}
\end{align}
where $C > 2$ and $\theta \in [-1, 1]$.
In this online optimization setting, \adam{} converges to a wrong solution (i.e., $\theta = 1$) instead of the true solution (i.e., $\theta = -1$) especially when the hyperparameter $\beta_2$ is set to a small value.
There have been several attempts to fix the non-convergent behavior of \adam{}~\citep{j.2018on,zou2019sufficient}. 
For example, \amsgrad{}~\citep{j.2018on} ensures the convergence for online convex optimization by making slight modifications to the \adam{} algorithm. 
Subsequent studies~\citep{chen2018on,zhou2018convergence} show that \amsgrad{} also converges to a stationary point for smooth nonconvex stochastic optimization problems. 
However, the convergence proofs rely on the assumption that the gradient noise is uniformly bounded.
This assumption is stronger than the one used for the analysis of vanilla SGD~\citep{ghadimi2013stochastic,bertsekas2000gradient,khaled2022better}, where the gradient {\it variance} is assumed to be uniformly bounded.
In fact, the bounded noise assumption is often violated in practice.
For example, when Gaussian noise is used in the gradient estimation (e.g., variational autoencoders~\citep{Kingma2014} and diffusion models~\citep{ho2020denoising,song2021scorebased}), the stochastic gradient is no longer bounded.

Concurrently, \citet{zhou2018adashift} analyze the non-convergence of Adam in the problem described in Eq. (\ref{eq:online_toy}) from the perspective of the correlation between the current gradient and the second moment estimate based on the exponential moving average.  
Specifically, they show that the non-convergence problem can be resolved by excluding the gradient of some recent steps from the calculation of the second moment estimate.
Based on the analysis, they propose \adashift{}, another variant of \adam{}.
However, their theoretical analysis is limited to a single online convex problem described in Eq. (\ref{eq:online_toy}), and the convergence of \adashift{} for general nonconvex problems is unclear.

More recently, some works have demonstrated that \adam{} can converge by choosing $\beta_2$ in a problem-dependent manner~\citep{shi2020rmsprop,zhang2022adam,wang2022provable,li2023convergence,wang2023closing}.
However, tuning $\beta_2$ for each specific problem is troublesome; hence developing algorithms with the problem-independent convergence guarantee is still important to safely apply adaptive gradient methods to a wide range of machine learning problems.

In this paper, we propose an alternative approach to addressing the non-convergence problem of \adam{} without relying on the choice of $\beta_2$ or strong assumptions such as the bounded noise assumption.
To derive our algorithm, we first examine the case without momentum, analyzing the convergence bound of \rmsprop{} for general smooth nonconvex optimization problems.
Through the analysis, we uncover the fundamental cause of non-convergence, which stems from the correlation between the second moment estimate and the current gradient. 
This finding aligns with the results demonstrated by \citet{zhou2018adashift} for online convex optimization.
This correlation can be easily eliminated by excluding the current gradient from the second moment estimate.

Subsequently, we extend our findings to the case where momentum is incorporated, as in \adam{}, and discover that the \adam{}-style momentum also contributes to non-convergence. 
To address it, we propose to change the order of the momentum update and the normalization by the second moment estimate. 
With this small adjustment, we successfully eliminate the non-convergence problem of \adam{} without relying on a specific hyperparameter choice and the bounded noise assumption.
We provide theoretical evidence demonstrating that our derived algorithm, named \adopt{}, can achieve convergence with the optimal rate of $\mathcal{O}(1/\sqrt{T})$ for smooth nonconvex optimization.

In our experiments, we begin by assessing the performance of \adopt{} in a toy example where \adam{} typically fails to converge depending on the choice of $\beta_2$.
This toy example is an extension of the one presented in Eq. (\ref{eq:online_toy}) by \citet{j.2018on}, but we consider a scenario where \amsgrad{} is also hard to converge due to the dependence on the bounded noise assumption.
Our results demonstrate that \adopt{} rapidly converges to the solution, while \adam{} fails to converge, and \amsgrad{} exhibits extremely slow convergence.
Next, we conduct an experiment using a simple multi-layer perceptron on the MNIST classification task to evaluate the performance of \adopt{} in nonconvex optimization. Our findings indicate that \adopt{} outperforms existing adaptive gradient methods, including \adam{}, \amsgrad{}, and \adashift{}.
Finally, we evaluate the performance of \adopt{} in various practical applications, such as image classification of CIFAR-10 and ImageNet using ResNet~\citep{He_2016_CVPR} and SwinTransformer~\citep{liu2021swin}, training of deep generative models (NVAE), fine-tuning of language models (LLaMA), and deep reinforcement learning for continuous control. 
Our empirical results demonstrate that \adopt{} achieves superior results over existing algorithms (e.g., \adam{}) in these practical applications. 

%% file: 2_preliminary.tex
\section{Preliminary}
\label{sec:preliminary}
\subsection{Problem Definition}
We consider the minimization of the objective function $f: \mathbb{R}^D \rightarrow \mathbb{R}$ with respect to the parameter $\vtheta \in \mathbb{R}^D$.
In this context, we focus on first-order stochastic optimization methods, where only the stochastic gradient $\vg$ is accessible.
As the objective $f$ can be nonconvex, the goal is to find a stationary point where $\nabla f \left( \vtheta \right) = 0$~\citep{blair1985problem,vavasis1995complexity}.
In order to analyze the convergence behavior of stochastic optimization algorithms, the following assumptions are commonly employed in the literature:
\begin{assumption}
\label{ass:bounded_objective}
The objective function $f (\vtheta)$ is lower-bounded, i.e., $f (\vtheta ) \geq f_\mathrm{inf} > - \infty$ for all $\vtheta$.
\end{assumption}

\begin{assumption}
\label{ass:unbiased}
The stochastic gradient $\vg_t$ is an unbiased estimator of the objective $f (\vtheta_{t-1})$, i.e., $\mathbb{E} [ \vg_t ] = \nabla f (\vtheta_{t-1} )$ for all $t \geq 1$.
\end{assumption}

\begin{assumption}
\label{ass:smooth}
The objective function is $L$-smooth on $\mathbb{R}^D$, i.e., there exists a constant $L > 0$ such that $\| \nabla f ( \vx ) - \nabla f ( \vy ) \| \leq L \| \vx - \vy \|$ for all $\vx, \vy \in \mathbb{R}^D$.
\end{assumption}

\begin{assumption}
\label{ass:bounded_variance}
Variance of the stochastic gradient is uniformly bounded , i.e., there exists a constant $\sigma > 0$ such that $\mathbb{E} [ \left\| {\vg_t} - \nabla f \left( \vtheta_{t-1} \right) \right\|^2 ] \leq \sigma^2$.
\end{assumption}
For the analysis of adaptive gradient methods (e.g., \adam{} and \adagrad{}), many of previous works~\citep{defossez2022a,li2019convergence,ward2020adagrad,zou2018weighted} use a little stronger assumption instead of Assumption \ref{ass:bounded_variance} for ease of proofs:
\begin{assumption}
\label{ass:second_moment}
The stochastic gradient has a finite second moment, i.e., there exists a constant $G > 0$ such that $\mathbb{E} [ \left\| {\vg_t} \right\|^2 ] \leq G^2$.
\end{assumption}
Assumption \ref{ass:second_moment} requires that the true gradient $\nabla f$ is also uniformly bounded in addition to the variance of the stochastic gradient $\vg$.
Moreover, the convergence proof of \amsgrad{} tends to rely on an even stronger assumption as follows~\citep{chen2018on,zhou2018convergence}.
\begin{assumption}
\label{ass:bounded_stochastic_gradient}
The stochastic gradient is uniformly upper-bounded, i.e., there exists a constant $G > 0$ such that $\left\| {\vg_t} \right\| \leq G$. 
\end{assumption}
In Assumption \ref{ass:bounded_stochastic_gradient}, the gradient noise $\xi_t \coloneqq \vg_t - \nabla f$ is assumed to be bounded almost surely in addition to the true graidient $\nabla f$.
Note that when Assumption \ref{ass:bounded_stochastic_gradient} holds, Assumption \ref{ass:second_moment} is automatically satisfied; hence, Assumption \ref{ass:bounded_stochastic_gradient} is a stronger assumption compared to Assumption \ref{ass:second_moment}.
In this paper, we adopt Assumptions \ref{ass:bounded_objective}, \ref{ass:unbiased}, \ref{ass:smooth} and \ref{ass:second_moment} for analysis, because one of our motivations is to address the omission of Assumption \ref{ass:bounded_stochastic_gradient}.
In the analysis, we derive the upper bound of $\min_{t} \{ \mathbb{E} [ \| \nabla f ( \vtheta_{t} ) ) \|^{4/3} ]^{3/2} \}$ to investigate the convergence rate of the stochastic optimization algorithms, which is commonly performed in the literature~\citep{defossez2022a,zou2019sufficient}.

\subsection{Review of Stochastic Optimization Algorithms for Nonconvex Objectives}
The convergence of the vanilla SGD have been studied extensively in previous works.
For smooth nonconvex functions, \citet{ghadimi2013stochastic} showed that SGD with a constant learning rate converges with an $\mathcal{O} ( 1 / \sqrt{T} )$ rate under Assumptions \ref{ass:bounded_objective}-\ref{ass:bounded_variance} by setting $\alpha_t = \alpha = \Theta ( 1 / \sqrt{T} )$, where $\alpha_t$ is a learning rate at the $t$-th step, and $T$ is a total number of parameter updates.
This convergence rate is known to be minimax optimal up to a constant~\citep{drori2020complexity}.
For the diminishing learning rate scheme, the convergence bound of $\mathcal{O} ( \log T / \sqrt{T} )$ is well-known for $\alpha_t = \alpha / \sqrt{t}$~\citep{ghadimi2013stochastic}.
Recently, \citet{wang2021convergence} have proved that SGD with $\alpha_t = \alpha / \sqrt{t}$ can also achieve the optimal rate $\mathcal{O} ( 1 / \sqrt{T} )$ by additionally assuming that the objective $f$ is upper-bounded.

While the vanilla SGD is still one of the most popular choices for stochastic optimization, adaptive gradient methods are dominantly used especially for deep learning.
In adaptive gradient methods, the parameter $\vtheta$ is updated additionally using the second moment estimate ${\vv}_t$ in the following form:
\begin{align}
    \vtheta_t = \vtheta_{t-1} - \alpha_{t} \frac{\vg_{t}} { \sqrt{ {\vv}_t + \epsilon^2 } } , \label{eq:no_momentum}
\end{align}
where $\eps$ is a small positive constant.
The division between vectors is applied in an element-wise manner, and the addition between a vector $\va$ and a scalar $b$ is defined as $(\va + b)_i \coloneqq \eva_i + b$.
In \adagrad{}~\citep{JMLR:v12:duchi11a}, ${\vv}_t$ is defined as ${\vv}_0 = \vzero$ and ${\vv}_t = {\vv}_{t-1} + \vg_{t} \odot \vg_{t}$.
In \rmsprop{}~\citep{rmsprop}, an exponential moving average is substituted for the simple summation, i.e., ${\vv}_t = \beta_2 {\vv}_{t-1} + ( 1 - \beta_2 ) \vg_{t} \odot \vg_{t}$, where $0 \leq \beta_2 < 1$.
\adam{}~\citep{kingma2014adam} uses momentum in addition to the second moment estimate to accelerate the convergence as follows:
\begin{align}
    &{\vm}_t = \beta_1 {\vm}_{t-1} + \left( 1 - \beta_1 \right) \vg_{t}, \label{eq:adam_momentum}\\
    &\vtheta_t = \vtheta_{t-1} - \alpha_{t} \frac{ { \vm }_t } { \sqrt{ {\vv}_t + \epsilon^2 } } , 
\end{align}
where ${\vm}_0 = \vzero$.
Here, we omit the bias correction technique used in the original paper for clarity.
Unfortunately, \rmsprop{} and \adam{} are not guaranteed to converge even in a simple convex optimization problem as demonstrated by \citet{j.2018on}, whereas \adagrad{} with a constant learning rate is known to converge with an $\mathcal{O} ( \log T / \sqrt{ T } )$ rate under Assupmtions \ref{ass:bounded_objective}-\ref{ass:smooth} and \ref{ass:second_moment} for smooth nonconvex cases~\citep{li2019convergence,ward2020adagrad,zou2018weighted,chen2018on,defossez2022a}.
Although the convergence of \adam{} can be assured by choosing $\beta_2$ in a problem-dependent manner~\citep{shi2020rmsprop,zhang2022adam,wang2022provable,li2023convergence,wang2023closing}, it is difficult to know the proper choice of $\beta_2$ for each problem before training.

To fix the non-convergence of \adam{} without depending on $\beta_2$, some researchers have proposed variants of \adam{}.
\citet{j.2018on} proposed \amsgrad{}, which substitute $\hat{\vv}_t$ for $\vv$ in Eq. (\ref{eq:adam_momentum}), where $\hat{\vv}_0 = \vzero$ and $\hat{\vv}_t = \max \left\{ \hat{\vv}_{t-1}, \vv_t \right\}$.
The idea behind \amsgrad{} is that the scaling factor $\sqrt{ \hat{\vv}_t + \epsilon^2 }$ should be non-decreasing to ensure the convergence.
After \citet{j.2018on} originally proved the convergence of \amsgrad{} for online convex optimization, \citet{chen2018on} showed that \amsgrad{} with $\alpha_t = \alpha / \sqrt{t}$ converges with $\mathcal{O} ( \log T / \sqrt{ T } )$ for nonconvex settings.
\citet{zhou2018convergence} also analyzed the convergence of \amsgrad{} for nonconvex optimization, and derived the convergence rate of $\mathcal{O} ( 1 / \sqrt{ T } )$ for a constant learning rate of $\alpha_t = \alpha = \Theta ( 1 / \sqrt{T} )$.
However, their results depend on Assumption \ref{ass:bounded_stochastic_gradient}, which is often violated in practice.
For example, variational autoencoders~\citep{Kingma2014} and diffusion models~\citep{ho2020denoising,song2021scorebased} are typical examples in which Assumption \ref{ass:bounded_stochastic_gradient} does not hold because they utilize unbounded Gaussian noise in the gradient estimation.
The cause of requirement for Assumption \ref{ass:bounded_stochastic_gradient} is the max operation in the definition of $\hat{\vv}_t$.
Since the max operation is convex, $\mathbb{E} [ \hat{\vv}_t ] \leq \max_t \{ \mathbb{E} [ \vv_t ] \}$ does not hold; hence Assumption \ref{ass:bounded_stochastic_gradient} is required to upper-bound $\mathbb{E} [ \hat{\vv}_t ]$ in their proofs.

\citet{zhou2018adashift} also tried to fix the non-convergent behavior of \adam{}.
Their proposed \adashift{} uses $\vv_{t-n}$ instead of $\vv_t$ for the second moment estimate, and calculate the momentum using the latest $n$ gradients as follows:
\begin{align}
    &{\vm}_t = \frac{\sum_{k=0}^{n-1} \beta_1^k \vg_{t-k} } { \sum_{k=0}^{n-1} \beta_1^k }, \label{eq:adashift} \\
    &\vtheta_t = \vtheta_{t-1} - \alpha_{t} \frac{{ \vm }_t} { \sqrt{ {\vv}_{t-n} + \epsilon^2 } }. 
\end{align}
In the original paper, some additional techniques (e.g., the block-wise adaptive learning rate) are used, but we omit them for clarity here.
Though they give theoretical analysis for a single online convex example, any convergence bounds are not provided for nonconvex cases.
More detailed discussion on existing analyses is provided in Appendix \ref{sec:detailed_related}.

%% file: 3_analysis.tex
\section{Analysis: Cause of Non-convergence of Adam and How to Fix It}
\label{sec:non-convergence_analysis}
In this section, to derive an algorithm that can converge with any $\beta_2$ without Assumption \ref{ass:bounded_stochastic_gradient}, we analyze the cause of non-convergence of \adam{}, and discuss how it can be eliminated.
To start from a simple case, we first analyze the case without momentum.
Subsequently, we extend it to the case with momentum and provide a way to fix the convergence issue of \adam{}.

\subsection{Case without Momentum}
We first analyze the convergence of \rmsprop{}, which corresponds to the no-momentum case of \adam{} when we omit the bias correction.
For \rmsprop{}, we derive the following convergence bound.

\begin{theorem}
\label{thm:rmsprop}
    Under Assumptions \ref{ass:bounded_objective}-\ref{ass:smooth} and \ref{ass:second_moment}, the following holds for the \rmsprop{} with a constant learning rate $\alpha_t = \alpha$:
    \begin{align}
        &\min_{t=1, \ldots, T} \left\{ \mathbb{E} \left[ \left\| \nabla f ( \vtheta_{t-1} ) ) \right\|^{4/3} \right]^{3/2} \right\} 
        \leq C_1 \left( \frac{ f_0 - f_\mathrm{inf} }{\alpha T} + \frac{ C_2 }{ T }  \log \left( 1 + \frac{ G^2 }{ \epsilon^2 } \right) - C_2 \log \beta_2 \right), \label{eq:bound_rmsprop}
    \end{align}
    where $C_1 = 2 \sqrt{ G^2 + \epsilon^2 }$, $C_2 = \frac{ \alpha D L }{ 2 \left( 1 - \beta_2 \right) } + \frac{2 D G}{ \sqrt{ 1 - \beta_2 } }$, and $f_0 = f \left( \vtheta_0 \right)$.

\begin{proof}[Sketch of proof]
By Assumption \ref{ass:smooth}, the following holds:
\begin{align}
    &\mathbb{E} \left[ f \left( \vtheta_{t} \right) \right]
    \leq \mathbb{E} \left[ f \left( \vtheta_{t-1} \right) + \frac{ \alpha^2 L }{ 2 } \left\| \frac{ \vg_t } { \sqrt{ {\vv}_t + \epsilon^2 } } \right\|^2 - \alpha \nabla f \left( \vtheta_{t-1} \right)^\top \left( \frac{ \vg_t } { \sqrt{ {\vv}_t + \epsilon^2 } } \right) \right] \label{eq:smooth_expectation}
\end{align}

Applying Lemmas \ref{lem:decompose_v} and \ref{lem:holder} in the appendix to this, the following inequality is derived:
\begin{align}
    &\mathbb{E} \left[ f \left( \vtheta_{t} \right) \right] \nonumber \\
    &\leq \mathbb{E} \left[ f \left( \vtheta_{t-1} \right) + \left( \frac{ \alpha^2 L }{ 2 } + 2 \alpha G \sqrt{ 1 - \beta_2 } \right) \left\| \frac{ \vg_t } { \sqrt{ {\vv}_t + \epsilon^2 } } \right\|^2 - \frac{ \alpha }{2} \nabla f \left( \vtheta_{t-1} \right)^\top \left( \frac{ \vg_t } { \sqrt{ \tilde{\vv}_{t} + \epsilon^2 } } \right) \right] \label{eq:decouple_v}\\
    &\leq \mathbb{E} \left[ f \left( \vtheta_{t-1} \right) + \left( \frac{ \alpha^2 L }{ 2 } + 2 \alpha G \sqrt{ 1 - \beta_2 } \right) \left\| \frac{ \vg_t } { \sqrt{ {\vv}_t + \epsilon^2 } } \right\|^2 \right] - \frac{ \alpha }{2} \frac{\mathbb{E} \left[ \left\| \nabla f \left( \vtheta_{t-1} \right) \right\|^{4/3} \right]^{3/2} } { \sqrt{ \left( 1 - \beta_2^T \right) G^2 + \epsilon^2 } } , \label{eq:make_norm_expectation}
\end{align}
where $\tilde{\vv}_t = \beta_2 \vv_{t-1} + ( 1 - \beta_2 ) \mathbb{E} [ \vg_t \odot \vg_t ]$.
Telescoping this for $t = 1, \ldots, T$ and rearranging the terms, we have
\begin{align}
    \sum_{t=1}^{T} \mathbb{E} \left[ \left\| \nabla f \left( \vtheta_{t-1} \right) \right\|^{4/3} \right]^{3/2}
    \leq C_1 \left( \frac{ f \left( \vtheta_0 \right) - f_\mathrm{inf} }{\alpha} + C_2 \log \left( \frac{ G^2 + \epsilon^2 }{ \beta_2^T \epsilon^2 } \right) \right),
\end{align}
where the last inequality holds due to Assumption \ref{ass:bounded_objective} and Lemma \ref{lem:telescope_v}.
Therefore, the bound in Eq. (\ref{eq:bound_rmsprop}) is derived using $\min_{t=1, \ldots, T} \{ \mathbb{E} [ \| \nabla f ( \vtheta_{t-1} ) ) \|^{4/3} ]^{3/2} \} \leq  \sum_{t=1}^{T} \mathbb{E} [ \| \nabla f \left( \vtheta_{t-1} \right) \|^{4/3} ]^{3/2} / T$.
\end{proof}
\end{theorem}
A detailed proof is provided in the appendix.
When the learning rate $\alpha$ is chosen so that $\alpha = \Theta ( 1 / \sqrt{T} )$, the first and second terms on the right hand side of Eq. (\ref{eq:bound_rmsprop}) converge with $\mathcal{O} ( 1 / \sqrt{T} )$ and $\mathcal{O} ( 1 / T )$ rates, respectively.
However, the last term includes a constant factor in terms of $T$, which represents the non-convergent behavior of \rmsprop{} in the smooth nonconvex setting.
More precisely, \rmsprop{} is guaranteed to converge only to a bounded region around a stationary point, and the size of the bounded region depends on the hyperparameter $\beta_2$ and the problem-dependent factors $D$, $G$, and $L$.
Therefore, we need to choose $\beta_2$ dependently on each problem to make the bounded region adequately small.
Since $\lim_{\beta_2 \to 1} \log \beta_2 / \sqrt{1 - \beta_2} = 0$, the size of the bounded region can be made small by setting $\beta_2$ to a value close to 1, which aligns with practical observations. 
However, how close to $1$ it should be relies on the problem-dependent factors, which cannot be observed in advance.
This result is consistent with recent results of convergence analyses of \adam{} and \rmsprop{}~\citep{shi2020rmsprop,zhang2022adam}.

As can be seen from Eqs. (\ref{eq:smooth_expectation}) and (\ref{eq:decouple_v}), the constant term in Eq. (\ref{eq:bound_rmsprop}) is derived from the last term of Eq. (\ref{eq:smooth_expectation}).
Because $\vg_t$ and $\vv_t$ are not statistically independent, this term is first decomposed as in Eq. (\ref{eq:decouple_v}).
After the decomposition, $\vg_t$ and $\tilde{\vv}_{t}$ is now conditionally independent given $\vg_0, \ldots, \vg_{t-1}$, so Eq. (\ref{eq:make_norm_expectation}) is derived using the following fact:
\begin{align}
    \mathbb{E} \left[ \frac{ \vg_t } { \sqrt{ \tilde{\vv}_{t} + \epsilon^2 } } \right] = \mathbb{E} \left[ \frac{ \nabla f \left( \vtheta_{t-1} \right) } { \sqrt{ \tilde{\vv}_{t} + \epsilon^2 } } \right].
\end{align}
This indicates that, if the second moment estimate $\vv_t$ is designed to be conditionally independent to $\vg_t$, the constant term in the convergence bound will be removed, because the second term of Eq. (\ref{eq:smooth_expectation}) can be directly lower-bounded without the decomposition.
A simple way to achieve the conditional independence is to substitute $\vv_{t-1}$ for $\vv_t$ as a second moment estimate, because $\vv_{t-1}$ does not have information about $\vg_t$.
This solution is similar to AdaShift, in which $\vv_{t-n}$ is substituted for $\vv_t$ as described in Eq. (\ref{eq:adashift}).
In fact, the modified version of \rmsprop{} is identical to AdaShift with $n=1$ and $\beta_1 = 0$ except for the additional techniques (e.g., the block-wise adaptive learning rate).

\subsection{Case with Momentum}
As we have described, \rmsprop{} can be modified to be convergent by removing the current gradient $\vg_t$ from the second moment estimate $\vv_t$.
However, when we combine adaptive gradient methods with momentum like \adam, the convergence analysis becomes more complicated.
Unfortunately, when \adam{}-style momentum in Eq. (\ref{eq:adam_momentum}) is applied, the algorithm does not converge in general even when using $\vv_{t-1}$ as a second moment estimate instead of $\vv_t$.
This is because the momentum $\vm_t$ contains all history of the past gradients $\vg_0 , \ldots, \vg_t$; hence the second moment estimate always correlates with $\vm_t$.
\adashift{} prevents this problem by calculating the momentum $\vm_t$ only using the latest $n$ gradients as described in Eq. (\ref{eq:adashift}).
In that case, the momentum $\vm_t$ and the second moment estimate $\vv_{t-n}$ are conditionally independent, so the convergence can be retained.
However, this approach has a trade-off in the choice of $n$.
When $n$ is small, $\vm_t$ has little information about the past gradients; when $n$ is large, $\vv_{t-n}$ only has access to the gradient information in the distant past.

To remove this trade-off, instead of truncating the momentum to the latest $n$ steps, we propose to use momentum of the following form:
\begin{align}
    &\vm_t = \beta_1 \vm_{t-1} + \left( 1 - \beta_1 \right) \frac{ \vg_{t} }{ \sqrt{ {\vv}_{t-1} + \epsilon^2 } }, \label{eq:adopt_momentum} \\
    &\vtheta_{t} = \vtheta_{t-1} - \alpha_{t} \vm_t . \label{eq:adopt_param}
\end{align}
The main difference to the \adam{}-style momentum in Eq. (\ref{eq:adam_momentum}) is the order of update of $\vm_t$ and the normalization by $\sqrt{ {\vv}_{t-1} + \epsilon^2 }$.
In Eq. (\ref{eq:adam_momentum}), the normalization is performed after the update of $\vm_t$, whereas in Eq. (\ref{eq:adopt_momentum}), the normalization is first applied to the current gradient $\vg_t$ in advance to the update of $\vm_t$.
In this case, the second moment estimate $\vv_{t-1}$ is only used to normalize the current gradient $\vg_t$, so the convergence can be guaranteed.
A more detailed convergence analysis is provided in Section \ref{sec:adopt}.

%% file: 4_method.tex
\section{Method: Adaptive Gradient Method with the Optimal Convergence Rate}
\label{sec:adopt}

\begin{algorithm}[t]
\caption{\adopt{} algorithm}
\label{alg:adopt}
\begin{algorithmic}
    \REQUIRE Learning rate $\left\{ \alpha_t \right\}$, initial parameter $\vtheta_0$
    \REQUIRE Exponential decay rate $0 \leq \beta_1, \beta_2, \beta_3 < 1$, small constant $\epsilon > 0$
    \STATE {$\vv_0 \leftarrow \vg_0 \odot \vg_0, \vm_1 \leftarrow \vg_1 / \max \left\{ \sqrt{ \vv_0 } , \epsilon \right\}$}
    \FOR {$t = 1$ to $T$}
        \STATE {$\vtheta_{t} \leftarrow \vtheta_{t-1} - \alpha_{t} \vm_t $}
        \STATE {$\vv_{t} \leftarrow \beta_2 \cdot \vv_{t-1} + \left( 1 - \beta_2 \right) \vg_{t} \odot \vg_{t}$}
        \STATE {{\color[HTML]{007AFF} $\vm_{t+1} \leftarrow \beta_1 \cdot \vm_{t} + \left( 1 - \beta_1 \right) \frac{ \vg_{t+1} }{ \max \left\{ \sqrt{ \vv_{t} } , \epsilon \right\} }$}}
    \ENDFOR
    \STATE {{\bf return} $\{ \vtheta_t \}_{t=1}^{T}$}
\end{algorithmic}
\end{algorithm}

Based on the analysis in the previous section, we propose a new adaptive gradient method named \adopt{} ({\it ADaptive gradient method with the OPTimal convergence rate}).
The entire procedure is summarized in Algorithm \ref{alg:adopt}.
For a simple discription, we place the update of $\vm$ after the parameter update in Algorithm \ref{alg:adopt}, but it is equivalent to Eqs. (\ref{eq:adopt_momentum}) and (\ref{eq:adopt_param}) except that \(\max \left\{ \sqrt{ \vv } , \epsilon \right\}\) is substitued for \( \sqrt{ \vv + \epsilon^2 } \).
The substitition is applied because we find that it contributes to slightly better performance in practice.
We provide an equivalent expression of Algorithm \ref{alg:adopt} in Algorithm \ref{alg:adopt_new} in the appendix, which is closer to a practical implementation.
By this modification, \adopt{} can converge with the optimal rate for smooth nonconvex optimization as follows:

\begin{algorithm}[t]
\caption{Clipped \adopt{} algorithm}
\label{alg:clipped_adopt}
\begin{algorithmic}
    \REQUIRE Learning rate $\left\{ \alpha_t \right\}$, clipping value $\left\{ c_t \right\}$, initial parameter $\vtheta_0$
    \REQUIRE Exponential decay rate $0 \leq \beta_1, \beta_2 < 1$, small constant $\epsilon > 0$
    \STATE {$\vm_0 \leftarrow \vzero, \vv_0 \leftarrow \vg_0 \odot \vg_0$}
    \FOR {$t = 1$ to $T$}
        \STATE {{\color[HTML]{007AFF} $\vm_{t} \leftarrow \beta_1 \cdot \vm_{t-1} + \left( 1 - \beta_1 \right) \mathrm{Clip} \left( \frac{ \vg_{t} }{ \max \left\{ \sqrt{ \vv_{t-1} } , \epsilon \right\} } , c_t \right)$}}
        \STATE {$\vtheta_{t} \leftarrow \vtheta_{t-1} - \alpha_{t} \vm_t $}
        \STATE {$\vv_{t} \leftarrow \beta_2 \cdot \vv_{t-1} + \left( 1 - \beta_2 \right) \vg_{t} \odot \vg_{t}$}
    \ENDFOR
    \STATE {{\bf return} $\{ \vtheta_t \}_{t=1}^{T}$}
\end{algorithmic}
\end{algorithm}

\begin{theorem}
\label{thm:adopt_constant}
Under Assumptions \ref{ass:bounded_objective}-\ref{ass:smooth} and \ref{ass:second_moment}, the following holds for the \adopt{} algorithm with a constant learning rate $\alpha_t=\alpha = \Theta \left( 1 / \sqrt{T} \right)$:
\begin{align}
    &\min_{t=1, \ldots, T} \left\{ \mathbb{E} \left[ \left\| \nabla f ( \vtheta_{t-1} ) ) \right\|^{4/3} \right]^{3/2} \right\}
    \leq \mathcal{O} \left( 1 / \sqrt{T} \right) ,  \label{eq:adopt_bound}
\end{align}
\end{theorem}

The detailed proof and related lemmas are provided in the appendix.
We also provide the convergence bound for the case of diminishing learning rate (i.e., $\alpha_t = \Theta ( 1 / \sqrt{t} )$) in the appendix, which is closer to practical situations.
In that case, \adopt{} also converges with the optimal rate.

In practice, however, the ADOPT algorithm sometimes becomes unstable especially when near-zero gradients are observed in the early phase of optimization.
For example, if some elements of \(\vv_0\) are almost zero, the corresponding elements of \(\vg_1 / \max \{ \sqrt{ \vv_0 } , \epsilon \}\) take very large values due to the near-zero division, leading to an unstable parameter update.
Such a phenomenon is typically observed when, for example, some parts of parameters (e.g., the last layer of a neural net) are initialized with zero, which is a commonly-used technique in deep learning.
To avoid the near-zero division, we also propose a clipped version of ADOPT in Algorithm \ref{alg:clipped_adopt}.
Note that the clipping operation is applied in an element-wise manner:
\begin{align}
    \mathrm{Clip} \left( \va , c \right)_i = \min \left\{ \max \left\{ \eva_i , - c \right\} , c \right\} ,
\end{align}
where \(c \geq 0\).
Even when the clipping is incorporated, the same convergence rate can be guaranteed by properly scheduling the clipping value $c_t$ (see Theorems \ref{thm:clipped_adopt_constant} and \ref{thm:clipped_adopt_deminishing} in the appendix).
Specifically, we can guarantee the convergence by setting \(c_t = c = \Theta \left( {T}^{1/4} \right) \) or \(c_t = \Theta \left( {t}^{1/4} \right) \).

%% file: 5_experiment.tex
\section{Experiments}
\label{sec:experiment}
\begin{figure*}[tbp]
    \centering
    \begin{minipage}{\linewidth}
        \centering
        $k = 10$
    \end{minipage}\\
    \begin{minipage}{0.325\linewidth}
        \includegraphics[width=\linewidth]{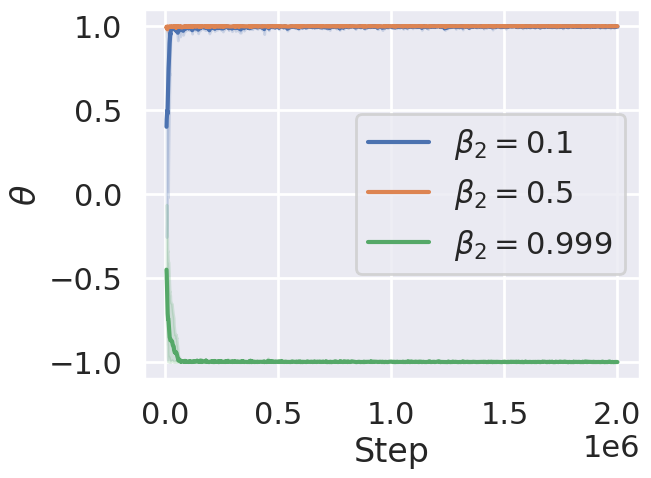}
    \end{minipage}
    \begin{minipage}{0.325\linewidth}
        \includegraphics[width=\linewidth]{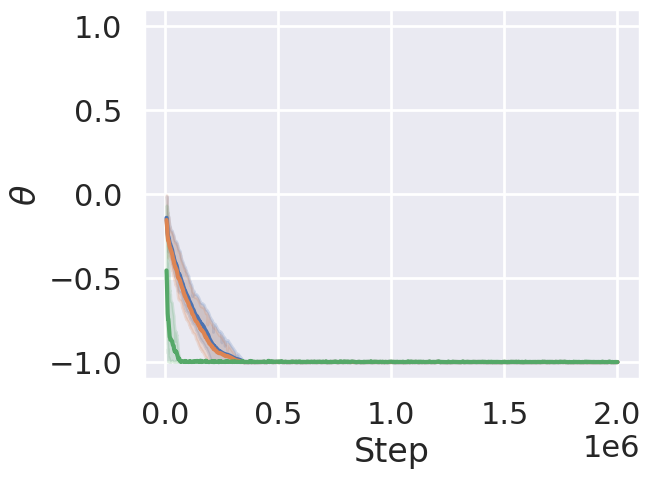}
    \end{minipage} 
    \begin{minipage}{0.325\linewidth}
        \includegraphics[width=\linewidth]{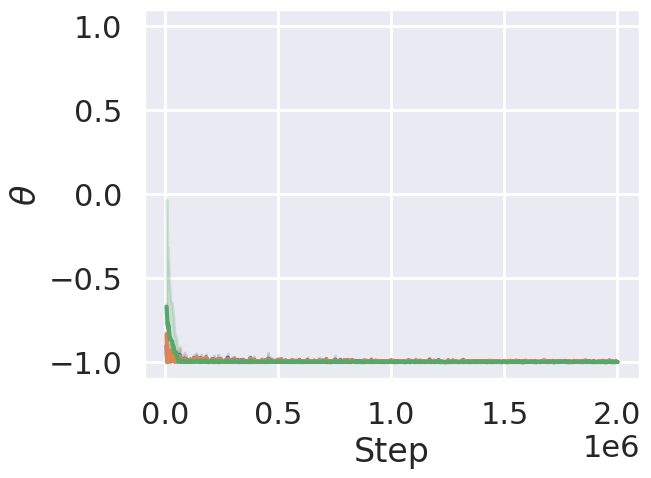}
    \end{minipage} \\
    \begin{minipage}{\linewidth}
    \end{minipage}\\
    \begin{minipage}{\linewidth}
        \centering
        $k = 50$
    \end{minipage}\\
    \begin{minipage}{0.325\linewidth}
        \includegraphics[width=\linewidth]{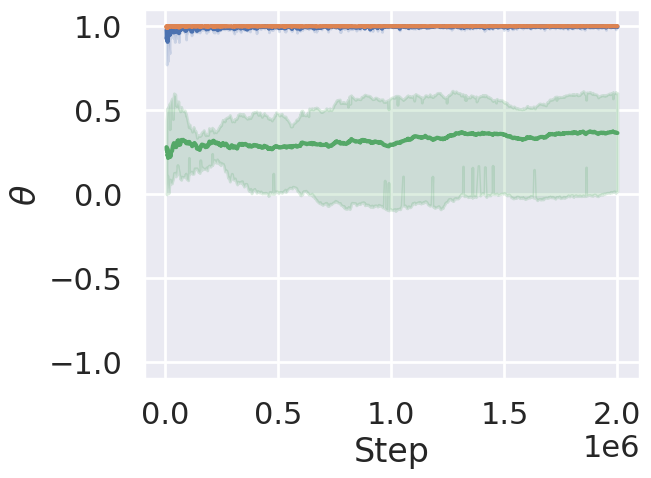}
    \end{minipage} 
    \begin{minipage}{0.325\linewidth}
        \includegraphics[width=\linewidth]{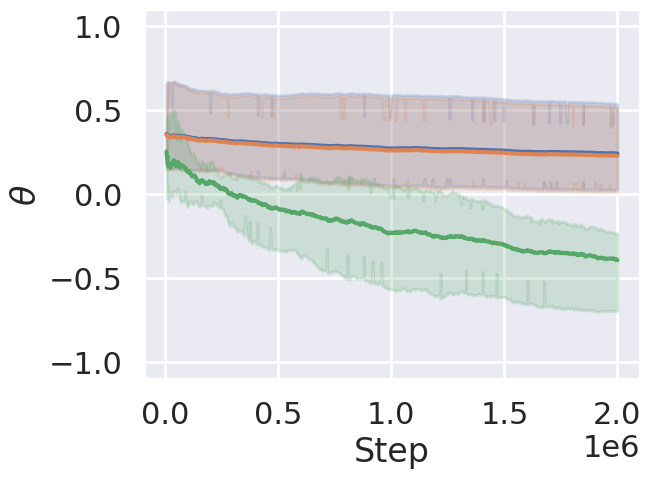}
    \end{minipage}
    \begin{minipage}{0.325\linewidth}
        \includegraphics[width=\linewidth]{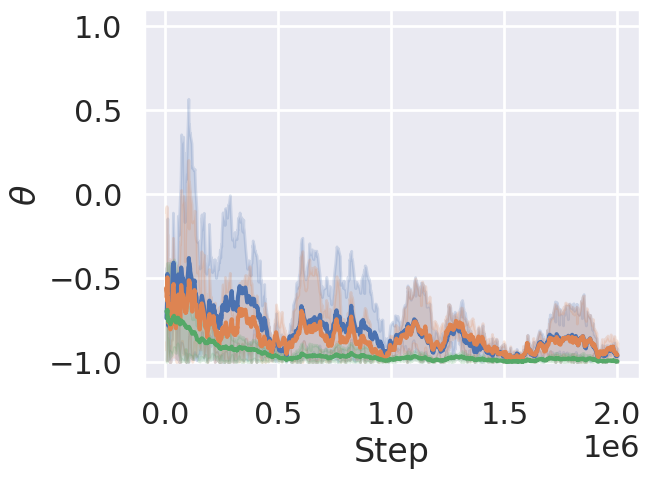}
    \end{minipage} \\
    \begin{minipage}{0.325\linewidth}
        \centering
        \adam{}
    \end{minipage} 
    \begin{minipage}{0.325\linewidth}
        \centering
        \amsgrad{}
    \end{minipage}
    \begin{minipage}{0.325\linewidth}
        \centering
        \adopt{}
    \end{minipage}
    \caption{Performance comparison between \adam, \amsgrad{} and \adopt{} in a simple univariate convex optimization problem. The plots show transitions of the parameter value, which should converge to the solution $\theta = -1$.}
    \label{fig:toy}
\end{figure*}
In the experiments, we first validate our \adopt{} algorithm using a simple toy example in which \adam{} is known to fail to converge, and confirm our theoretical findings through numerical simulation.
Secondly, we run an experiment of training a simple multi-layer perceptron (MLP) for the MNIST dataset to verify the effectiveness of our \adopt{} for nonconvex optimization problems.
Finally, we evaluate our \adopt{} in a wide range of practical applications, including image classification, natural language processing (NLP) tasks, generative modeling, and deep reinforcement learning.
Detailed experimental settings are described in the appendix.
In these experiments, we use the vanilla ADOPT algorithm in Algorithm \ref{alg:adopt} and do not use the clipped version in Algorithm \ref{alg:clipped_adopt}.

{\bf Toy problem:} We consider a convex optimization problem with an objective $f ( \theta ) = \theta$ for $\theta \in [ -1, 1 ]$.
It is obvious that a solution for the problem is $\theta = -1$.
Through the optimization, we only have access to the stochastic objective $f_t$ as follows:
\begin{align}
    &f_t \left( \theta \right) = 
    \begin{cases} 
    k^2 \theta, & \text { with probability } 1/k \\ 
    -k \theta, & \text { with probability } 1 - 1 / k
    \end{cases}, 
\end{align}
where $k \geq 1$.
Because $\mathbb{E} [ f_t ( \theta ) ] = f( \theta )$ holds, the stochastic gradient $g_t = \nabla f_t ( \theta )$ is an unbiased estimator of the true gradient $\nabla f$ regardless of the choice of $k$, satisfying Assumption \ref{ass:unbiased}.
This problem is equivalent, except for scaling, to the stochastic optimization version of Eq. (\ref{eq:online_toy}) provided by \citet{j.2018on} as a case where \adam{} fails to converge.
In this setting, the constant $k$ controls the magnitude of gradient noise.
When $k = 1$, it corresponds to the noiseless case where $f_t = f$ with probability $1$.
As $k$ gets large, stochastic gradient becomes noisy, making $G$ in Assumptions \ref{ass:second_moment} and \ref{ass:bounded_stochastic_gradient} large.
Therefore, the optimization will be more difficult when $k$ becomes larger.
In the experiment, we set $k = 10$ or $50$, and compare the robustness of \adam{}, \amsgrad{}, and \adopt{} for various hyperparameter settings by changing $\beta_2$ from $0.1 \sim 0.999$.
We set $\beta_1 = 0.9$ for all the algorithms, which is a common choice in practice.
We set the learning rate to $\alpha_t = 0.01 / \sqrt{1 + 0.01 t}$.

The result is shown in Figure \ref{fig:toy}.
It can be seen that, when $k = 10$, \adam{} fails to converge except for $\beta_2 = 0.999$ while \amsgrad{} and \adopt{} rapidly converge to the correct solution, i.e., $\theta = -1$, with any $\beta_2$.
In a more extreme case where $k = 50$, \adam{} fails to converge even with $\beta_2 = 0.999$.
This aligns with Theorem \ref{thm:rmsprop}, since, when the gradient noise is large (i.e., $G$ is large), the bounded region of the convergence bound also gets large, leading to divergence of \adam{}.
Moreover, when $k = 50$, it is observed that the convergence of \amsgrad{} also becomes much slower than \adopt{}.
In fact, this phenomenon is also consistent with theory.
In this problem setting, the second moment $\mathbb{E} [ g_t^2 ]$ is $\mathcal{O} ( k^3 )$, while the squared norm of the stochastic gradient $g_t^2$ is $\mathcal{O} ( k^4 )$.
Since the convergence bound of AMSGrad depends on the uniform bound of the stochastic gradient in Assumption \ref{ass:bounded_stochastic_gradient}, instead of the second moment in Assumption \ref{ass:second_moment}, its convergence also deteriorates with the order of $g_t^2$.
Compared to AMSGrad, ADOPT only depends on the second moment bound for its convergence, so it converges much faster than AMSGrad even in such an extreme setting.

We also perform ablation study on how the two algorithmic changes from \adam{} to \adopt{} affect the convergence.
The differences between \adam{} and \adopt{} are (1) decorrelation between the second moment estimate and the current gradient, and (2) change of order of momentum update and normalization by the second moment estimate.
In this experiment, we remove each algorithmic change from \adopt{}, and compare the result in the toy example.
We set $k = 50$, and $( \beta_1, \beta_2 ) = ( 0.9, 0.999 )$, since it is a common hyperparameter choice.
The result is shown in Figure \ref{fig:ablation}.
It can be observed that \adopt{} fails to converge with the exception of either algorithmic change.
Therefore, applying both changes is essential to overcome the non-convergence of \adam{}, which also aligns with theory.
These results correspond to the theoretical findings, showing the superiority of \adopt{} to \adam{} and \amsgrad{} in terms of the convergence speed and its robustness to hyperparameter choices.

{\bf MNIST classification:}
To investigate the effectiveness of \adopt{} on nonconvex optimization, we train nonlinear neural networks for MNIST classification tasks, and compare the performance between \adopt{} and existing optimization algorithms, such as \adam{}, \amsgrad{} and \adashift{}.
In this experiment, we use a simple MLP with a single hidden layer, and the number of hidden units is set to 784.
We set the learning rate to $\alpha_t = \alpha / \sqrt{t}$, and $\alpha$ is tuned in the range of $\{ 1, 10^{-1}, 10^{-2}, 10^{-3} \}$.
We apply weight decay of $1 \times 10^{-4}$ to prevent over-fitting, and run 10K iterations of parameter updates.
Figure \ref{fig:mnist} shows the learning curves of training and test accuracy.
We observe our \adopt{} performs slightly better than the others in terms of the convergence speed and the final performance.

{\bf Image classification:}
As a more practical application, we conduct experiments of image classification using real-world image datasets.
We first compare \adopt{} and \adam{} in the classification task of the CIFAR-10 dataset using ResNet-18~\citep{He_2016_CVPR}, a widely-used convolutional neural network.
We conduct a similar hyperparameter search to the case of MNIST classification.
A detailed experimental setting is provided in the appendix.
The learning curves of test accuracy are visualized in Figure \ref{fig:cifar10}.
It can be observed that \adopt{} converges a little faster than \adam{}.

\begin{figure}[tbp]
\label{}
    \centering
    \begin{minipage}{0.4\linewidth}
        \includegraphics[width=\linewidth]{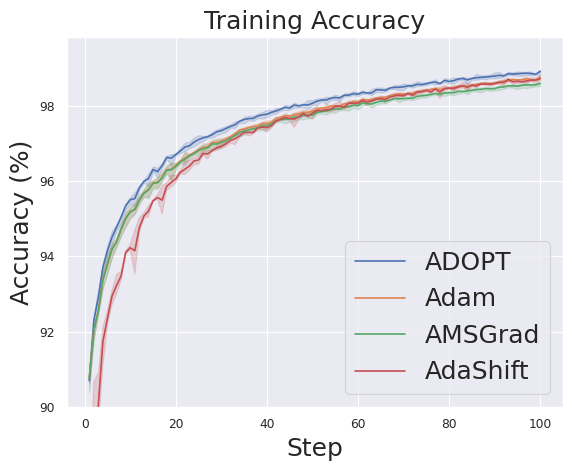}
    \end{minipage}
    \hspace{0.05\linewidth}
    \begin{minipage}{0.4\linewidth}
        \includegraphics[width=\linewidth]{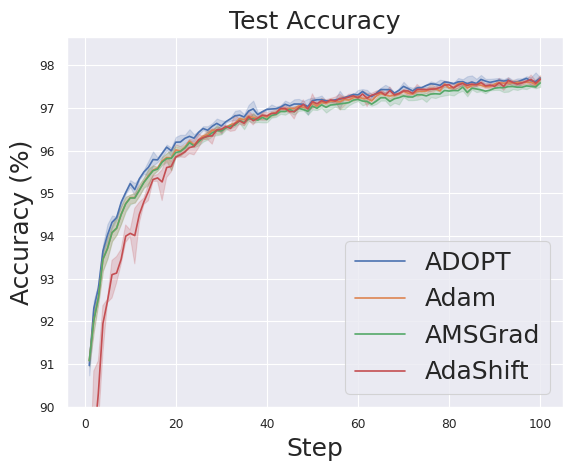}
    \end{minipage}
    \caption{Accuracy for training data (left) and test data(right) in MNIST classification. The error bars show the 95\% confidence intervals of three trials.}
    \label{fig:mnist}
\end{figure}

\begin{figure}[tbp]
    \centering
    \begin{minipage}{0.45\linewidth}
        \centering
        \vspace{\baselineskip}
        \includegraphics[width=\linewidth]{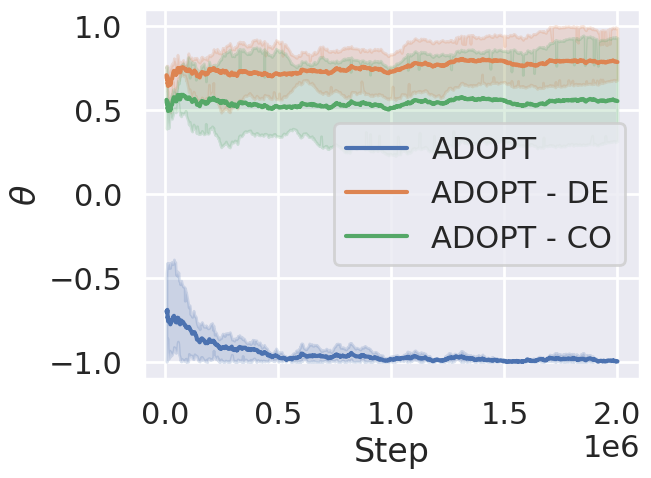}
        \vspace{-\baselineskip}
        \caption{Ablation study of algorithmic changes between \adam{} and \adopt{}. "DE" and CO denote "decorrelation" and "change of order", respectively.}
        \label{fig:ablation}
    \end{minipage}
    \hspace{0.05\linewidth}
    \begin{minipage}{0.45\linewidth}
        \centering
        \includegraphics[width=\linewidth]{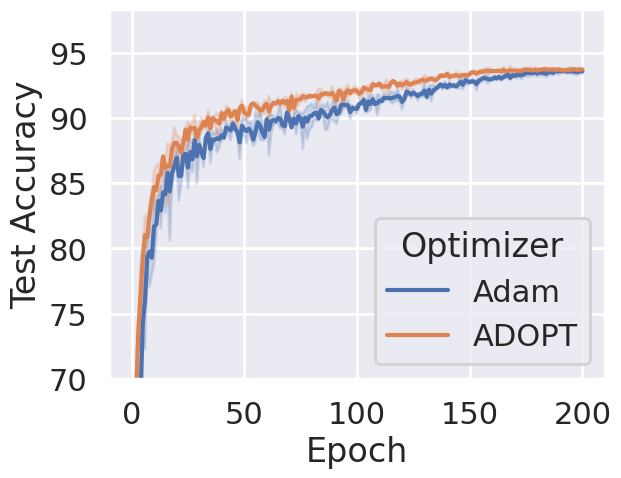}
        \vspace{-1.5\baselineskip}
        \caption{Learning curves of test accuracy for CIFAR-10 classification by ResNet-18 trained with \adam{} and \adopt{}.}
        \label{fig:cifar10}
    \end{minipage}
\end{figure}

To confirm that our \adopt{} works well for modern neural network architectures based on Transformers~\citep{NIPS2017_3f5ee243}, we perform an experiment of ImageNet classification using SwinTransformer~\citep{liu2021swin}.
We follow the official training recipe of Swin Transformer-tiny provided by Torchvision~\citep{paszke2019pytorch}, and fix the training settings except for the optimizer choice.
We use \adamw{}~\citep{loshchilov2018decoupled} as a baseline because it is set as the default official optimizer.
We also compare with \amsgrad{} as another way to fix the non-convergence issue of \adam{}.
Since \adamw{} uses decoupled weight decay, we also apply it to the other optimizers for fair comparison.
We report the top-1 accuracy at $200$ and $300$ epochs in Tables \ref{tab:imagenet}.
We observe that \adopt{} outperforms \adamw{} and \amsgrad{} throughout the training in terms of the test accuracy, demonstrating the effectiveness of \adopt{} for this setting.

\begin{table}[tb]
    \begin{minipage}{0.48\linewidth}
        \centering
        \caption{Top-1 accuracy (\%) for ImageNet classification by SwinTransformer.}
        \begin{tabular}{l|cc}
            \toprule
            Epoch       & $200$                     & $300$ \\ \midrule
            \adamw{}    & $79.29 \pm 0.05$          & $81.26 \pm 0.04$ \\
            \amsgrad{}  & $78.91 \pm 0.03$          & $81.17 \pm 0.03$    \\
            \adopt{}    & $\mathbf{79.62} \pm 0.03$    & $\mathbf{81.50} \pm 0.04$    \\
            \bottomrule
        \end{tabular}
        \label{tab:imagenet}
    \end{minipage}
    \hspace{0.02\linewidth}
    \begin{minipage}{0.48\linewidth}
        \centering
        \caption{Negative log-likelihood of NVAEs for MNIST density estimation. Lower is better.}
        \begin{tabular}{l|cc}
            \toprule
            Epoch     & $200$                       & $300$       \\ \midrule
            \adamax{} & $80.19 \pm 0.08$            & $79.41 \pm 0.07$   \\
            \adopt{}  & $\mathbf{79.02} \pm 0.10$   & $\mathbf{78.88} \pm 0.09$  \\
            \bottomrule
        \end{tabular}
        \label{tab:nvae}
    \end{minipage}
\end{table}

{\bf Generative modeling:}
We train NVAE~\citep{vahdat2020nvae} for MNIST using our \adopt{}.
In the official implementation of NVAE, \adamax{}~\citep{kingma2014adam}, an infinite-norm variant of \adam, is used as an optimizer, so we use \adamax{} as a baseline method.
We use the exactly the same setting of the official implementation except that the learning rate for \adopt{} is set to $2 \times 10^{-4}$ since the default value $0.01$ is too large for \adopt{}.
We report the negative log-likelihood for test data on Table \ref{tab:nvae}.
It is observed that the model trained with \adopt{} shows the better likelihood.

\begin{figure}[tbp]
    \begin{minipage}{0.49\linewidth}
        \centering
        \includegraphics[width=\linewidth]{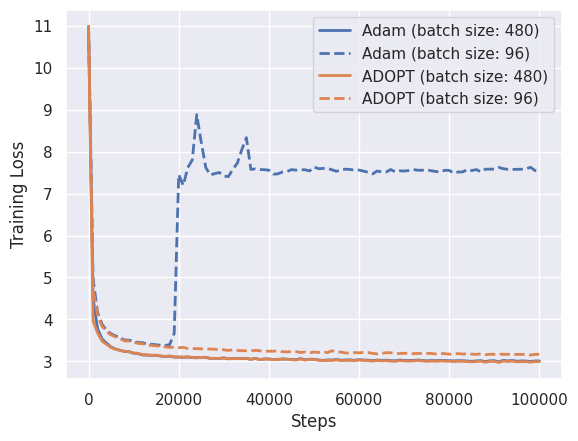}
    \end{minipage}
    \begin{minipage}{0.49\linewidth}
        \centering
        \includegraphics[width=\linewidth]{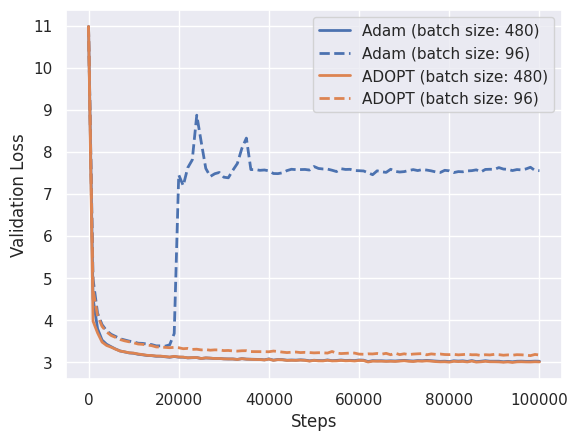}
    \end{minipage}
    \caption{Learning curves of GPT-2 pretraining for training set (left) and validation set (right).}
    \label{fig:gpt}
\end{figure}

{\bf Pretraining of large language models:}
We run a pre-training of GPT-2~\citep{radford2019language} using the nanoGPT~\citep{nanogpt} code base to compare Adam and ADOPT.
We use OpenWebText~\citep{Gokaslan2019OpenWeb} as the training data. 
Experimental setup conforms to the default settings of nanoGPT except for the selection of the optimizer. 
We also test a case in which the total batch size was changed from 480 to 96, as a setting where the gradient noise becomes larger. 
The results are summarized in Figure~\ref{fig:gpt}.
The most notable finding is that in the small batch size case, Adam causes loss spikes in the early stages of training and fails to converge, while ADOPT is always able to train stably.
This is consistent with Adam's theory of non-convergence.
As the gradient noise increases, \(G\) in Theorem 3.1 also increases, and the constant term in Adam's convergence bounds becomes non-negligible especially when using a large-scale dataset like OpenWebText.
As a result, Adam is more likely to fail to train in such cases.
Our ADOPT, on the other hand, does not suffer from this problem because it can always guarantee convergence.
We also observed that both Adam and ADOPT work well when the batch size is large, but even in this case, ADOPT performs slightly better.

{\bf Finetuning of large language models:}
We finetune the pretrained LLaMA-7B on 52K instruction-following data provided by Stanford Alpaca and compare the performance between the default optimizer (\adam{}) and our \adopt{} under the exactly same experimental setting.
For evaluation, we use Multi-task Language Understanding (MMLU) Benchmark~\citep{hendrycks2021measuring}, which is widely used to assess the performance of large language models.
The MMLU score for LLaMA-7B without finetuning is $35.1$.
After fine-tuned via instruction-following using the baseline implementation with \adam{}, the score improves to $41.2$.
When we substitute \adopt{} for \adam{}, the score even improves to $42.13$. 
The detailed score comparison for each task is summarized in Figure \ref{fig:alpaca} in the appendix. 
Other experimental results, including deep RL experiments, and detailed experimental settings are also provided in the appendix.

%% file: 6_conclusion.tex
\section{Conclusion}
\label{sec:conclusion}
In this paper, we demystified the fundamental cause of divergence of adaptive gradient methods based on the exponential moving average, such as \adam{} and \rmsprop{}, in general smooth nonconvex optimization problems, and demonstrate a way to fix the issue, proposing a new optimizer named \adopt{}.
Not only does \adopt{} converge with the optimal rate without depending on a hyperparameter choice in theory, but \adopt{} demonstrates better performance in a wide range of pracital applications.

We expect that this work will serve as a bridge between theory and practice in the research of adaptive gradient methods.
Since \adopt{} can be safely applied to many machine learning problems without careful tuning of hyperparameters, it can be expected to improve the training stability and the model performance in practice by substituting it for the existing adaptive gradient methods (e.g., \adam{}).

One of the limitations of our analysis is that it still relies on the assumption that the second moment of stochastic gradient is uniformly bounded (i.e., Assumption \ref{ass:second_moment}).
Although this assumption is weaker than the bounded stochastic gradient assumption (i.e., Assumption \ref{ass:bounded_stochastic_gradient}), it would be more desirable to relax it to the bounded {\it variance} assumption (i.e., Assumption \ref{ass:bounded_variance}), which is often adopted in the analysis of the vanilla SGD~\citep{ghadimi2013stochastic}.
For \adam{}, a recent work by \citet{wang2023closing} have derived a problem-dependent convergence bound which achieves the $\mathcal{O} ( 1 / \sqrt{T} )$ rate without Assumption \ref{ass:second_moment}.
Their proof techniques may help to relax our assumptions in the proof of Theorem \ref{thm:adopt_constant}, which we leave as future work.

From a broader perspective, adaptive gradient methods like \adam{} have been widely used even for the training of large-scale foundation models (e.g., large language models).
Although such models can be useful for people, their negative aspects, such as concerns about copyright infringement, are not negligible.
Researchers needs to deeply recognize and understand such social impacts of machine learning algorithms.

%% file: A_related_work.tex
\section{Detailed Relationships to Existing Analyses}
\label{sec:detailed_related}

In this section, we discuss the relationships between our analysis and existing ones on the convergence of \adam{}-like optimizers in smooth nonconvex optimization problems.
Tables \ref{tab:problem_setting_comparison} and \ref{tab:rate_comparison} are a summary of comparisons between them in terms of their problem settings and derived convergence rates.

{\bf \citet{zhang2022adam}} focus on convergence of \adam{} in the finite sum problem, where the objective has a following form:
\begin{align}
    f \left( \vtheta \right) = \sum_{i=1}^n f_i \left( \vtheta \right).
\end{align}
$f_i$ is, for example, a loss function for $i$-th training sample.
Although many deep learning problems can be formulated as a finite sum problem, training of the variational autoencoders (VAEs) or diffusion models is out of the finite-sum problem, since their objective is formulated as an infinite sum (i.e., an expectation over continuous variables).
Moreover, they assume the stochastic gradient $\vg$ is $L$-Lipschitz, whereas we only assume true gradient $\nabla f$ is $L$-Lipschitz.
They also assume a growth condition as follows:
\begin{align}
    \mathbb{E} \left[ \left\| \vg_t \right\|^2 \right] \leq G_0^2 + G_1^2 \left\| \nabla f \left( \vtheta_{t-1} \right) \right\|^2. \label{eq:growth}
\end{align}
This growth condition is weaker than our Assumption \ref{ass:second_moment}.
Assumption \ref{ass:second_moment} is a special case of the growth condition where $G_1 = 0$.
Their derived convergence rate has a constant factor of $\mathcal{O} ( G_0 )$; hence the strong growth condition (i.e., $G_0 = 0$) is required to assure convergence.
Moreover, to assure convergence, one needs to choose sufficiently large $\beta_2$, which has to be tuned in a problem-dependent manner. 

{\bf \citet{wang2022provable}} also focus on convergence of \adam{} in the finite sum problem, but they relax the $L$-Lipschitz condition on $\vg$ to the $(L_0, L_1)$-Lipschitz condition.
They also assume the growth condition in Eq. (\ref{eq:growth}), and their convergence rate has the same order with \citet{zhang2022adam}, so it still requires the strong growth condition (i.e., $G_0 = 0$) to assure convergence.
The condition of $\beta_2$ is also similar to \citet{zhang2022adam}.

{\bf \citet{li2023convergence}} consider Adam's convergence on general smooth nonconvex problems.
Similar to \cite{wang2022provable}, they use $(L_0, L_\rho)$-Lipschitz condition on the true gradient $\nabla f$.
They also assume that the gradient noise is almost surely bounded:
\begin{align}
    \left\| \vg - \nabla f \right\| \leq \sigma \label{eq:bounded_noise}
\end{align}
The relationship between this assumption and our Assumption \ref{ass:second_moment} is a little complicated.
Assumption \ref{ass:second_moment} is equivalent to a combination of Assumption \ref{ass:bounded_variance} and the following assumption:
\begin{assumption}
    The true gradient is uniformly bounded, i.e., there exist constants $G$ and $\sigma$ such that $\left\| \nabla f \left( \vtheta \right) \right\|^2 \leq G^2 - \sigma^2$ and $0 < \sigma \leq G$.
    \label{ass:bounded_true_grad}
\end{assumption}
The bounded noise assumption of Eq. (\ref{eq:bounded_noise}) is strictly stronger than Assumption \ref{ass:bounded_variance}, but they do not assume the bounded true gradient (i.e., Assumption \ref{ass:bounded_true_grad}).
The bounded noise assumption is often violated in practice (e.g., training of VAEs), because the gradient is often estimated using unbounded noise (i.e., Gaussian noise).
Their convergence rate $\mathcal{O} ( 1 / \sqrt{T} )$ is better than \citet{zhang2022adam} and \citet{wang2022provable}, while it still requires constraints on the hyperparameters, which have to be chosen in a problem-dependent manner.

{\bf \citet{defossez2022a}} analyzes the convergence of \adam{} under exactly the same assumptions with ours, and they derive the $\mathcal{O} ( \log T / \sqrt{T} )$ rate, which is worse than our ADOPT's convergence rate.
Moereover, to assure the convergence, $\beta_2$ has to be chosen dependently on the total number of iterations $T$.

{\bf \citet{wang2023closing}} analyzes the convergence of \adam{} under Assumptions \ref{ass:bounded_objective}-\ref{ass:bounded_variance}, and they derive the $\mathcal{O} ( 1 / \sqrt{T} )$ rate.
However, to assure the convergence, $\beta_2$ has to be chosen dependently on the total number of iterations $T$ as in \citet{defossez2022a}.

{\bf \citet{chen2018on}} and {\bf \citet{zhou2018convergence}} analyze the convergence of \amsgrad{} for general smooth nonconvex problems, and derive the convergence rate of $\mathcal{O} ( \log T / \sqrt{T} )$ and $\mathcal{O} ( 1 / \sqrt{T} )$, respectively.
However, to guarantee the convergence, the stochastic gradient $\vg$ has to be bounded almost surely (Assumption \ref{ass:bounded_stochastic_gradient}), which is often violated in practice.
In addition, the hyperparameter $\beta_1$ and $\beta_2$ should be chosen satisfying $\beta_1 < \sqrt{\beta_2}$.
This constraint is relatively minor compared to the constraint imposed in the analyses of \adam{}, since it can be satisfied in a problem-independent manner.

\begin{table}[t]
    \centering
    {\scriptsize
    \begin{tabular}{c|cccc}
        \toprule
                                    & Algorithm  & Problem       & Smoothness                                  & Gradient Growth \\ \midrule
        \citet{zhang2022adam}       & Adam       & Finite sum    & $L$-Lipschitz $\vg$                  & $\mathbb{E} [ \| g \|^2 ] \leq G_0^2 + G_1^2 \| \nabla f \|^2$ \\
        \citet{wang2022provable}    & Adam       & Finite sum    & $(L_0, L_1)$-Lipschitz $\vg$         & $\mathbb{E} [ \| g \|^2 ] \leq G_0^2 + G_1^2 \| \nabla f \|^2$ \\ 
        \citet{li2023convergence}   & Adam       & General       & $(L_0, L_\rho)$-Lipschitz $\nabla f$    & $\| \vg - \nabla f \| \leq \sigma$ \\ 
        \citet{defossez2022a}       & Adam       & General       & $L$-Lipschitz $\nabla f$             & $\mathbb{E} [ \| \vg \|^2 ] \leq G^2$ \\
        \citet{wang2023closing}     & Adam       & General       & $L$-Lipschitz $\nabla f$             & $\mathbb{E} [ \| \vg - \nabla f \|^2 ] \leq G^2$ \\
        \citet{chen2018on}          & AMSGrad    & General       & $L$-Lipschitz $\nabla f$             & $\| \vg \| \leq G$ \\ 
        \citet{zhou2018convergence} & AMSGrad    & General       & $L$-Lipschitz $\nabla f$             & $\| \vg \| \leq G$ \\ 
        Ours                        & ADOPT      & General       & $L$-Lipschitz $\nabla f$             & $\mathbb{E} [ \| \vg \|^2 ] \leq G^2$ \\
        \bottomrule
    \end{tabular}
    }
    \caption{Comparison of the problem settings between our analysis and other existing works.}
    \label{tab:problem_setting_comparison}
\end{table}

\begin{table}[t]
    \centering
    {\scriptsize
    \begin{tabular}{c|cc}
        \toprule
                                    & Constraints                                                               & Convergence        \\ \midrule
        \citet{zhang2022adam}       & $\beta_1 < \sqrt{\beta_2}, \beta_2 \geq \gamma (n)$                       & $\mathcal{O} ( \log T / \sqrt{T} ) + \mathcal{O} ( G_0 )$ \\
        \citet{wang2022provable}    & $\beta_1 < \sqrt{\beta_2}, \delta ( \beta_2 ) = \mathcal{O} (1 / G_1)$    & $\mathcal{O} ( \log T / \sqrt{T} ) + \mathcal{O} ( G_0 )$ \\ 
        \citet{li2023convergence}   & $\beta_1 < \sqrt{\beta_2}, \beta_1 \leq c ( L_0, L_\rho, G )$             & $\mathcal{O} ( 1 / \sqrt{T} )$ \\
        \citet{defossez2022a}       & $\beta_1 < \sqrt{\beta_2}, 1 - \beta_2 = \Theta ( 1 / T )$                & $\mathcal{O} ( \log T / \sqrt{T} )$       \\
        \citet{wang2023closing}     & $\beta_1 \leq \sqrt{\beta_2}-8 \left(1-\beta_2\right) \beta_2^{-2}, 1 - \beta_2 = \Theta ( 1 / T )$            & $\mathcal{O} ( 1 / \sqrt{T} )$       \\
        \citet{chen2018on}          & $\beta_1 < \sqrt{\beta_2}$                                                & $\mathcal{O} ( \log T / \sqrt{T} )$ \\ 
        \citet{zhou2018convergence} & $\beta_1 < \sqrt{\beta_2}$                                                & $\mathcal{O} ( 1 / \sqrt{T} )$ \\
        Ours                        & -                                                                         & $\mathcal{O} ( 1 / \sqrt{T} )$ \\
        \bottomrule
    \end{tabular}
    }
    \caption{Comparison of the convergence rate and imposed constraints on the hyperparameters between our analysis and other existing works. Please refer to the original papers for the definitions of $\gamma$ and $c$.}
    \label{tab:rate_comparison}
\end{table}

%% file: B_replacement.tex
\section{With-Replacement vs. Without-Replacement}
In the optimization of finite-sum problems, practitioners often use {\it without-replacement sampling}, which is also known as {\it random shuffling}, to obtain stochastic gradient.
In this case, the stochastic gradient has a small bias due to the lack of replacement, so Assumption \ref{ass:unbiased} is violated.
However, the vanilla SGD is known to converge with the without-replacement strategy~\citep{haochen2019random}, and some of the analyses of \adam{} also adopt without-replacement sampling~\citep{zhang2022adam,wang2022provable}.

Unfortunately, we find that our ADOPT has a counter example, in which ADOPT fails to converge when using without-replacement sampling.
For example, when we consider minimizing $f ( \theta ) = \sum_{i=1}^3 f_i ( \theta )$, where $\theta \in [ -1, 1 ]$, $f_1 ( \theta ) = 1.9 \theta$ and  $f_2 ( \theta ) = f_2 ( \theta ) = - \theta$, it can be easily observed that ADOPT with $\beta_1 = \beta_2 = 0$ fails to converge to the correct solution, i.e., $\theta = 1$.

This non-convergence can be easily avoided by using the with-replacement strategy.
Moreover, the difference between with- and without-replacement sampling becomes negligible when $n$ in the finite-sum $\sum_{i=1}^n f_i$ is large enough; hence it does not affect the practical performance very much.
In fact, our experiments except for the toy example are performed using without-replacement sampling, but divergent behaviors are not observed.
If one applies ADOPT to problems where the difference seems severe (e.g., when training with a small dataset), we recommend to use with-replacement sampling instead of random shuffling for stable training. 
When one uses PyTorch~\citep{NEURIPS2019_9015} for the implementation, for example, with-replacement sampling can be easily applied by specifying {\tt replacemnet=True} for {\tt torch.utils.data.RandomSampler}, and feeding it to the {\tt sampler} argument of {\tt torch.utils.data.DataLoader}.

%% file: C_another.tex
\section{Another Expression of ADOPT}

\begin{algorithm}[tb]
\begin{algorithmic}
    \caption{Alternative representation of \adopt{} algorithm}
    \label{alg:adopt_new}
    \REQUIRE Learning rate $\left\{ \alpha_t \right\}$, initial parameter $\vtheta_0$
    \REQUIRE Exponential decay rate $0 \leq \beta_1 < 1, 0 \leq \beta_2 \leq 1$, small constant $\epsilon > 0$
    \STATE {$\vv_0 \leftarrow \vg_0 \odot \vg_0$}
    \FOR {$t = 1$ to $T$}
        \STATE {$\vm_{t} \leftarrow \beta_1 \cdot \vm_{t-1} + \left( 1 - \beta_1 \right) { \vg_{t} } / \max \left\{ \sqrt{ \vv_{t-1} } , \epsilon \right\}$}
        \STATE {$\vtheta_{t} \leftarrow \vtheta_{t-1} - \alpha_{t} \vm_t $}
        \STATE {$\vv_{t} \leftarrow \beta_2 \cdot \vv_{t-1} + \left( 1 - \beta_2 \right) \vg_{t} \odot \vg_{t}$}
    \ENDFOR
    \STATE {{\bf return} $\{ \vtheta_t \}_{t=1}^{T}$}
\end{algorithmic}
\end{algorithm}



%% file: D_hyperparam.tex
\section{Recommendation of Hyperparameter Settings for ADOPT}
\label{sec:recommended_hparams}
We experimentally find that our \adopt{} works similarly to \adam{} when the same hyperparameters are used, but $\epsilon$ should be set to a little larger value (e.g., $1 \times 10^{-6}$) for \adopt{} compared to \adam{}, in which $\epsilon$ is set to $1 \times 10^{-8}$ by default.
Our recommendation of the hyperparameter settings for \adopt{} is provided in Table \ref{tab:adopt_default}.

\begin{table}[ht]
    \centering
    \begin{tabular}{c|c}
    \toprule
      $\beta_1$   &  $0.9$ \\
      $\beta_2$   &  $0.9999$ \\
      $\epsilon$  &  $1 \times 10^{-6}$ \\
      \(c_t\)     &  \(t^{1/4}\) \\
    \bottomrule
    \end{tabular}
    \caption{Recommended hyperparameters for the \adopt{} algorithm}
    \label{tab:adopt_default}
\end{table}

%% file: E_theorem.tex
\section{Theorems}
\begin{theorem}
\label{thm:adopt_deminishing}
Under Assumptions \ref{ass:bounded_objective}, \ref{ass:unbiased}, \ref{ass:smooth}, and \ref{ass:second_moment}, if the objective $f$ is upper-bounded by $f_\mathrm{sup}$, the following holds for the \adopt{} algorithm with a learning rate $\alpha_t= \Theta \left( 1 / \sqrt{t} \right)${\rm :}
\begin{align}
    &\min_{t=1,\ldots,T} \left\{ \mathbb{E} \left[ \left\| \nabla f \left( \vtheta_{t} \right) \right\|^{4/3} \right]^{3/2} \right\} = \mathcal{O} \left( 1 / \sqrt{T} \right) .
\end{align}
\end{theorem}

\begin{theorem}
\label{thm:clipped_adopt_constant}
Under Assumptions \ref{ass:bounded_objective}-\ref{ass:smooth} and \ref{ass:second_moment}, the following holds for the clipped \adopt{} algorithm with a constant learning rate $\alpha_t= \alpha = \Theta \left( 1 / \sqrt{T} \right)$ when clipping value meets \( c_t = c = \Theta \left( T^{1/4} \right) \) or \(c_t = \Theta \left( t^{1/4} \right)\){\rm :}
\begin{align}
    &\min_{t=1, \ldots, T} \left\{ \mathbb{E} \left[ \left\| \nabla f ( \vtheta_{t-1} ) ) \right\|^{4/3} \right]^{3/2} \right\}
    \leq \mathcal{O} \left( 1 / \sqrt{T} \right) .  \label{eq:clipped_adopt_bound}
\end{align}
\end{theorem}

\begin{theorem}
\label{thm:clipped_adopt_deminishing}
Under Assumptions \ref{ass:bounded_objective}, \ref{ass:unbiased}, \ref{ass:smooth}, and \ref{ass:second_moment}, if the objective $f$ is upper-bounded by $f_\mathrm{sup}$, the following holds for the \adopt{} algorithm with a learning rate $\alpha_t= \Theta \left( 1 / \sqrt{t} \right)$ and a clipping value \(c_t = \Theta \left( t^{1/4} \right) \){\rm :}
\begin{align}
    &\min_{t=1, \ldots, T} \left\{ \mathbb{E} \left[ \left\| \nabla f ( \vtheta_{t-1} ) ) \right\|^{4/3} \right]^{3/2} \right\}
    \leq \mathcal{O} \left( 1 / \sqrt{T} \right) .
\end{align}
\end{theorem}

%% file: F_proof.tex
\section{Proofs}

\begin{proof}[Proof of Theorems \ref{thm:adopt_constant}, \ref{thm:clipped_adopt_constant}, \ref{thm:adopt_deminishing}, and \ref{thm:clipped_adopt_deminishing}]
We only provide the proof for the clipped ADOPT here, since the vanilla ADOPT can be seen as a special case of the clipped ADOPT with $c_t = \infty$.

We define $\vphi_t$ for $t \geq 1$ as follows:
\begin{align}
    \vphi_t = \frac{ 1 }{ 1 - \beta_1 } \vtheta_t - \frac{ \beta_1 }{ 1 - \beta_1 } \vtheta_{t-1}.
\end{align}
We also define $\vphi_0 = \vtheta_0$.
By Assumption \ref{ass:smooth}, the following holds for $t \geq 1$:
\begin{align}
    f \left( \vphi_{t} \right) 
    &\leq f \left( \vphi_{t-1} \right) + \nabla f \left( \vphi_{t-1} \right)^\top \left( \vphi_{t} - \vphi_{t-1} \right) + \frac{L}{2} \left\| \vphi_{t} - \vphi_{t-1} \right\|^2 \\
    &= f \left( \vphi_{t-1} \right) + \nabla f \left( \vtheta_{t-1} \right)^\top \left( \vphi_{t} - \vphi_{t-1} \right) \nonumber \\
    &\quad \ + \left( \nabla f \left( \vphi_{t-1} \right) - \nabla f \left( \vtheta_{t-1} \right) \right)^\top \left( \vphi_{t} - \vphi_{t-1} \right) + \frac{L}{2} \left\| \vphi_{t} - \vphi_{t-1} \right\|^2 \\
    &\leq f \left( \vphi_{t-1} \right) + \nabla f \left( \vtheta_{t-1} \right)^\top \left( \vphi_{t} - \vphi_{t-1} \right) \nonumber \\
    &\quad \ + \left\| \nabla f \left( \vphi_{t-1} \right) - \nabla f \left( \vtheta_{t-1} \right) \right\| \left\| \vphi_{t} - \vphi_{t-1} \right\| + \frac{L}{2} \left\| \vphi_{t} - \vphi_{t-1} \right\|^2 \\
    &\leq f \left( \vphi_{t-1} \right) + \nabla f \left( \vtheta_{t-1} \right)^\top \left( \vphi_{t} - \vphi_{t-1} \right) \nonumber \\
    &\quad \ + L \left\| \vphi_{t-1} - \vtheta_{t-1} \right\| \left\| \vphi_{t} - \vphi_{t-1} \right\| + \frac{L}{2} \left\| \vphi_{t} - \vphi_{t-1} \right\|^2,
\end{align}
where the second inequality holds due to the Cauchy-Schwarz inequality, and the last inequality holds due to Assumption \ref{ass:smooth}.
By taking the expectation, the following holds:
\begin{align}
    \mathbb{E} \left[ f \left( \vphi_{t} \right) \right]
    &\leq \mathbb{E} \left[ f \left( \vphi_{t-1} \right) \right] + \mathbb{E} \left[ \nabla f \left( \vtheta_{t-1} \right)^\top \left( \vphi_{t} - \vphi_{t-1} \right) \right] \nonumber \\
    &\quad \ + L \mathbb{E} \left[ \left\| \vphi_{t-1} - \vtheta_{t-1} \right\| \left\| \vphi_{t} - \vphi_{t-1} \right\| \right] + \frac{L}{2} \mathbb{E} \left[ \left\| \vphi_{t} - \vphi_{t-1} \right\|^2 \right] \\
    &\leq \mathbb{E} \left[ f \left( \vphi_{t-1} \right) \right] + \frac{ \sqrt{2} \left( \alpha_{t-1} - \alpha_{t} \right) \beta_1 \left( 1 - \beta_1^{t-1} \right) G^2 }{ \left( 1 - \beta_1 \right) \epsilon } - \frac{\alpha_{t}}{2} \frac{ \mathbb{E} \left[ \left\| \nabla f \left( \vtheta_{t-1} \right) \right\|_i^{4/3} \right]^{3/2} }{ \sqrt{ G^2 + \epsilon^2 } } \label{eq:exp_ineq} \\
    &\quad \ + \frac{ 2 \alpha_{t-1} \left( \alpha_{t-1} - \alpha_{t} \right) \beta_1^2 \left( 1 - \beta_1^{t-1} \right) G^2 L }{ \epsilon^2 \left( 1 - \beta_1 \right)^2 } + \frac{ 2 \alpha_{t} \alpha_{t-1} \beta_1 \sqrt{ 1 - \beta_1^{t-1} } G^2 L }{ \left( 1 - \beta_1 \right) \epsilon^2 } \nonumber \\
    &\quad \ + \frac{ \left( \alpha_{t-1} - \alpha_{t} \right)^2 \beta_1^2 \left( 1 - \beta_1^{t-1} \right) G^2 L }{ \left( 1 - \beta_1 \right)^2 \epsilon^2 } \nonumber \\ 
    &\quad \ + \frac{\alpha_{t}^2 G^2 L }{ \epsilon^2 } + \frac{ \alpha_t \left( \alpha_{t-1} - \alpha_{t} \right) \beta_1 \sqrt{ 1 - \beta_1^{t-1} }G^2 L }{ \left( 1 - \beta_1 \right) \epsilon^2 } + \frac{\alpha_t G^4}{2 \epsilon^3 c_t^2} . \nonumber
\end{align}

When $\alpha_t = \alpha$, the following holds:
\begin{align}
    &\mathbb{E} \left[ f \left( \vphi_{t} \right) \right] \nonumber \\
    &\leq \mathbb{E} \left[ f \left( \vphi_{t-1} \right) \right] 
    - \frac{\alpha}{2} \frac{ \mathbb{E} \left[ \left\| \nabla f \left( \vtheta_{t-1} \right) \right\|_i^{4/3} \right]^{3/2} }{ \sqrt{ G^2 + \epsilon^2 } } + \frac{ 2 \alpha^2 \beta_1 \sqrt{ 1 - \beta_1^{t-1} } G^2 L }{ \left( 1 - \beta_1 \right) \epsilon^2 } 
    + \frac{\alpha^2 G^2 L }{ \epsilon^2 } + \frac{\alpha G^4}{2 \epsilon^3 c_t^2} \\ 
    &\leq \mathbb{E} \left[ f \left( \vphi_{t-1} \right) \right] - \frac{\alpha}{2} \frac{ \mathbb{E} \left[ \left\| \nabla f \left( \vtheta_{t-1} \right) \right\|^{4/3} \right]^{3/2} }{ \sqrt{ G^2 + \epsilon^2 } } + \frac{ \alpha^2 \left( 1 + \beta_1 \right) G^2 L }{ \left( 1 - \beta_1 \right) \epsilon^2 } + \frac{\alpha G^4}{2 \epsilon^3 c_t^2} . 
\end{align}
Telescoping it for $t = 1, \ldots, T$, we have
\begin{align}
    &\mathbb{E} \left[ f \left( \vphi_{T} \right) \right] \nonumber \\
    &\leq f \left( \vtheta_{0} \right) - \frac{\alpha}{2} \frac{ \sum_{t=1}^T \mathbb{E} \left[ \left\| \nabla f \left( \vtheta_{t-1} \right) \right\|^{4/3} \right]^{3/2} }{ \sqrt{ G^2 + \epsilon^2 } } + \frac{ \alpha^2 \left( 1 + \beta_1 \right) G^2 L T }{ \left( 1 - \beta_1 \right) \epsilon^2 } + \frac{\alpha G^4 T}{2 \epsilon^3} \sum_{t=1}^T c_t^{-2} \\
    &\leq f \left( \vtheta_{0} \right) - \frac{\alpha}{2} \frac{ \sum_{t=1}^T \mathbb{E} \left[ \left\| \nabla f \left( \vtheta_{t-1} \right) \right\|^{4/3} \right]^{3/2} }{ \sqrt{ G^2 + \epsilon^2 } } + \frac{ \alpha^2 \left( 1 + \beta_1 \right) G^2 L T }{ \left( 1 - \beta_1 \right) \epsilon^2 } + \frac{\alpha G^4 T}{2 \epsilon^3} \sum_{t=1}^T c_t^{-2}
\end{align}
By rearranging the terms, we have
\begin{align}
    &\min_{t=1,\ldots,T} \left\{ \mathbb{E} \left[ \left\| \nabla f \left( \vtheta_{t-1} \right) \right\|^{4/3} \right]^{3/2} \right\} \nonumber \\
    &\leq \frac{ \sum_{t=1}^T \mathbb{E} \left[ \left\| \nabla f \left( \vtheta_{t-1} \right) \right\|^{4/3} \right]^{3/2} }{ T } \\
    &\leq 2 \sqrt{ G^2 + \epsilon^2 } \left( \frac{f \left( \vtheta_{0} \right) - f_\mathrm{inf}}{\alpha T} + \frac{ \alpha \left( 1 + \beta_1 \right) G^2 L }{ \left( 1 - \beta_1 \right) \epsilon^2 } + \frac{ G^4 }{2 \epsilon^3 T} \sum_{t=1}^T c_t^{-2} \right)  .
\end{align}
When \(\alpha = \Theta \left( 1 / \sqrt{T} \right)\), the first and second terms are \(\mathcal{O} \left( 1 / \sqrt{T} \right)\).
If the clipping value is constant, and \(c_t = c = \Theta \left( T^{1/4} \right)\), the last term is also \(\mathcal{O} \left( 1 / \sqrt{T} \right)\):
\begin{align}
    \frac{1}{T} \sum_{t=1}^T c_t^{-2} = \mathcal{O} \left( c^{-2} \right) = \mathcal{O} \left( 1 / \sqrt{T} \right) .
\end{align}
The same holds true when the clipping value is \(c_t = \Theta \left( t^{1/4} \right)\), since the following holds:
\begin{align}
    \frac{1}{T} \sum_{t=1}^T c_t^{-2} = \mathcal{O} \left( \frac{1}{T} \sum_{t=1}^T 1 / \sqrt{t} \right) = \mathcal{O} \left( 1 / \sqrt{T} \right) .
\end{align}

For the case of \(\alpha_t = \Theta \left( 1 / \sqrt{t} \right) \) and \( c_t = \Theta \left( t^{1/4} \right) \), we can set $\alpha_t = \alpha / \sqrt{t}$ and $c_t = c {t}^{1/4}$ without loss of generality.
The following holds for $t \geq 2$ for this case:
\begin{align}
    \alpha_{t-1} - \alpha_t 
    &= \alpha \left( \frac{1}{ \sqrt{ t - 1 } } - \frac{ 1 }{ \sqrt{ t } } \right) \\
    &= \frac{ \alpha \left( \sqrt{t} - \sqrt{ t - 1 } \right) }{ \sqrt{ t \left( t - 1 \right) } } \\
    &= \frac{ \alpha }{ \sqrt{ t \left( t - 1 \right) } \left( \sqrt{ t } + \sqrt{ t - 1 } \right) } \\
    &\leq \frac{ \alpha }{ 2 \left( t - 1 \right)^{3/2} } \\
    &\leq \frac{ \sqrt{2} \alpha }{ t^{3/2} }.    
\end{align}
This also holds for $t=1$ by defining $\alpha_0 = \alpha$.
Applying it to Eq. (\ref{eq:exp_ineq}), we have
\begin{align}
    \mathbb{E} \left[ f \left( \vphi_{t} \right) \right]
    &\leq \mathbb{E} \left[ f \left( \vphi_{t-1} \right) \right] + \frac{ \sqrt{2} \left( \alpha_{t-1} - \alpha_{t} \right) \beta_1 \left( 1 - \beta_1^{t-1} \right) G^2 }{ \left( 1 - \beta_1 \right) \epsilon } - \frac{ \alpha_{t} }{2} \frac{ \mathbb{E} \left[ \left\| \nabla f \left( \vtheta_{t-1} \right) \right\|_i^{4/3} \right]^{3/2} }{ \sqrt{ G^2 + \epsilon^2 } } \nonumber \\
    &\quad \ + \frac{ 2 \alpha_{t-1} \left( \alpha_{t-1} - \alpha_{t} \right) \beta_1^2 \left( 1 - \beta_1^{t-1} \right) G^2 L }{ \epsilon^2 \left( 1 - \beta_1 \right)^2 } + \frac{ 2 \alpha_{t} \alpha_{t-1} \beta_1 \sqrt{ 1 - \beta_1^{t-1} } G^2 L }{ \left( 1 - \beta_1 \right) \epsilon^2 } \nonumber \\
    &\quad \ + \frac{ \left( \alpha_{t-1} - \alpha_{t} \right)^2 \beta_1^2 \left( 1 - \beta_1^{t-1} \right) G^2 L }{ \left( 1 - \beta_1 \right)^2 \epsilon^2 } + \frac{\alpha_{t}^2 G^2 L }{ \epsilon^2 } \nonumber \\
    &\quad \ + \frac{ \alpha_t \left( \alpha_{t-1} - \alpha_{t} \right) \beta_1 \sqrt{ 1 - \beta_1^{t-1} }G^2 L }{ \left( 1 - \beta_1 \right) \epsilon^2 } + \frac{ \alpha_t G^4 }{ 2 \epsilon^3 c_t^2 } \\
    &\leq \mathbb{E} \left[ f \left( \vphi_{t-1} \right) \right] + \frac{ {2} \alpha \beta_1 \left( 1 - \beta_1^{t-1} \right) G^2 }{ t^{3/2} \left( 1 - \beta_1 \right) \epsilon } - \frac{\alpha}{ 2 \sqrt{ t } } \frac{ \mathbb{E} \left[ \left\| \nabla f \left( \vtheta_{t-1} \right) \right\|_i^{4/3} \right]^{3/2} }{ \sqrt{ G^2 + \epsilon^2 } } \nonumber \\
    &\quad \ + \frac{ 4 \alpha^2 \beta_1^2 G^2 L }{ \epsilon^2 \left( 1 - \beta_1 \right)^2 t^2 } + \frac{ 2 \sqrt{ 2 } \alpha^2 \beta_1 G^2 L }{ \left( 1 - \beta_1 \right) \epsilon^2 t } + \frac{ 2 \alpha^2 \beta_1^2 G^2 L }{ \left( 1 - \beta_1 \right)^2 \epsilon^2 t^3 } + \frac{\alpha^2 G^2 L }{ \epsilon^2 t } \nonumber \\
    &\quad \ + \frac{ \sqrt{ 2 } \alpha^2 \beta_1 G^2 L }{ \left( 1 - \beta_1 \right) \epsilon^2 t^2 } + \frac{ \alpha G^4 }{ 2 c^2 \epsilon^3 {t} } \\
    &= \mathbb{E} \left[ f \left( \vphi_{t-1} \right) \right] - \frac{\alpha}{ 2 \sqrt{ t } } \frac{ \mathbb{E} \left[ \left\| \nabla f \left( \vtheta_{t-1} \right) \right\|_i^{4/3} \right]^{3/2} }{ \sqrt{ G^2 + \epsilon^2 } } \nonumber \\
    &\quad \ + \left( \frac{ \alpha^2 \left( 1 + \left( 2 \sqrt{2} - 1 \right) \beta_1 \right) G^2 L }{ \left( 1 - \beta_1 \right) \epsilon^2 } + \frac{ \alpha G^4 }{ 2 c^2 \epsilon^3 } \right) \cdot t^{-1} + \frac{ {2} \alpha \beta_1 G^2 }{ \left( 1 - \beta_1 \right) \epsilon } \cdot t^{- \frac{3}{2} } \nonumber \\
    &\quad \ + \frac{ \sqrt{ 2 } \alpha^2 \beta_1 \left( 1 + \left( 2 \sqrt{ 2 } - 1 \right) \beta_1 \right) G^2 L }{ \left( 1 - \beta_1 \right)^2 \epsilon^2 } \cdot t^{-2} + \frac{ 2 \alpha^2 \beta_1^2 G^2 L }{ \left( 1 - \beta_1 \right)^2 \epsilon^2 } \cdot t^{-3} .
\end{align}
Multiplying $t$ to the both sides and rearranging the terms, we have
\begin{align}
    &\frac{ \sqrt{t} \ \mathbb{E} \left[ \left\| \nabla f \left( \vtheta_{t-1} \right) \right\|^{4/3} \right]^{3/2} }{2 \sqrt{ G^2 + \epsilon^2 }} \nonumber \\
    &\leq \frac{ \mathbb{E} \left[ f \left( \vphi_{t-1} \right) - f \left( \vphi_{t} \right) \right] }{ \alpha } \cdot t + \frac{ \alpha \left( 1 + \left( 2 \sqrt{2} - 1 \right) \beta_1 \right) G^2 L }{ \left( 1 - \beta_1 \right) \epsilon^2 } + \frac{ G^4 }{ 2 c^2 \epsilon^3 } + \frac{ {2} \beta_1 G^2 }{ \left( 1 - \beta_1 \right) \epsilon } \cdot t^{- \frac{1}{2} } \\
    &\quad \ + \frac{ \sqrt{ 2 } \alpha \beta_1 \left( 1 + \left( 2 \sqrt{ 2 } - 1 \right) \beta_1 \right) G^2 L }{ \left( 1 - \beta_1 \right)^2 \epsilon^2 } t^{-1} 
    + \frac{ 2 \alpha \beta_1^2 G^2 L }{ \left( 1 - \beta_1 \right)^2 \epsilon^2 } t^{-2} .
\end{align}

Telescoping it for $t = 1, \ldots, T$, we have
\begin{align}
    &\sum_{t=1}^T \frac{ \sqrt{t} \ \mathbb{E} \left[ \left\| \nabla f \left( \vtheta_{t-1} \right) \right\|^{4/3} \right]^{3/2} }{ 2 \sqrt{ G^2 + \epsilon^2 } } \nonumber \\
    &\leq \frac{ f \left( \vphi_0 \right) - T f \left( \vphi_T \right) + \sum_{t=1}^{T-1} f \left( \vphi_t \right) }{ \alpha } + \frac{ \alpha \left( 1 + \left( 2 \sqrt{2} - 1 \right) \beta_1 \right) G^2 L T }{ \left( 1 - \beta_1 \right) \epsilon^2 } + \frac{ G^4 T }{2 c^2 \epsilon^3} \nonumber \\
    &\quad \ + \frac{ {2} \beta_1 G^2 }{ \left( 1 - \beta_1 \right) \epsilon } \sum_{t=1}^T t^{- \frac{1}{2} } + \frac{ \sqrt{ 2 } \alpha \beta_1 \left( 1 + \left( 2 \sqrt{ 2 } - 1 \right) \beta_1 \right) G^2 L }{ \left( 1 - \beta_1 \right)^2 \epsilon^2 } \sum_{t=1}^T t^{-1} + \frac{ 2 \alpha \beta_1^2 G^2 L }{ \left( 1 - \beta_1 \right)^2 \epsilon^2 } \sum_{t=1}^T t^{-2} \\
    &\leq \frac{ f_\mathrm{sup} - f_\mathrm{inf} }{ \alpha } T + \frac{ \alpha \left( 1 + \left( 2 \sqrt{2} - 1 \right) \beta_1 \right) G^2 L T }{ \left( 1 - \beta_1 \right) \epsilon^2 } + \frac{G^4 T}{2 c^2 \epsilon^3} \nonumber \\
    &\quad \ + \frac{ {2} \beta_1 G^2 }{ \left( 1 - \beta_1 \right) \epsilon } \left( 1 + \int_1^T t^{- \frac{1}{2} } dt \right) + \frac{ \sqrt{ 2 } \alpha \beta_1 \left( 1 + \left( 2 \sqrt{ 2 } - 1 \right) \beta_1 \right) G^2 L }{ \left( 1 - \beta_1 \right)^2 \epsilon^2 } \left( 1 + \int_1^T t^{-1} dt \right) \nonumber \\
    &\quad \ + \frac{ 2 \alpha \beta_1^2 G^2 L }{ \left( 1 - \beta_1 \right)^2 \epsilon^2 } \left( 1 + \int_1^T t^{-2} dt \right) \\
    &\leq \frac{ f_\mathrm{sup} - f_\mathrm{inf} }{ \alpha } T + \frac{ \alpha \left( 1 + \left( 2 \sqrt{2} - 1 \right) \beta_1 \right) G^2 L T }{ \left( 1 - \beta_1 \right) \epsilon^2 } + \frac{G^4 T}{2 c^2 \epsilon^3} + \frac{ {2} \beta_1 G^2 }{ \left( 1 - \beta_1 \right) \epsilon } \left( 2 \sqrt{T} - 1 \right) \nonumber \\
    &\quad \ + \frac{ \sqrt{ 2 } \alpha \beta_1 \left( 1 + \left( 2 \sqrt{ 2 } - 1 \right) \beta_1 \right) G^2 L }{ \left( 1 - \beta_1 \right)^2 \epsilon^2 } \left( 1 + \log T \right) + \frac{ 2 \alpha \beta_1^2 G^2 L }{ \left( 1 - \beta_1 \right)^2 \epsilon^2 } \left( 2 - \frac{1}{T} \right)
\end{align}

Therefore, the following bound is derived.
\begin{align}
    &\min_{t=1,\ldots,T} \left\{ \mathbb{E} \left[ \left\| \nabla f \left( \vtheta_{t-1} \right) \right\|^{4/3} \right]^{3/2} \right\} \nonumber \\
    &\leq \frac{ \sum_{t=1}^T \sqrt{t} \ \mathbb{E} \left[ \left\| \nabla f \left( \vtheta_{t-1} \right) \right\|^{4/3} \right]^{3/2} }{\sum_{t=1}^T \sqrt{t} } \\
    &\leq \frac{ \sum_{t=1}^T \sqrt{t} \ \mathbb{E} \left[ \left\| \nabla f \left( \vtheta_{t-1} \right) \right\|^{4/3} \right]^{3/2} }{ \int_0^T \sqrt{t} dt } \\
    &\leq \frac{ 3C \left( f_\mathrm{sup} - f_\mathrm{inf} \right) }{ 2 \alpha } \frac{1}{\sqrt{T}} + \frac{ 3 \alpha \left( 1 + \left( 2 \sqrt{2} - 1 \right) \beta_1 \right) C G^2 L }{ 2 \left( 1 - \beta_1 \right) \epsilon^2 \sqrt{T} } + \frac{3 G^4 T}{4 c^2 \epsilon^3} + \frac{ 3 \beta_1 C G^2 }{ \left( 1 - \beta_1 \right) \epsilon } \left( \frac{2}{T} - \frac{1}{T^{3/2}} \right) \nonumber \\
    &\quad \ + \frac{ 3 \alpha \beta_1 \left( 1 + \left( 2 \sqrt{ 2 } - 1 \right) \beta_1 \right) C G^2 L }{ \sqrt{ 2 } \left( 1 - \beta \right)^2 \epsilon^2 } \left( \frac{1}{T^{3/2}} + \frac{\log T}{T^{3/2}} \right) + \frac{ 3 \alpha \beta_1^2 C G^2 L }{ \left( 1 - \beta_1 \right)^2 \epsilon^2 } \left( \frac{2}{T^{3/2}} - \frac{1}{T^{5/2}} \right) \\
    &= \mathcal{O} \left( 1 / \sqrt{T} \right),
\end{align}
where $C = 2 \sqrt{ G^2 + \epsilon^2 }$.

\end{proof}

\section{Lemmas}

\begin{lemma}
\label{lem:bound_true_grad_norm}
For all $\vtheta \in \mathbb{R}^D$ and $t \geq 1$, the following holds
\begin{align}
    \left\| \nabla f \left( \vtheta_{t-1} \right) \right\| \leq G.
\end{align}
\begin{proof}
\begin{align}
    \left\| \nabla f \left( \vtheta_{t-1} \right) \right\|
    &= \sqrt{ \left\| \mathbb{E} \left[ \vg_t \right] \right\|^2 } \\
    &\leq \sqrt{ \mathbb{E} \left[ \left\|  \vg_t \right\|^2 \right]} \\
    &\leq G.
\end{align}
The first inequality holds because $\mathbb{E} [ ( \vg_t )_i ]^2 \leq \mathbb{E} [ ( \vg_t )_i^2 ]$, and the second inequality holds due to Assumption \ref{ass:second_moment}.
\end{proof}
\end{lemma}

\begin{lemma}
\label{lem:bound_exp_grad_norm}
For all $\vtheta \in \mathbb{R}^D$ and $t \geq 1$, the following holds
\begin{align}
    \mathbb{E} \left[ \left\| \vg_t \right\| \right] \leq G
\end{align}

\begin{proof}
\begin{align}
    \mathbb{E} \left[ \left\| \vg_t \right\| \right]
    &\leq \mathbb{E} \left[ \left\| \vg_t \right\|^2 \right]^{1/2} \\
    &\leq G,
\end{align}
where the first inequality holds due to the Hölder's inequality and the second one holds due to Assumption \ref{ass:second_moment}.
\end{proof}
\end{lemma}

\begin{lemma}
\label{lem:bound_exp_v_rmsprop}
For the \rmsprop{} algorithm, the following holds for $t \geq 1$:
\begin{align}
    \mathbb{E} \left[ \sum_{i=1}^D \left( \vv_t \right)_i \right] \leq \left( 1 - \beta_2^t \right) G^2
\end{align}

\begin{proof}
\begin{align}
    \mathbb{E} \left[ \sum_{i=1}^D \left( \vv_t \right)_i \right] 
    &= \mathbb{E} \left[ \left( 1 - \beta_2 \right) \sum_{i=1}^D \sum_{k=1}^t \beta_2^{t-k} \left( \vg_{k} \right)_i^2 \right] \\
    &\leq \left( 1 - \beta_2 \right) G^2 \sum_{k=1}^t \beta_2^{t-k} \\
    &= \left( 1 - \beta_2^t \right) G^2.
\end{align}
\end{proof}
\end{lemma}

\begin{lemma}
\label{lem:decompose_v}
For the \rmsprop{} algorithm, the following holds:
\begin{align}
    &\mathbb{E} \left[ \nabla f \left( \vtheta_{t-1} \right)^\top \left( \frac{ \vg_t } { \sqrt{ {\vv}_t + \epsilon^2 } } \right) \right] \nonumber \\
    &\geq \frac{1}{2} \mathbb{E} \left[ \nabla f \left( \vtheta_{t-1} \right)^\top \left( \frac{ \vg_t } { \sqrt{ \tilde{\vv}_t + \epsilon^2 } } \right) \right] - 2G \sqrt{ 1 - \beta_2 } \mathbb{E} \left[ \left\| \frac{ \vg_t  }{ \sqrt{ {\vv}_t + \epsilon^2 } } \right\|^2 \right]
\end{align}
\begin{proof}    
\begin{align}
    &\mathbb{E} \left[ \nabla f \left( \vtheta_{t-1} \right)^\top \left( \frac{ \vg_t } { \sqrt{ {\vv}_t + \epsilon^2 } } \right) \right] 
    = \sum_{i=1}^D \mathbb{E} \left[ \frac{ \left( \nabla f \left( \vtheta_{t-1} \right) \right)_i \left( \vg_t \right)_i } { \sqrt{ \left( {\vv}_t \right)_i + \epsilon^2 } } \right]
\end{align}
We define $\tilde{\vv}_t$ as follows:
\begin{align}
    \tilde{\vv}_t = \beta_2 \vv_{t-1} + \left( 1 - \beta_2 \right) \mathbb{E} \left[ \vg_t \odot \vg_t \right]
\end{align}
Using this, the following holds:
\begin{align}
    &\mathbb{E} \left[ \frac{ \left( \nabla f \left( \vtheta_{t-1} \right) \right)_i \left( \vg_t \right)_i } { \sqrt{ \left( {\vv}_t \right)_i + \epsilon^2 } } \right] \nonumber \\
    &= \mathbb{E} \left[  \frac{ \left( \nabla f \left( \vtheta_{t-1} \right) \right)_i \left( \vg_t \right)_i } { \sqrt{ \left( \tilde{\vv}_t \right)_i + \epsilon^2 } } \right] + \mathbb{E} \left[ \left( \nabla f \left( \vtheta_{t-1} \right) \right)_i \left( \vg_t \right)_i \left( \frac{ 1 } { \sqrt{ \left( {\vv}_t \right)_i + \epsilon^2 } } - \frac{ 1 } { \sqrt{ \left( \tilde{\vv}_t \right)_i + \epsilon^2 } } \right) \right]\\
    &= \mathbb{E} \left[ \frac{ \left( \nabla f \left( \vtheta_{t-1} \right) \right)_i^2 } { \sqrt{ \left( \tilde{\vv}_t \right)_i + \epsilon^2 } } \right] + \mathbb{E} \left[ \left( \nabla f \left( \vtheta_{t-1} \right) \right)_i \left( \vg_t \right)_i \left( \frac{ 1 } { \sqrt{ \left( {\vv}_t \right)_i + \epsilon^2 } } - \frac{ 1 } { \sqrt{ \left( \tilde{\vv}_t \right)_i + \epsilon^2 } } \right) \right] \\
    &\geq \mathbb{E} \left[ \frac{ \left( \nabla f \left( \vtheta_{t-1} \right) \right)_i^2 } { \sqrt{ \left( \tilde{\vv}_t \right)_i + \epsilon^2 } } \right] - \mathbb{E} \left[ \left| \left( \nabla f \left( \vtheta_{t-1} \right) \right)_i \left( \vg_t \right)_i \left( \frac{ 1 } { \sqrt{ \left( {\vv}_t \right)_i + \epsilon^2 } } - \frac{ 1 } { \sqrt{ \left( \tilde{\vv}_t \right)_i + \epsilon^2 } } \right) \right| \right] ,
\end{align}
where the last inequality holds due to $A \geq - | A |$.
For the second term, the following holds:
\begin{align}
    &\left| \left( \nabla f \left( \vtheta_{t-1} \right) \right)_i \left( \vg_t \right)_i \left( \frac{ 1 } { \sqrt{ \left( {\vv}_t \right)_i + \epsilon^2 } } - \frac{ 1 } { \sqrt{ \left( \tilde{\vv}_t \right)_i + \epsilon^2 } } \right) \right| \nonumber \\
    &= \left( 1 - \beta_2 \right) \left| \left( \nabla f \left( \vtheta_{t-1} \right) \right)_i \left( \vg_t \right)_i \frac{ \mathbb{E} \left[ \left( \vg_t \right)_i^2 \right] - \left( \vg_t \right)_i^2 }{ \sqrt{ \left( {\vv}_t \right)_i + \epsilon^2 } \sqrt{ \left( \tilde{\vv}_t \right)_i + \epsilon^2 } \left( \sqrt{ \left( {\vv}_t \right)_i + \epsilon^2 } + \sqrt{ \left( \tilde{\vv}_t \right)_i + \epsilon^2 }\right) } \right| \\
    &\leq \left( 1 - \beta_2 \right) \left( \frac{ \left| \left( \nabla f \left( \vtheta_{t-1} \right) \right)_i \left( \vg_t \right)_i \right| \mathbb{E} \left[ \left( \vg_t \right)_i^2 \right] }{ \sqrt{ \left( {\vv}_t \right)_i + \epsilon^2 } \left( \left( \tilde{\vv}_t \right)_i + \epsilon^2 \right) } + \frac{ \left| \left( \nabla f \left( \vtheta_{t-1} \right) \right)_i \left( \vg_t \right)_i \right| \left( \vg_t \right)_i^2 }{ \left( \left( {\vv}_t \right)_i + \epsilon^2 \right) \sqrt{ \left( \tilde{\vv}_t \right)_i + \epsilon^2 } } \right), \label{eq:89}
\end{align}
where the last inequality holds due to the triangle inequality.
For the first term, the following holds:
\begin{align}
    &\mathbb{E} \left[ \frac{ \left| \left( \nabla f \left( \vtheta_{t-1} \right) \right)_i \left( \vg_t \right)_i \right| \mathbb{E} \left[ \left( \vg_t \right)_i^2 \right] }{ \sqrt{ \left( {\vv}_t \right)_i + \epsilon^2 } \left( \left( \tilde{\vv}_t \right)_i + \epsilon^2 \right) } \right] \nonumber \\
    &\leq \frac{1}{ \left( 1 - \beta_2 \right) } \mathbb{E} \left[ \frac{ \left( \nabla f \left( \vtheta_{t-1} \right) \right)_i^2 }{ 4 \sqrt{ \left( \tilde{\vv}_t \right)_i + \epsilon^2 } } \right] + \left( 1 - \beta_2 \right) \mathbb{E} \left[ \frac{ \left( \vg_t \right)_i^2 \mathbb{E} \left[ \left( \vg_t \right)_i^2 \right]^2 }{ \left( \left( {\vv}_t \right)_i + \epsilon^2 \right) \left( \left( \tilde{\vv}_t \right)_i + \epsilon^2 \right)^{3/2} } \right] \\
    &\leq \frac{1}{ \left( 1 - \beta_2 \right) } \mathbb{E} \left[ \frac{ \left( \nabla f \left( \vtheta_{t-1} \right) \right)_i^2 }{ 4 \sqrt{ \left( \tilde{\vv}_t \right)_i + \epsilon^2 } } \right] + \mathbb{E} \left[ \frac{ \left( \vg_t \right)_i^2 \sqrt{ \mathbb{E} \left[ \left( \vg_t \right)_i^2 \right] } }{ \sqrt{ 1 - \beta_2 } \left( \left( {\vv}_t \right)_i + \epsilon^2 \right) } \right]\\
    &\leq \frac{1}{ \left( 1 - \beta_2 \right) } \mathbb{E} \left[ \frac{ \left( \nabla f \left( \vtheta_{t-1} \right) \right)_i^2 }{ 4 \sqrt{ \left( \tilde{\vv}_t \right)_i + \epsilon^2 } } \right]+ \frac{G}{ \sqrt{ 1 - \beta_2 } } \mathbb{E} \left[ \frac{ \left( \vg_t \right)_i^2 }{ \left( {\vv}_t \right)_i + \epsilon^2 } \right]
\end{align}
The first inequality is derived using the following fact:
\begin{align}
    \forall \lambda>0, x, y \in \mathbb{R}, x y \leq \frac{\lambda}{2} x^2+\frac{y^2}{2 \lambda}. \label{eq:generalized_mean_ineq}
\end{align}
For the second term of Eq. (\ref{eq:89}), the following holds:
\begin{align}
    &\mathbb{E} \left[ \frac{ \left| \left( \nabla f \left( \vtheta_{t-1} \right) \right)_i \left( \vg_t \right)_i \right| \left( \vg_t \right)_i^2 }{ \left( \left( {\vv}_t \right)_i + \epsilon^2 \right) \sqrt{ \left( \tilde{\vv}_t \right)_i + \epsilon^2 } } \right] \\
    &\leq \frac{1}{ \left( 1 - \beta_2 \right) } \mathbb{E} \left[ \frac{ \left( \nabla f \left( \vtheta_{t-1} \right) \right)_i^2 }{ 4 \sqrt{ \left( \tilde{\vv}_t \right)_i + \epsilon^2 } } \frac{ \left( \vg_t \right)_i^2 }{ \mathbb{E} \left[ \left( \vg_t \right)_i^2 \right] } \right] + \left( 1 - \beta_2 \right) \mathbb{E} \left[ \frac{ \mathbb{E} \left[ \left( \vg_t \right)_i^2 \right] }{ \sqrt{ \left( \tilde{\vv}_t \right)_i + \epsilon^2 } } \frac{ \left( \vg_t \right)_i^4 }{ \left( \left( \tilde{\vv}_t \right)_i + \epsilon^2 \right)^2 } \right] \\
    &\leq \frac{1}{ \left( 1 - \beta_2 \right) } \mathbb{E} \left[ \frac{ \left( \nabla f \left( \vtheta_{t-1} \right) \right)_i^2 }{ 4 \sqrt{ \left( \tilde{\vv}_t \right)_i + \epsilon^2 } } \right] + \mathbb{E} \left[ \frac{ \sqrt{ \mathbb{E} \left[ \left( \vg_t \right)_i^2 \right] } \left( \vg_t \right)_i^2 }{ \sqrt{ 1 - \beta_2 } \left( \left( \tilde{\vv}_t \right)_i + \epsilon^2 \right) } \right] \\
    &\leq \frac{1}{ \left( 1 - \beta_2 \right) } \mathbb{E} \left[ \frac{ \left( \nabla f \left( \vtheta_{t-1} \right) \right)_i^2 }{ 4 \sqrt{ \left( \tilde{\vv}_t \right)_i + \epsilon^2 } } \right] + \frac{G}{ \sqrt{ 1 - \beta_2 } } \mathbb{E} \left[ \frac{ \left( \vg_t \right)_i^2 }{ \left( {\vv}_t \right)_i + \epsilon^2 } \right]
\end{align}
The first inequality is derived using Eq. (\ref{eq:generalized_mean_ineq}).

Putting these inequalities together, the following is derived:
\begin{align}
    &\mathbb{E} \left[ \nabla f \left( \vtheta_{t-1} \right)^\top \left( \frac{ \vg_t } { \sqrt{ {\vv}_t + \epsilon^2 } } \right) \right] \nonumber \\
    &\geq \sum_{i=1}^D \mathbb{E} \left[ \frac{ \left( \nabla f \left( \vtheta_{t-1} \right) \right)_i^2 } { 2 \sqrt{ \left( \tilde{\vv}_t \right)_i + \epsilon^2 } } \right] - 2G \sqrt{ 1 - \beta_2 } \mathbb{E} \left[ \frac{ \left( \vg_t \right)_i^2 }{ \left( {\vv}_t \right)_i + \epsilon^2 } \right]  \\
    &\geq \frac{1}{2} \mathbb{E} \left[ \nabla f \left( \vtheta_{t-1} \right)^\top \left( \frac{ \vg_t } { \sqrt{ \tilde{\vv}_t + \epsilon^2 } } \right) \right] - 2G \sqrt{ 1 - \beta_2 } \mathbb{E} \left[ \left\| \frac{ \vg_t  }{ \sqrt{ {\vv}_t + \epsilon^2 } } \right\|^2 \right].
\end{align}

\end{proof}
\end{lemma}

\begin{lemma}
\label{lem:telescope_v}
For the \rmsprop{} algorithm, the following holds:
\begin{align}
    \sum_{t=1}^T \mathbb{E} \left[ \left\| \frac{ \vg_t  }{ \sqrt{ {\vv}_t + \epsilon^2 } } \right\|^2 \right]
    \leq D \left( \log \left( 1 + \frac{ \left( 1 - \beta_2^T \right) G^2 }{ \epsilon^2 } \right) - T \log \beta_2 \right)
\end{align}
\begin{proof}

\begin{align}
    \left\| \frac{ \vg_t  }{ \sqrt{ {\vv}_t + \epsilon^2 } } \right\|^2
    = \sum_{i=1}^D \frac{ \left( \vg_t \right)_i^2 }{ \left( {\vv}_t \right)_i + \epsilon^2  }
\end{align}

\begin{align}
    \frac{ \left( \vg_t \right)_i^2 }{ \left( {\vv}_t \right)_i + \epsilon^2  }
    &= \frac{1}{ 1 - \beta_2 } \frac{ \left( 1 - \beta_2 \right) \left( \vg_t \right)_i^2 }{ \left( {\vv}_t \right)_i + \epsilon^2  } \\
    &\leq - \frac{1}{1 - \beta_2} \log \left( 1 - \frac{ \left( 1 - \beta_2 \right) \left( \vg_t \right)_i^2 }{ \left( {\vv}_t \right)_i + \epsilon^2  } \right) \\
    &= \frac{1}{1 - \beta_2} \log \left( \frac{ \left( {\vv}_t \right)_i + \epsilon^2  }{ \beta_2 \left( {\vv}_{t-1} \right)_i + \epsilon^2 } \right) \\
    &= \frac{1}{1 - \beta_2} \left( \log \left( \frac{ \left( {\vv}_t \right)_i + \epsilon^2  }{ \left( {\vv}_{t-1} \right)_i + \epsilon^2 } \right) + \log \left( \frac{ \left( {\vv}_{t-1} \right)_i + \epsilon^2  }{ \beta_2 \left( {\vv}_{t-1} \right)_i + \epsilon^2 } \right) \right) \\
    &\leq \frac{1}{1 - \beta_2} \left( \log \left( \frac{ \left( {\vv}_t \right)_i + \epsilon^2  }{ \left( {\vv}_{t-1} \right)_i + \epsilon^2 } \right) - \log \beta_2 \right)
\end{align}

\begin{align}
    \sum_{t=1}^T \frac{ \left( \vg_t \right)_i^2 }{ \left( {\vv}_t \right)_i + \epsilon^2  } 
    &\leq \frac{1}{1 - \beta_2} \left( \log \left( \frac{ \left( {\vv}_T \right)_i + \epsilon^2  }{ \epsilon^2 } \right) - T \log \beta_2 \right) \\
    &\leq \frac{1}{1 - \beta_2} \left( \log \left( 1 + \frac{ \left( 1 - \beta_2^T \right) G^2 }{ \epsilon^2 } \right) - T \log \beta_2 \right)
\end{align}

\begin{align}
    \sum_{t=1}^T \mathbb{E} \left[ \left\| \frac{ \vg_t  }{ \sqrt{ {\vv}_t + \epsilon^2 } } \right\|^2 \right]
    &\leq \sum_{i=1}^D \mathbb{E} \left[ \sum_{t=1}^T \frac{ \left( \vg_t \right)_i^2 }{ \left( {\vv}_t \right)_i + \epsilon^2 } \right] \\
    &\leq \frac{1}{1 - \beta_2} \sum_{i=1}^D \mathbb{E} \left[ \log \left( 1 + \frac{ \left( {\vv}_T \right)_i }{ \epsilon^2 } \right) \right] - \frac{ D T \log \beta_2 }{1 - \beta_2} \\
    &\leq \sum_{i=1}^D \log \left( 1 + \frac{ \mathbb{E} \left[ \left( {\vv}_T \right)_i \right] }{ \epsilon^2 } \right) - \frac{ D T \log \beta_2 }{1 - \beta_2} \\
    &\leq \frac{ D } { 1 - \beta_2 } \left( \log \left( 1 + \frac{ \left( 1 - \beta_2^T \right) G^2 }{ \epsilon^2 } \right) - T \log \beta_2 \right)
\end{align}

\end{proof}
\end{lemma}

\begin{lemma}
\label{lem:holder}
For the \rmsprop{} algorithm, the following holds:
\begin{align}
    \mathbb{E} \left[ \nabla f \left( \vtheta_{t-1} \right)^\top \left( \frac{ \vg_t } { \sqrt{ \beta_2 \tilde{\vv}_{t} + \epsilon^2 } } \right)  \right] 
    \geq \frac{\mathbb{E} \left[ \left\| \nabla f \left( \vtheta_{t-1} \right) \right\|^{4/3} \right]^{3/2} } { \sqrt{ \left( 1 - \beta_2^t \right) G^2 + \epsilon^2 } }
\end{align}
\begin{proof}
\begin{align}
    &\mathbb{E} \left[ \nabla f \left( \vtheta_{t-1} \right)^\top \left( \frac{ \vg_t } { \sqrt{ \tilde{\vv}_{t} + \epsilon^2 } } \right)  \right] \nonumber\\
    &= \sum_{i=1}^D \mathbb{E} \left[  \frac{ \left( \nabla f \left( \vtheta_{t-1} \right) \right)_i \cdot \left( \vg_t \right)_i } { \sqrt{ \left( \tilde{\vv}_{t} \right)_i + \epsilon^2 } }  \right] \nonumber\\
    &= \sum_{i=1}^D \mathbb{E} \left[  \frac{ \left( \nabla f \left( \vtheta_{t-1} \right) \right)_i^2 } { \sqrt{ \beta_2 \left( {\vv}_{t-1} \right)_i + \epsilon^2 } }  \right] \nonumber\\
    &\geq \mathbb{E} \left[  \frac{ \left\| \nabla f \left( \vtheta_{t-1} \right) \right\|^2 } { \sqrt{ \sum_{i=1}^D \left( \tilde{\vv}_{t} \right)_i + \epsilon^2 } }  \right] \nonumber\\
    &\geq \frac{ \mathbb{E} \left[  \left\| \nabla f \left( \vtheta_{t-1} \right) \right\|^{4/3} \right]^{3/2} } { \sqrt{ \mathbb{E} \left[ \sum_{i=1}^D \left( \tilde{\vv}_{t} \right)_i \right] + \epsilon^2 } }  \nonumber\\
    &\geq \frac{\mathbb{E} \left[ \left\| \nabla f \left( \vtheta_{t-1} \right) \right\|^{4/3} \right]^{3/2} } { \sqrt{ \left( 1 - \beta_2^t \right) G^2 + \epsilon^2 } }.
\end{align}
The second equality holds due to Assumption \ref{ass:unbiased}.
The first inequality holds because $\left( \tilde{\vv}_{t} \right)_i \geq 0$ for all $i = 1, \ldots, D$.
The second inequality holds due to the Hölder's inequality.
The last inequality holds due to Lemma \ref{lem:bound_exp_v_rmsprop}.
\end{proof}
\end{lemma}

\begin{lemma}
\label{lem:phi}

For the clipped \adopt{} algorithm, the following holds for $t \geq 1$:
\begin{align}
    \vphi_{t} - \vphi_{t-1} 
    &= \frac{ \left( \alpha_{t-1} - \alpha_{t} \right) \beta_1 }{ 1 - \beta_1 } \vm_{t-1} - \alpha_{t} \mathrm{Clip} \left( \frac{ \vg_{t} }{ \max \left\{ \sqrt{ \vv_{t-1} } , \epsilon \right\} } , c_t \right),
\end{align}
where we define $\alpha_{0} = \alpha$.

\begin{proof}

For $t = 1$, the following holds by definition:
\begin{align}
    \vphi_1 - \vphi_0
    &= \frac{ 1 }{ 1 - \beta_1 } \vtheta_1 - \left( \frac{ \beta_1 }{ 1 - \beta_1 } + 1 \right) \vtheta_{0} \\
    &= \frac{ 1 }{ 1 - \beta_1 } \left( \vtheta_1 - \vtheta_0 \right) \\
    &= - \alpha_1 \cdot \mathrm{Clip} \left( \frac{ \vg_1 }{ \max \left\{ \sqrt{ \vv_0 } ,  \epsilon \right\} } , c_t \right).
\end{align}

For $t \geq 2$, the following holds:
\begin{align}
    &\vphi_{t} - \vphi_{t-1} \\
    &= \frac{ 1 }{ 1 - \beta_1 } \left( \vtheta_{t} - \vtheta_{t-1} \right) - \frac{ \beta_1 }{ 1 - \beta_1 } \left( \vtheta_{t-1} - \vtheta_{t-2} \right) \\
    &= \frac{ 1 }{ 1 - \beta_1 } \left( \alpha_{t-1} \beta_1 \vm_{t-1} - \alpha_{t} \vm_{t} \right) \\
    &= \frac{ 1 }{ 1 - \beta_1 } \left( \alpha_{t-1} \beta_1 \vm_{t-1} - \alpha_{t} \left( \beta_1 \vm_{t-1} + \left( 1 - \beta_1 \right) \mathrm{Clip} \left( \frac{ \vg_{t} }{ \max \left\{ \sqrt{ \vv_{t-1} } , \epsilon \right\} } , c_t \right) \right) \right) \\
    &= \frac{ 1 }{ 1 - \beta_1 } \left( \left( \alpha_{t-1} - \alpha_{t} \right) \beta_1 \vm_{t-1} - \alpha_{t} \left( 1 - \beta_1 \right) \mathrm{Clip} \left( \frac{ \vg_{t} }{ \max \left\{ \sqrt{ \vv_{t-1} } , \epsilon \right\} } , c_t \right) \right) \\
    &= \frac{ \left( \alpha_{t-1} - \alpha_{t} \right) \beta_1 }{ 1 - \beta_1 } \vm_{t-1} - \alpha_{t} \mathrm{Clip} \left( \frac{ \vg_{t} }{ \max \left\{ \sqrt{ \vv_{t-1} } , \epsilon \right\} } , c_t \right)
\end{align}
\end{proof}

\end{lemma}

\begin{lemma}
\label{lem:phi_theta}

For the \adopt{} algorithm, the following holds for $t \geq 1$:
\begin{align}
    \vphi_{t-1} - \vtheta_{t-1} 
    &= - \frac{ \alpha_{t-1} \beta_1 }{ 1 - \beta_1 } \vm_{t-1}. \label{eq:phi_theta}
\end{align}

\begin{proof}

For $t = 1$, Eq. (\ref{eq:phi_theta}) holds obviously because $\vphi_0 = \vtheta_0$ and $\vm_0 = \vzero$.
For $t \geq 2$, the following holds:
\begin{align}
    \vphi_{t-1} - \vtheta_{t-1} 
    &= \left( \frac{ 1 }{ 1 - \beta_1 } - 1 \right) \vtheta_{t-1} - \frac{ \beta_1 }{ 1 - \beta_1 } \vtheta_{t-2} \\
    &= \frac{ \beta_1 }{ 1 - \beta_1 } \left( \vtheta_{t-1} - \vtheta_{t-2} \right) \\
    &= - \frac{ \alpha_{t-1} \beta_1 }{ 1 - \beta_1 } \vm_{t-1}.
\end{align}
    
\end{proof}

\end{lemma}

\begin{lemma}
\label{lem:bound_dot_prod}
For the clipped \adopt{} algorithm, the following holds for $t \geq 1$:
\begin{align}
    &\mathbb{E} \left[ \nabla f \left( \vtheta_{t-1} \right)^\top \left( \vphi_{t} - \vphi_{t-1} \right) \right] \nonumber \\
    &\leq \frac{ \sqrt{2} \left( \alpha_{t-1} - \alpha_{t} \right) \beta_1 \left( 1 - \beta_1^{t-1} \right) G^2 }{ \left( 1 - \beta_1 \right) \epsilon } - \frac{ \alpha_{t} }{2} \frac{ \mathbb{E} \left[ \left\| \nabla f \left( \vtheta_{t-1} \right) \right\|_i^{4/3} \right]^{3/2} }{ \sqrt{ G^2 + \epsilon^2 } } + \frac{\alpha_t G^4}{2 \epsilon^3 c_t^2}.
\end{align}

\begin{proof}

\begin{align}
    &\nabla f \left( \vtheta_{t-1} \right)^\top \left( \vphi_{t} - \vphi_{t-1} \right) \nonumber \\ 
    &= \frac{ \left( \alpha_{t-1} - \alpha_{t} \right) \beta_1 }{ 1 - \beta_1 } \nabla f \left( \vtheta_{t-1} \right) )^\top \vm_{t-1} - \alpha_{t} \nabla f \left( \vtheta_{t-1} \right)^\top \mathrm{Clip} \left( \frac{ \vg_{t} }{ \max \left\{ \sqrt{ \vv_{t-1} } , \epsilon \right\} } , c_t \right) \\
    &\leq \frac{ \left( \alpha_{t-1} - \alpha_{t} \right) \beta_1 }{ 1 - \beta_1 } \left\| \nabla f \left( \vtheta_{t-1} \right) ) \right\| \left\| \vm_{t-1} \right\| - \alpha_{t} \nabla f \left( \vtheta_{t-1} \right)^\top \mathrm{Clip} \left( \frac{ \vg_{t} }{ \max \left\{ \sqrt{ \vv_{t-1} } , \epsilon \right\} } , c_t \right) \\
    &\leq \frac{ \left( \alpha_{t-1} - \alpha_{t} \right) \beta_1 G }{ 1 - \beta_1 } \left\| \vm_{t-1} \right\| - \alpha_{t} \nabla f \left( \vtheta_{t-1} \right)^\top \mathrm{Clip} \left( \frac{ \vg_{t} }{ \max \left\{ \sqrt{ \vv_{t-1} } , \epsilon \right\} } , c_t \right) .
\end{align}
By taking the expectation for both sides, the following holds:
\begin{align}
    &\mathbb{E} \left[ \nabla f \left( \vtheta_{t-1} \right)^\top \cdot \left( \vphi_t - \vphi_{t-1} \right) \right] \nonumber\\ 
    &\leq \frac{ \left( \alpha_{t-1} - \alpha_{t} \right) \beta_1 G }{ 1 - \beta_1 } \mathbb{E} \left[ \left\| \vm_{t-1} \right\| \right] - \alpha_{t} \mathbb{E} \left[ \nabla f \left( \vtheta_{t-1} \right)^\top \mathrm{Clip} \left( \frac{ \vg_{t} }{ \max \left\{ \sqrt{ \vv_{t-1} } , \epsilon \right\} } , c_t \right) \right] \\
    &\leq \frac{ \left( \alpha_{t-1} - \alpha_{t} \right) \beta_1 G }{ 1 - \beta_1 } \mathbb{E} \left[ \left\| \vm_{t-1} \right\| \right] - \frac{\alpha_{t}}{2} \sum_{i=1}^D \mathbb{E} \left[ \frac{ \left( \nabla f \left( \vtheta_{t-1} \right) \right)_i \cdot \left( \vg_{t} \right)_i }{ \max \left\{ \sqrt{ \left( \vv_{t-1} \right)_i } , \epsilon \right\} } \right] + \frac{\alpha_t G^4}{2 \epsilon^3 c_t^2} \\
    &\leq \frac{ \left( \alpha_{t-1} - \alpha_{t} \right) \beta_1 G }{ 1 - \beta_1 } \mathbb{E} \left[ \left\| \vm_{t-1} \right\| \right] - \frac{ \alpha_{t} }{2} \sum_{i=1}^D \mathbb{E} \left[ \frac{ \left( \nabla f \left( \vtheta_{t-1} \right) \right)_i^2 }{ \max \left\{ \sqrt{ \left( \vv_{t-1} \right)_i } , \epsilon \right\} } \right] + \frac{\alpha_t G^4}{2 \epsilon^3 c_t^2} \\
    &\leq \frac{ \left( \alpha_{t-1} - \alpha_{t} \right) \beta_1 G }{ 1 - \beta_1 } \mathbb{E} \left[ \left\| \vm_{t-1} \right\| \right] - \frac{ \alpha_{t} }{2} \sum_{i=1}^D \mathbb{E} \left[ \frac{ \left( \nabla f \left( \vtheta_{t-1} \right) \right)_i^2 }{ \sqrt{ \left( \vv_{t-1} \right)_i + \epsilon^2 } } \right] + \frac{\alpha_t G^4}{2 \epsilon^3 c_t^2} \\
    &\leq \frac{ \left( \alpha_{t-1} - \alpha_{t} \right) \beta_1 G }{ 1 - \beta_1 } \mathbb{E} \left[ \left\| \vm_{t-1} \right\| \right] - \frac{ \alpha_{t} }{2} \mathbb{E} \left[ \frac{ \left\| \nabla f \left( \vtheta_{t-1} \right) \right\|^2 }{ \sqrt{ \sum_{i=1}^D \left( \vv_{t-1} \right)_i + \epsilon^2 } } \right] + \frac{\alpha_t G^4}{2 \epsilon^3 c_t^2} \\
    &\leq \frac{ \left( \alpha_{t-1} - \alpha_{t} \right) \beta_1 G }{ 1 - \beta_1 } \mathbb{E} \left[ \left\| \vm_{t-1} \right\| \right] - \frac{ \alpha_{t} }{2} \frac{ \mathbb{E} \left[ \left\| \nabla f \left( \vtheta_{t-1} \right) \right\|_i^{4/3} \right]^{3/2} }{ \sqrt{ \mathbb{E} \left[ \sum_{i=1}^D \left( \vv_{t-1} \right)_i \right] + \epsilon^2 } } + \frac{\alpha_t G^4}{2 \epsilon^3 c_t^2} \\
    &\leq \frac{ \left( \alpha_{t-1} - \alpha_{t} \right) \beta_1 G }{ 1 - \beta_1 } \mathbb{E} \left[ \left\| \vm_{t-1} \right\| \right] - \frac{ \alpha_{t} }{2} \frac{ \mathbb{E} \left[ \left\| \nabla f \left( \vtheta_{t-1} \right) \right\|_i^{4/3} \right]^{3/2} }{ \sqrt{ G^2 + \epsilon^2 } } + \frac{\alpha_t G^4}{2 \epsilon^3 c_t^2} \\
    &\leq \frac{ \sqrt{2} \left( \alpha_{t-1} - \alpha_{t} \right) \beta_1 \left( 1 - \beta_1^{t-1} \right) G^2 }{ \left( 1 - \beta_1 \right) \epsilon } - \frac{ \alpha_{t} }{2} \frac{ \mathbb{E} \left[ \left\| \nabla f \left( \vtheta_{t-1} \right) \right\|_i^{4/3} \right]^{3/2} }{ \sqrt{ G^2 + \epsilon^2 } } + \frac{\alpha_t G^4}{2 \epsilon^3 c_t^2}.
\end{align}

\end{proof}

\end{lemma}

\begin{lemma}
\label{lem:bound_dp_clip}
For the clipped \adopt{} algorithm, the following holds for $t \geq 0$:
\begin{align}
    &2 \nabla f \left( \vtheta_{t-1} \right)^\top \mathbb{E} \left[ \mathrm{Clip} \left( \frac{ \vg_{t} }{ \max \left\{ \sqrt{ \vv_{t-1} } , \epsilon \right\} } , c_t \right) \right] \\
    &\geq \sum_{i=1}^D \frac{ \left( \nabla f \left( \vtheta_{t-1} \right) \right)_i^2 }{{ \max \left\{ \sqrt{ \left( \vv_{t-1} \right)_i } , \epsilon \right\} } } - \frac{G^4}{ \epsilon^3 c_t^2 }
\end{align}
\begin{proof}
\begin{align}
    &\mathbb{E} \left[ \mathrm{Clip} \left( \frac{ \left( \vg_{t} \right)_i }{ \max \left\{ \sqrt{ \left( \vv_{t-1} \right)_i } , \epsilon \right\} } , c_t \right) \right] \\
    &= \mathbb{E} \left[ \left( \delta_{t,i} \cdot \frac{ c_t }{ \left| \left( \vg_{t} \right)_i / { \max \left\{ \sqrt{ \left( \vv_{t-1} \right)_i } , \epsilon \right\} } \right| } + \left( 1 - \delta_{t,i} \right) \right) \cdot \frac{ \left( \vg_{t} \right)_i }{ \max \left\{ \sqrt{ \left( \vv_{t-1} \right)_i } , \epsilon \right\} } \right] \\
    &= \frac{ \left( \nabla f \left( \vtheta_{t-1} \right) \right)_i }{ { \max \left\{ \sqrt{ \left( \vv_{t-1} \right)_i } , \epsilon \right\} } } - \mathbb{E} \left[ \delta_{t,i} \left( 1 - \frac{ c_t }{ \left| \left( \vg_{t} \right)_i / { \max \left\{ \sqrt{ \left( \vv_{t-1} \right)_i } , \epsilon \right\} } \right| } \right) \cdot \frac{ \left( \vg_{t} \right)_i }{ \max \left\{ \sqrt{ \left( \vv_{t-1} \right)_i } , \epsilon \right\} } \right] ,
\end{align}
where $\delta_{t,i}$ is an indicator function of whether the normalized gradient is clipped:
\begin{align}
    \delta_{t,i} = \vone_{\left\{\left| { \left( \vg_{t} \right)_i } / { \max \left\{ \sqrt{ \left( \vv_{t-1} \right)_i } , \epsilon \right\} } \right| > c_t \right\}} .
\end{align}

Using this equation, the following inequality is derived:
\begin{align}
    &\left( \frac{ \left( \nabla f \left( \vtheta_{t-1} \right) \right)_i }{ { \max \left\{ \sqrt{ \left( \vv_{t-1} \right)_i } , \epsilon \right\} } } - \mathbb{E} \left[ \mathrm{Clip} \left( \frac{ \left( \vg_{t} \right)_i }{ \max \left\{ \sqrt{ \left( \vv_{t-1} \right)_i } , \epsilon \right\} } , c_t \right) \right] \right)^2 \\
    &= \mathbb{E} \left[ \delta_{t,i} \left( 1 - \frac{ c_t }{ \left| \left( \vg_{t} \right)_i / { \max \left\{ \sqrt{ \left( \vv_{t-1} \right)_i } , \epsilon \right\} } \right| } \right) \cdot \frac{ \left( \vg_{t} \right)_i }{ \max \left\{ \sqrt{ \left( \vv_{t-1} \right)_i } , \epsilon \right\} } \right]^2 \\
    &= \mathbb{E} \left[ \delta_{t,i} \right]^2 \mathbb{E} \left[ \left( 1 - \frac{ c_t }{ \left| \left( \vg_{t} \right)_i / { \max \left\{ \sqrt{ \left( \vv_{t-1} \right)_i } , \epsilon \right\} } \right| } \right) \cdot \frac{ \left( \vg_{t} \right)_i }{ \max \left\{ \sqrt{ \left( \vv_{t-1} \right)_i } , \epsilon \right\} } \mid \delta_{t,i} = 1 \right]^2 \\
    &\leq \mathbb{E} \left[ \delta_{t,i} \right]^2 \mathbb{E} \left[ \left( 1 - \frac{ c_t }{ \left| \left( \vg_{t} \right)_i / { \max \left\{ \sqrt{ \left( \vv_{t-1} \right)_i } , \epsilon \right\} } \right| } \right)^2 \cdot \frac{ \left( \vg_{t} \right)_i^2 }{ \max \left\{ \left( \vv_{t-1} \right)_i , \epsilon^2 \right\} } \mid \delta_{t,i} = 1 \right] \\
    &\leq \mathbb{E} \left[ \delta_{t,i} \right]^2 \mathbb{E} \left[ \frac{ \left( \vg_{t} \right)_i^2 }{ \max \left\{ \left( \vv_{t-1} \right)_i , \epsilon^2 \right\} } \mid \delta_{t,i} = 1 \right] \\
    &= \mathbb{E} \left[ \delta_{t,i} \right] \mathbb{E} \left[ \frac{ \left( \vg_{t} \right)_i^2 }{ \max \left\{ \left( \vv_{t-1} \right)_i , \epsilon^2 \right\} } \right] \\
    &= \mathrm{Pr} \left[ \left| \frac{ \left( \vg_{t} \right)_i }{ \max \left\{ \sqrt{ \left( \vv_{t-1} \right)_i } , \epsilon \right\} } \right| > c_t \right] \mathbb{E} \left[ \frac{ \left( \vg_{t} \right)_i^2 }{ \max \left\{ \left( \vv_{t-1} \right)_i , \epsilon^2 \right\} } \right] \\
    &\leq \frac{1}{c_t^2} \mathbb{E} \left[ \frac{ \left( \vg_{t} \right)_i^2 }{ \max \left\{ \left( \vv_{t-1} \right)_i , \epsilon^2 \right\} } \right]^2 ,
\end{align}
where the first inequality is due to the Jensen's inequality, and the last inequality is due to the Chebyshev's inequality.
Therefore, the following holds:
\begin{align}
    &2 \left( \nabla f \left( \vtheta_{t-1} \right) \right)_i \cdot \mathbb{E} \left[ \mathrm{Clip} \left( \frac{ \left( \vg_{t} \right)_i }{ \max \left\{ \sqrt{ \left( \vv_{t-1} \right)_i } , \epsilon \right\} } , c_t \right) \right] \\
    &= \frac{ \left( \nabla f \left( \vtheta_{t-1} \right) \right)_i^2 }{{ \max \left\{ \sqrt{ \left( \vv_{t-1} \right)_i } , \epsilon \right\} } } + { \max \left\{ \sqrt{ \left( \vv_{t-1} \right)_i } , \epsilon \right\} } \mathbb{E} \left[ \mathrm{Clip} \left( \frac{ \left( \vg_{t} \right)_i }{ \max \left\{ \sqrt{ \left( \vv_{t-1} \right)_i } , \epsilon \right\} } , c_t \right) \right]^2 \\
    &\quad - { \max \left\{ \sqrt{ \left( \vv_{t-1} \right)_i } , \epsilon \right\} } \left( \frac{ \left( \nabla f \left( \vtheta_{t-1} \right) \right)_i }{ { \max \left\{ \sqrt{ \left( \vv_{t-1} \right)_i } , \epsilon \right\} } } - \mathbb{E} \left[ \mathrm{Clip} \left( \frac{ \left( \vg_{t} \right)_i }{ \max \left\{ \sqrt{ \left( \vv_{t-1} \right)_i } , \epsilon \right\} } , c_t \right) \right] \right)^2 \nonumber \\
    &\geq \frac{ \left( \nabla f \left( \vtheta_{t-1} \right) \right)_i^2 }{{ \max \left\{ \sqrt{ \left( \vv_{t-1} \right)_i } , \epsilon \right\} } } \\
    &\quad - { \max \left\{ \sqrt{ \left( \vv_{t-1} \right)_i } , \epsilon \right\} } \left( \frac{ \left( \nabla f \left( \vtheta_{t-1} \right) \right)_i }{ { \max \left\{ \sqrt{ \left( \vv_{t-1} \right)_i } , \epsilon \right\} } } - \mathbb{E} \left[ \mathrm{Clip} \left( \frac{ \left( \vg_{t} \right)_i }{ \max \left\{ \sqrt{ \left( \vv_{t-1} \right)_i } , \epsilon \right\} } , c_t \right) \right] \right)^2 \nonumber \\
    &\geq \frac{ \left( \nabla f \left( \vtheta_{t-1} \right) \right)_i^2 }{{ \max \left\{ \sqrt{ \left( \vv_{t-1} \right)_i } , \epsilon \right\} } } - \frac{1}{c_t^2 \left( \max \left\{ \sqrt{ \left( \vv_{t-1} \right)_i } , \epsilon \right\} \right)^3 } \mathbb{E} \left[ \left( \vg_t \right)_i^2 \right]^2 \\
    &\geq \frac{ \left( \nabla f \left( \vtheta_{t-1} \right) \right)_i^2 }{{ \max \left\{ \sqrt{ \left( \vv_{t-1} \right)_i } , \epsilon \right\} } } - \frac{1}{\epsilon^3 c_t^2 } \mathbb{E} \left[ \left( \vg_t \right)_i^2 \right]^2
\end{align}

By summing it up for all the dimensions, the following inequality is derived.
\begin{align}
    &2 \nabla f \left( \vtheta_{t-1} \right)^\top \mathbb{E} \left[ \mathrm{Clip} \left( \frac{ \vg_{t} }{ \max \left\{ \sqrt{ \vv_{t-1} } , \epsilon \right\} } , c_t \right) \right] \\
    &\geq \sum_{i=1}^D \frac{ \left( \nabla f \left( \vtheta_{t-1} \right) \right)_i^2 }{{ \max \left\{ \sqrt{ \left( \vv_{t-1} \right)_i } , \epsilon \right\} } } - \frac{1}{ \epsilon^3 c_t^2 } \mathbb{E} \left[ \left( \vg_t \right)_i^2 \right]^2 \\
    &\geq \sum_{i=1}^D \frac{ \left( \nabla f \left( \vtheta_{t-1} \right) \right)_i^2 }{{ \max \left\{ \sqrt{ \left( \vv_{t-1} \right)_i } , \epsilon \right\} } } - \frac{1}{ \epsilon^3 c_t^2 } \mathbb{E} \left[ \left\| \vg_t \right\|^2 \right]^2 \\
    &\geq \sum_{i=1}^D \frac{ \left( \nabla f \left( \vtheta_{t-1} \right) \right)_i^2 }{{ \max \left\{ \sqrt{ \left( \vv_{t-1} \right)_i } , \epsilon \right\} } } - \frac{G^4}{ \epsilon^3 c_t^2 }
\end{align}

\end{proof}
\end{lemma}

\begin{lemma}
\label{lem:bound_exp_v}
For the \adopt{} algorithm, the following holds for $t \geq 0$:
\begin{align}
    \mathbb{E} \left[ \sum_{i=1}^D \left( \vv_t \right)_i \right] \leq G^2.
\end{align}

\begin{proof}
\begin{align}
    \mathbb{E} \left[ \sum_{i=1}^D \left( \vv_t \right)_i \right] 
    &= \mathbb{E} \left[ \beta_2^t \sum_{i=1}^D \left( \vg_0 \right)_i^2 + \left( 1 - \beta_2 \right) \sum_{i=1}^D \sum_{k=1}^t \beta_2^{t-k} \left( \vg_{k-1} \right)_i^2 \right] \\
    &\leq \beta_2^t G^2 + \left( 1 - \beta_2 \right) G^2 \sum_{k=1}^t \beta_2^{t-k} \\
    &= \beta_2^t G^2 + \left( 1 - \beta_2^t \right) G^2 \\
    &= G^2
\end{align}
\end{proof}
\end{lemma}

\begin{lemma}
For the \adopt{} algorithm, the following holds for $t \geq 0$:
\begin{align}
    &\mathbb{E} \left[ \left\| \mathrm{Clip} \left( \frac{ \vg_{t} }{ \max \left\{ \sqrt{ \vv_{t-1} } , \epsilon \right\} } , c_t \right) \right\|^2 \right] \leq \frac{2 G^2}{\epsilon^2} .
\end{align}

\begin{proof}
\begin{align}
    &\mathbb{E} \left[ \left\| \mathrm{Clip} \left( \frac{ \vg_{t} }{ \max \left\{ \sqrt{ \vv_{t-1} } , \epsilon \right\} } , c_t \right) \right\|^2 \right] \\
    &= \mathbb{E} \left[ \sum_{i=1}^D \mathrm{Clip} \left( \frac{ \left( \vg_{t} \right)_i }{ \max \left\{ \sqrt{ \left( \vv_{t-1} \right)_i } , \epsilon \right\} } , c_t \right)^2 \right] \\
    &= \mathbb{E} \left[ \sum_{i=1}^D \delta_{t, i} c_t^2 + \left( 1 - \delta_{t, i} \right) \frac{ \left( \vg_{t} \right)_i^2 }{ \max \left\{ \left( \vv_{t-1} \right)_i , \epsilon^2 \right\} } \right] \\
    &\leq \sum_{i=1}^D c_t^2 \mathbb{E} \left[ \delta_{t, i} \right] + \mathbb{E} \left[ \frac{ \left( \vg_{t} \right)_i^2 }{ \max \left\{ \left( \vv_{t-1} \right)_i , \epsilon^2 \right\} } \right] \\
    &= \sum_{i=1}^D c_t^2 \mathrm{Pr} \left[ \left| \frac{ \left( \vg_{t} \right)_i }{ \max \left\{ \sqrt{ \left( \vv_{t-1} \right)_i } , \epsilon \right\} } \right| > c_t \right] + \mathbb{E} \left[ \frac{ \left( \vg_{t} \right)_i^2 }{ \max \left\{ \left( \vv_{t-1} \right)_i , \epsilon^2 \right\} } \right] \\
    &\leq 2 \cdot \mathbb{E} \left[ \sum_{i=1}^D \frac{ \left( \vg_{t} \right)_i^2 }{ \max \left\{ \left( \vv_{t-1} \right)_i , \epsilon^2 \right\} } \right] \\
    &\leq \frac{2}{\epsilon^2} \cdot \mathbb{E} \left[ \left\| \vg_t \right\|^2 \right] \\
    &\leq \frac{2 G^2}{\epsilon^2} ,
\end{align}
where the second inequality is due to the Chebyshev's inequality.
\end{proof}
\end{lemma}
\begin{lemma}
\label{lem:bound_exp_sq_norm_m}
For the \adopt{} algorithm, the following holds for $0 \leq t \leq T$.
\begin{align}
    \mathbb{E} \left[ \left\| \vm_{t} \right\|^2 \right] \leq \frac{ 2 G^2 }{\epsilon^2}.
\end{align}

\begin{proof}
\begin{align}
    &\mathbb{E} \left[ \left\| \vm_t \right\|^2 \right] \nonumber \\
    &= \mathbb{E} \left[ \left\| \beta_1 \vm_{t-1} + \left( 1 - \beta_1 \right) \mathrm{Clip} \left( \frac{ \vg_{t} }{ \max \left\{ \sqrt{ \vv_{t-1} } , \epsilon \right\} } , c_t \right) \right\|^2 \right] \\
    &= \mathbb{E} \left[ \beta_1^2 \left\| \vm_{t-1} \right\|^2 + \left( 1 - \beta_1 \right)^2 \left\| \mathrm{Clip} \left( \frac{ \vg_{t} }{ \max \left\{ \sqrt{ \vv_{t-1} } , \epsilon \right\} } , c_t \right) \right\|^2 \right] \nonumber \\
    &\quad + \mathbb{E} \left[ 2 \beta_1 \left( 1 - \beta_1 \right) \vm_{t-1}^\top \mathrm{Clip} \left( \frac{ \vg_{t} }{ \max \left\{ \sqrt{ \vv_{t-1} } , \epsilon \right\} } , c_t \right) \right] \label{eq:extended_sq_norm_m}\\
    &\leq \mathbb{E} \left[ \beta_1 \left\| \vm_{t-1} \right\|^2 + \left( 1 - \beta_1 \right) \left\| \mathrm{Clip} \left( \frac{ \vg_{t} }{ \max \left\{ \sqrt{ \vv_{t-1} } , \epsilon \right\} } , c_t \right) \right\|^2 \right] \label{eq:bound_inner_prod_m} \\
    &\leq \mathbb{E} \left[ \beta_1 \left\| \vm_{t-1} \right\|^2 + \frac{ 2 \left( 1 - \beta_1 \right) G^2 }{ \epsilon^2 } \right] \\
    &\leq \frac{ 2 \left( 1 - \beta_1 \right) G^2 }{ \epsilon^2 } \sum_{k=1}^t \beta_1^{t-k} \\
    &\leq \frac{ 2 \left( 1 - \beta_1^t \right) G^2 }{ \epsilon^2 } \\
    &\leq \frac{ 2 G^2 }{\epsilon^2}.
\end{align}

First inequality is derived using the following fact:
\begin{align}
    \forall \lambda>0, \vx, \vy \in \mathbb{R}^d, \vx^\top \vy \leq \frac{\lambda}{2} \left\| \vx \right\|^2+\frac{1}{2 \lambda} \left\| \vy \right\|^2
\end{align}
By setting \(\lambda = \left( 1 - \beta_1 \right) / \beta_1, \vx = \beta_1 \vm_{t-1}, \vy = \left(1 - \beta_1 \right) \mathrm{Clip} \left( \vg_t / \max \left\{ \sqrt{ \vv_{t-1} } , \epsilon \right\} , c_t \right) \), we obtain
\begin{align}
    &2 \beta_1 \left( 1 - \beta_1 \right) \vm_{t-1}^\top \mathrm{Clip} \left( \frac{ \vg_{t} }{ \max \left\{ \sqrt{ \vv_{t-1} } , \epsilon \right\} } , c_t \right) \\
    &\leq \beta_1 \left( 1 - \beta_1 \right) \left( \left\| \vm_{t-1} \right\|^2 + \left\| \mathrm{Clip} \left( \frac{ \vg_{t} }{ \max \left\{ \sqrt{ \vv_{t-1} } , \epsilon \right\} } , c_t \right) \right\|^2 \right)
\end{align}
Injecting it into Eq. (\ref{eq:extended_sq_norm_m}), we obtain Eq. (\ref{eq:bound_inner_prod_m}).

\end{proof}
\end{lemma}

\begin{lemma}
\label{lem:bound_exp_norm_m}
For the \adopt{} algorithm, the following holds for $t \geq 0$.
\begin{align}
    \mathbb{E} \left[ \left\| \vm_t \right\| \right] \leq \frac{ \sqrt{2} G }{\epsilon}
\end{align}

\begin{proof}
\begin{align}
    \mathbb{E} \left[ \left\| \vm_t \right\| \right]
    &= \mathbb{E} \left[ \left\| \left( 1 - \beta_1 \right) \sum_{k=1}^t \beta_1^{t-k} \mathrm{Clip} \left( \frac{ \vg_{k} }{ \max \left\{ \sqrt{ \vv_{k-1} } , \epsilon \right\} } , c_t \right) \right\| \right] \\
    &\leq \left( 1 - \beta_1 \right) \sum_{k=1}^t \beta_1^{t-k} \mathbb{E} \left[ \left\| \mathrm{Clip} \left( \frac{ \vg_{k} }{ \max \left\{ \sqrt{ \vv_{k-1} } , \epsilon \right\} } , c_t \right) \right\| \right] \\
    &\leq \left( 1 - \beta_1 \right) \sum_{k=1}^t \beta_1^{t-k} \sqrt{ \mathbb{E} \left[ \left\| \mathrm{Clip} \left( \frac{ \vg_{k} }{ \max \left\{ \sqrt{ \vv_{k-1} } , \epsilon \right\} } , c_t \right) \right\|^2 \right] } \\
    &\leq \frac{ \sqrt{2} \left( 1 - \beta_1 \right) G }{\epsilon} \sum_{k=1}^t \beta_1^{t-k} \\
    &= \frac{ \sqrt{2} \left( 1 - \beta_1^t \right) G }{ \epsilon } \\
    &\leq \frac{ \sqrt{2} G }{ \epsilon }.
\end{align}
\end{proof}
\end{lemma}

\begin{lemma}
\label{lem:bound_exp_norm_prod}
For the \adopt{} algorithm, the following holds for $t \geq 1$:
\begin{align}
    &\mathbb{E} \left[ \left\| \vphi_{t-1} - \vtheta_{t-1} \right\| \left\| \vphi_{t} - \vphi_{t-1} \right\| \right] \nonumber\\
    &\leq \frac{ 2 \alpha_{t-1} \left( \alpha_{t-1} - \alpha_{t} \right) \beta_1^2 \left( 1 - \beta_1^{t-1} \right) G^2 }{ \epsilon^2 \left( 1 - \beta_1 \right)^2 } + \frac{ 2 \alpha_{t} \alpha_{t-1} \beta_1 \sqrt{ 1 - \beta_1^{t-1} } G^2 }{ \epsilon^2 \left( 1 - \beta_1 \right) }.
\end{align}

\begin{proof}

\begin{align}
    &\left\| \vphi_{t-1} - \vtheta_{t-1} \right\| \left\| \vphi_{t} - \vphi_{t-1} \right\| \nonumber \\
    &= \left\| - \frac{ \alpha_{t-1} \beta_1 }{ 1 - \beta_1 } \vm_{t-1} \right\| \left\| \frac{ \left( \alpha_{t-1} - \alpha_{t} \right) \beta_1 }{ 1 - \beta_1 } \vm_{t-1} - \alpha_{t} \mathrm{Clip} \left( \frac{ \vg_{t} }{ \max \left\{ \sqrt{ \vv_{t-1} } , \epsilon \right\} } , c_t \right) \right\| \\
    &\leq \frac{ \alpha_{t-1} \beta_1 }{ 1 - \beta_1 } \left\| \vm_{t-1} \right\| \left( \frac{ \left( \alpha_{t-1} - \alpha_{t} \right) \beta_1 }{ 1 - \beta_1 } \left\| \vm_{t-1} \right\| + \alpha_{t} \left\| \mathrm{Clip} \left( \frac{ \vg_{t} }{ \max \left\{ \sqrt{ \vv_{t-1} } , \epsilon \right\} } , c_t \right) \right\| \right) \\
    &\leq \frac{ \alpha_{t-1} \left( \alpha_{t-1} - \alpha_{t} \right) \beta_1^2 }{ \left( 1 - \beta_1 \right)^2 } \left\| \vm_{t-1} \right\|^2 +  \frac{ \alpha_{t} \alpha_{t-1} \beta_1  }{ 1 - \beta_1 } \left\| \vm_{t-1} \right\| \left\| \mathrm{Clip} \left( \frac{ \vg_{t} }{ \max \left\{ \sqrt{ \vv_{t-1} } , \epsilon \right\} } , c_t \right) \right\|.
\end{align}
Taking the expectation yields:
\begin{align}
    &\mathbb{E} \left[ \left\| \vphi_{t-1} - \vtheta_{t-1} \right\| \left\| \vphi_t - \vphi_{t-1} \right\| \right] \nonumber \\
    &\leq \frac{ \alpha_{t-1} \left( \alpha_{t-1} - \alpha_{t} \right) \beta_1^2 }{ \left( 1 - \beta_1 \right)^2 } \mathbb{E} \left[ \left\| \vm_{t-1} \right\|^2 \right] + \frac{ \alpha_{t} \alpha_{t-1} \beta_1 }{ 1 - \beta_1 } \mathbb{E} \left[ \left\| \vm_{t-1} \right\| \left\| \mathrm{Clip} \left( \frac{ \vg_{t} }{ \max \left\{ \sqrt{ \vv_{t-1} } , \epsilon \right\} } , c_t \right) \right\| \right] \\
    &\leq \frac{ \alpha_{t-1} \left( \alpha_{t-1} - \alpha_{t} \right) \beta_1^2 }{ \left( 1 - \beta_1 \right)^2 } \mathbb{E} \left[ \left\| \vm_{t-1} \right\|^2 \right] + \frac{ \sqrt{2} \alpha_{t} \alpha_{t-1} \beta_1 G }{ \left( 1 - \beta_1 \right) \epsilon } \mathbb{E} \left[ \left\| \vm_{t-1} \right\| \right] \\
    &\leq \frac{ 2 \alpha_{t-1} \left( \alpha_{t-1} - \alpha_{t} \right) \beta_1^2 \left( 1 - \beta_1^{t-1} \right) G^2 }{ \epsilon^2 \left( 1 - \beta_1 \right)^2 } + \frac{ 2 \alpha_{t} \alpha_{t-1} \beta_1 \sqrt{ 1 - \beta_1^{t-1} } G^2 }{ \left( 1 - \beta_1 \right) \epsilon^2 }.
\end{align}
\end{proof}
\end{lemma}

\begin{lemma}
\label{lem:bound_exp_sq_norm}
For the \adopt{} algorithm, the following holds for $t \geq 1$:
\begin{align}
    &\mathbb{E} \left[ \left\| \vphi_{t} - \vphi_{t-1} \right\|^2 \right] \nonumber \\
    &\leq \frac{ 2 \left( \alpha_{t-1} - \alpha_{t} \right)^2 \beta_1^2 \left( 1 - \beta_1^{t-1} \right) G^2 }{ \left( 1 - \beta_1 \right)^2 \epsilon^2 } + \frac{2 \alpha_{t}^2 G^2 }{ \epsilon^2 } + \frac{ 2 \alpha_t \left( \alpha_{t-1} - \alpha_{t} \right) \beta_1 \sqrt{ 1 - \beta_1^{t-1} }G^2 }{ \left( 1 - \beta_1 \right) \epsilon^2 }.
\end{align}

\begin{proof}

\begin{align}
    &\left\| \vphi_{t} - \vphi_{t-1} \right\|^2 \nonumber \\
    &= \left\| \frac{ \left( \alpha_{t-1} - \alpha_{t} \right) \beta_1 }{ 1 - \beta_1 } \vm_{t-1} - \alpha_{t} \mathrm{Clip} \left( \frac{ \vg_{t} }{ \max \left\{ \sqrt{ \vv_{t-1} } , \epsilon \right\} } , c_t \right) \right\|^2 \\
    &= \frac{ \left( \alpha_{t-1} - \alpha_{t} \right)^2 \beta_1^2 }{ \left( 1 - \beta_1 \right)^2 } \left\| \vm_{t-1} \right\|^2 + \alpha_{t}^2 \left\| \mathrm{Clip} \left( \frac{ \vg_{t} }{ \max \left\{ \sqrt{ \vv_{t-1} } , \epsilon \right\} } , c_t \right) \right\|^2 
    \nonumber \\
    &\quad - \frac{ \alpha_t \left( \alpha_{t-1} - \alpha_{t} \right) \beta_1 }{ 1 - \beta_1 } \vm_{t-1}^\top \mathrm{Clip} \left( \frac{ \vg_{t} }{ \max \left\{ \sqrt{ \vv_{t-1} } , \epsilon \right\} } , c_t \right) \\
    &\leq \frac{ \left( \alpha_{t-1} - \alpha_{t} \right)^2 \beta_1^2 }{ \left( 1 - \beta_1 \right)^2 } \left\| \vm_{t-1} \right\|^2 + \alpha_{t}^2 \left\| \mathrm{Clip} \left( \frac{ \vg_{t} }{ \max \left\{ \sqrt{ \vv_{t-1} } , \epsilon \right\} } , c_t \right) \right\|^2
    \nonumber\\
    &\quad + \frac{ \alpha_t \left( \alpha_{t-1} - \alpha_{t} \right) \beta_1 }{ 1 - \beta_1 } \left\| \vm_{t-1} \right\| \left\| \mathrm{Clip} \left( \frac{ \vg_{t} }{ \max \left\{ \sqrt{ \vv_{t-1} } , \epsilon \right\} } \right) \right\| \\
\end{align}
Taking the expectation yields:
\begin{align}
    &\mathbb{E} \left[ \left\| \vphi_{t} - \vphi_{t-1} \right\|^2 \right] \nonumber \\
    &\leq \frac{ 2 \left( \alpha_{t-1} - \alpha_{t} \right)^2 \beta_1^2 \left( 1 - \beta_1^{t-1} \right) G^2 }{ \left( 1 - \beta_1 \right)^2 \epsilon^2 } + \frac{2 \alpha_{t}^2 G^2 }{ \epsilon^2 } + \frac{ 2 \alpha_t \left( \alpha_{t-1} - \alpha_{t} \right) \beta_1 \sqrt{ 1 - \beta_1^{t-1} }G^2 }{ \left( 1 - \beta_1 \right) \epsilon^2 }.
\end{align}

\end{proof}

\end{lemma}

%% file: G_setup.tex
\begin{figure}[tb]
    \centering
    \begin{minipage}{0.48\linewidth}
        \includegraphics[width=\linewidth]{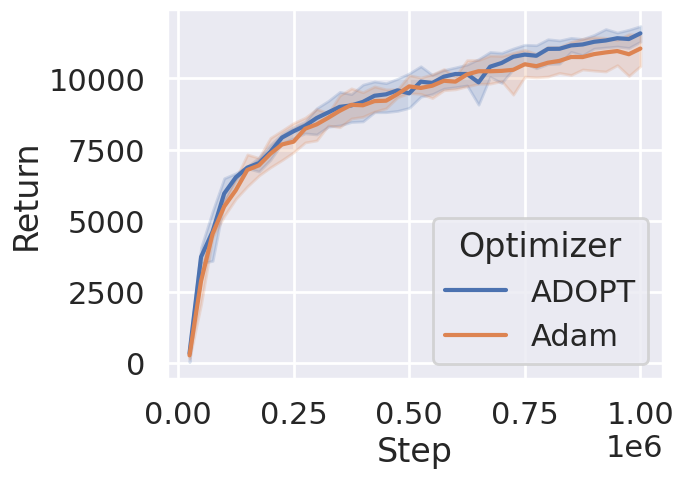}
    \end{minipage}
    \begin{minipage}{0.48\linewidth}
        \includegraphics[width=\linewidth]{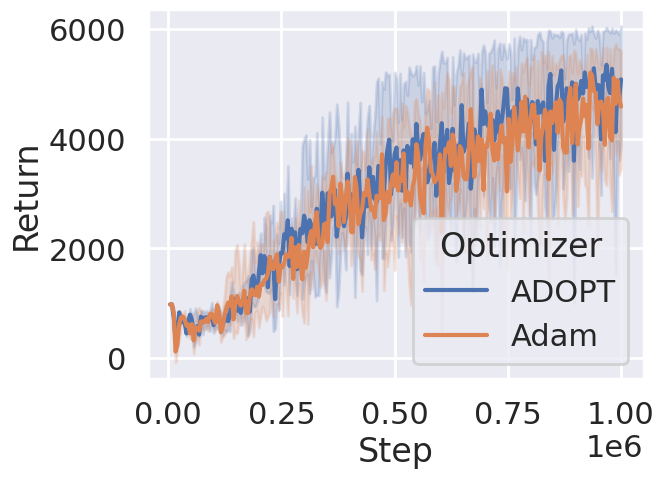}
    \end{minipage}\\
    \begin{minipage}{0.48\linewidth}
        \centering
        {\bf HalfCheetah-v4}
    \end{minipage}
    \begin{minipage}{0.48\linewidth}
        \centering
        {\bf Ant-v4}
    \end{minipage}
    \caption{Performance comparison between \adam{} and \adopt{} in reinforcement learning.}
    \label{fig:deep_rl}
\end{figure}

\begin{figure}
    \centering
    \includegraphics[width=\hsize]{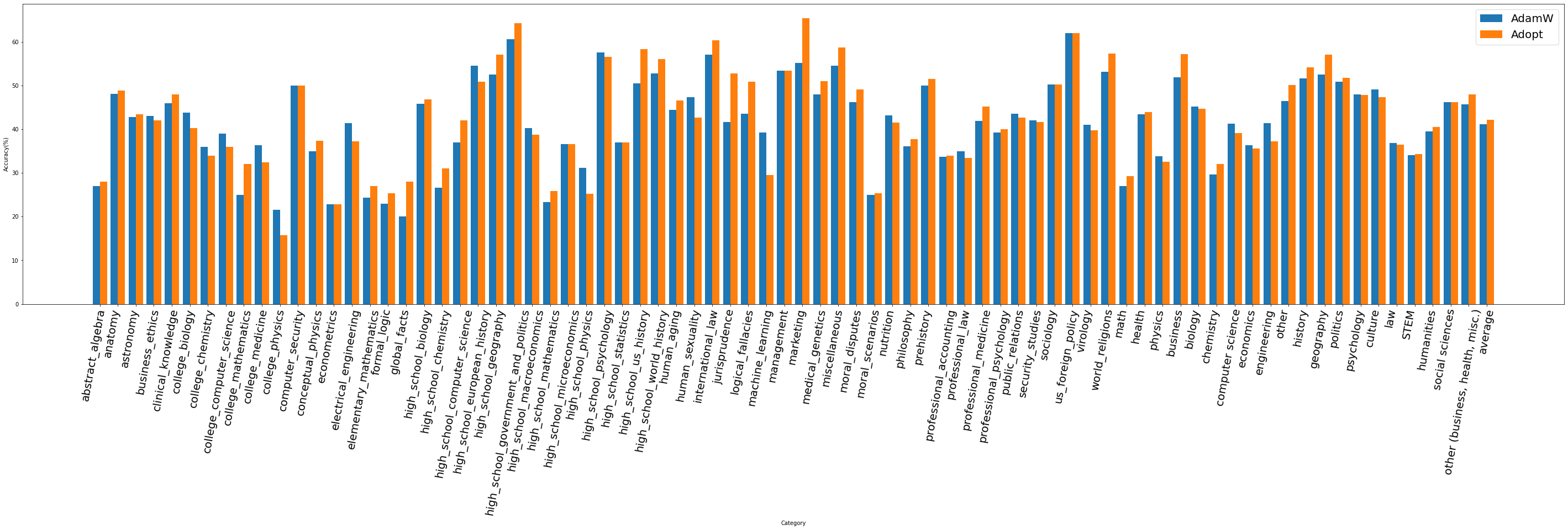}
    \caption{Comparison of MMLU scores for LLaMA-7B finetuned via instruction following using \adamw{} and \adopt{}.}
    \label{fig:alpaca}
\end{figure}

\section{Additional Experiments}
\label{sec:add_exp}
{\bf Deep reinforcement learning:}

We train reinforcement learning (RL) agents using the soft actor crtitic algorithm~\citep{pmlr-v80-haarnoja18b} with \adopt{} for the optimizer.
As a benchmark, we use a continuous control tasks of HalfCheetah-v4 on MuJoCo simulator~\citep{6386109}.
For comparison to \adopt{}, \adam{} is used as a baseline optimizer. 
We follow the hyperparameter settings recommended by Stable-Baselines3~\citep{raffin2021stable}, and just change the choice of an optimizer.
The result is shown in Figure \ref{fig:deep_rl}.
The error bars indicate 95\% confidence intervals of three trials.
We observe slight performance improvement by using \adopt{} instead of \adam{}.

\section{Details of Experimental Setups}
\label{sec:setup}

\subsection{Code}
Our implementation for the experiment is available at \url{https://github.com/iShohei220/adopt}.

\subsection{Total amount of compute}
We run our experiments mainly on cloud GPU instances with $8 \times$ A100. 
It took approximately 320 hours for our experiments in total.

\subsection{License of Assets}
{\bf Datasets:} 
The MNIST database is downloaded from \url{http://yann.lecun.com/exdb/mnist}, which is license-free.
The terms of access for the ImageNet database is provided at \url{https://www.image-net.org/download}. 
The dataset of Stanford Alpaca is CC BY NC 4.0 (allowing only non-commercial use).

{\bf Pretrained models:}
The pretrained model of LLaMA is provided under GNU General Public License v3.0.

{\bf Simulator:}
MuJoCo is provided under Apache License 2.0.

{\bf Code:}
Our implementation of ImageNet classification is based on the Torchvision's official training recipe provided at \url{https://github.com/UiPath/torchvision/tree/master/references/classification}.
Torchvision is provided under BSD 3-Clause License.
We use the official implementation of NVAE provided at \url{https://github.com/NVlabs/NVAE}, whose license is described at \url{https://github.com/NVlabs/NVAE/blob/master/LICENSE}.

%% file: neurips_2024.bbl
\begin{thebibliography}{42}
\providecommand{\natexlab}[1]{#1}
\providecommand{\url}[1]{\texttt{#1}}
\expandafter\ifx\csname urlstyle\endcsname\relax
  \providecommand{\doi}[1]{doi: #1}\else
  \providecommand{\doi}{doi: \begingroup \urlstyle{rm}\Url}\fi

\bibitem[Kingma and Ba(2014)]{kingma2014adam}
Diederik~P Kingma and Jimmy Ba.
\newblock Adam: A method for stochastic optimization.
\newblock \emph{arXiv preprint arXiv:1412.6980}, 2014.

\bibitem[Reddi et~al.(2018)Reddi, Kale, and Kumar]{j.2018on}
Sashank~J. Reddi, Satyen Kale, and Sanjiv Kumar.
\newblock On the convergence of adam and beyond.
\newblock In \emph{International Conference on Learning Representations}, 2018.
\newblock URL \url{https://openreview.net/forum?id=ryQu7f-RZ}.

\bibitem[Zou et~al.(2019)Zou, Shen, Jie, Zhang, and Liu]{zou2019sufficient}
Fangyu Zou, Li~Shen, Zequn Jie, Weizhong Zhang, and Wei Liu.
\newblock A sufficient condition for convergences of adam and rmsprop.
\newblock In \emph{Proceedings of the IEEE/CVF conference on computer vision and pattern recognition}, pages 11127--11135, 2019.

\bibitem[Chen et~al.(2019)Chen, Liu, Sun, and Hong]{chen2018on}
Xiangyi Chen, Sijia Liu, Ruoyu Sun, and Mingyi Hong.
\newblock On the convergence of a class of adam-type algorithms for non-convex optimization.
\newblock In \emph{International Conference on Learning Representations}, 2019.
\newblock URL \url{https://openreview.net/forum?id=H1x-x309tm}.

\bibitem[Zhou et~al.(2018)Zhou, Chen, Cao, Tang, Yang, and Gu]{zhou2018convergence}
Dongruo Zhou, Jinghui Chen, Yuan Cao, Yiqi Tang, Ziyan Yang, and Quanquan Gu.
\newblock On the convergence of adaptive gradient methods for nonconvex optimization.
\newblock \emph{arXiv preprint arXiv:1808.05671}, 2018.

\bibitem[Ghadimi and Lan(2013)]{ghadimi2013stochastic}
Saeed Ghadimi and Guanghui Lan.
\newblock Stochastic first-and zeroth-order methods for nonconvex stochastic programming.
\newblock \emph{SIAM Journal on Optimization}, 23\penalty0 (4):\penalty0 2341--2368, 2013.

\bibitem[Bertsekas and Tsitsiklis(2000)]{bertsekas2000gradient}
Dimitri~P Bertsekas and John~N Tsitsiklis.
\newblock Gradient convergence in gradient methods with errors.
\newblock \emph{SIAM Journal on Optimization}, 10\penalty0 (3):\penalty0 627--642, 2000.

\bibitem[Khaled and Richt{\'a}rik(2023)]{khaled2022better}
Ahmed Khaled and Peter Richt{\'a}rik.
\newblock Better theory for {SGD} in the nonconvex world.
\newblock \emph{Transactions on Machine Learning Research}, 2023.
\newblock ISSN 2835-8856.
\newblock URL \url{https://openreview.net/forum?id=AU4qHN2VkS}.
\newblock Survey Certification.

\bibitem[Kingma and Welling(2014)]{Kingma2014}
Diederik~P. Kingma and Max Welling.
\newblock {Auto-Encoding Variational Bayes}.
\newblock In \emph{2nd International Conference on Learning Representations, {ICLR} 2014, Banff, AB, Canada, April 14-16, 2014, Conference Track Proceedings}, 2014.

\bibitem[Ho et~al.(2020)Ho, Jain, and Abbeel]{ho2020denoising}
Jonathan Ho, Ajay Jain, and Pieter Abbeel.
\newblock Denoising diffusion probabilistic models.
\newblock \emph{Advances in Neural Information Processing Systems}, 33:\penalty0 6840--6851, 2020.

\bibitem[Song et~al.(2021)Song, Sohl-Dickstein, Kingma, Kumar, Ermon, and Poole]{song2021scorebased}
Yang Song, Jascha Sohl-Dickstein, Diederik~P Kingma, Abhishek Kumar, Stefano Ermon, and Ben Poole.
\newblock Score-based generative modeling through stochastic differential equations.
\newblock In \emph{International Conference on Learning Representations}, 2021.
\newblock URL \url{https://openreview.net/forum?id=PxTIG12RRHS}.

\bibitem[Zhou et~al.(2019)Zhou, Zhang, Lu, Wang, Zhang, and Yu]{zhou2018adashift}
Zhiming Zhou, Qingru Zhang, Guansong Lu, Hongwei Wang, Weinan Zhang, and Yong Yu.
\newblock Adashift: Decorrelation and convergence of adaptive learning rate methods.
\newblock In \emph{International Conference on Learning Representations}, 2019.
\newblock URL \url{https://openreview.net/forum?id=HkgTkhRcKQ}.

\bibitem[Shi et~al.(2020)Shi, Li, Hong, and Sun]{shi2020rmsprop}
Naichen Shi, Dawei Li, Mingyi Hong, and Ruoyu Sun.
\newblock Rmsprop converges with proper hyper-parameter.
\newblock In \emph{International Conference on Learning Representations}, 2020.

\bibitem[Zhang et~al.(2022)Zhang, Chen, Shi, Sun, and Luo]{zhang2022adam}
Yushun Zhang, Congliang Chen, Naichen Shi, Ruoyu Sun, and Zhi-Quan Luo.
\newblock Adam can converge without any modification on update rules.
\newblock In Alice~H. Oh, Alekh Agarwal, Danielle Belgrave, and Kyunghyun Cho, editors, \emph{Advances in Neural Information Processing Systems}, 2022.
\newblock URL \url{https://openreview.net/forum?id=l5UNyaHqFdO}.

\bibitem[Wang et~al.(2022)Wang, Zhang, Zhang, Meng, Ma, Liu, and Chen]{wang2022provable}
Bohan Wang, Yushun Zhang, Huishuai Zhang, Qi~Meng, Zhi-Ming Ma, Tie-Yan Liu, and Wei Chen.
\newblock Provable adaptivity in adam.
\newblock \emph{arXiv preprint arXiv:2208.09900}, 2022.

\bibitem[Li et~al.(2023)Li, Jadbabaie, and Rakhlin]{li2023convergence}
Haochuan Li, Ali Jadbabaie, and Alexander Rakhlin.
\newblock Convergence of adam under relaxed assumptions.
\newblock \emph{arXiv preprint arXiv:2304.13972}, 2023.

\bibitem[Wang et~al.(2023)Wang, Fu, Zhang, Zheng, and Chen]{wang2023closing}
Bohan Wang, Jingwen Fu, Huishuai Zhang, Nanning Zheng, and Wei Chen.
\newblock Closing the gap between the upper bound and lower bound of adam's iteration complexity.
\newblock In \emph{Thirty-seventh Conference on Neural Information Processing Systems}, 2023.

\bibitem[He et~al.(2016)He, Zhang, Ren, and Sun]{He_2016_CVPR}
Kaiming He, Xiangyu Zhang, Shaoqing Ren, and Jian Sun.
\newblock Deep residual learning for image recognition.
\newblock In \emph{Proceedings of the IEEE Conference on Computer Vision and Pattern Recognition (CVPR)}, June 2016.

\bibitem[Liu et~al.(2021)Liu, Lin, Cao, Hu, Wei, Zhang, Lin, and Guo]{liu2021swin}
Ze~Liu, Yutong Lin, Yue Cao, Han Hu, Yixuan Wei, Zheng Zhang, Stephen Lin, and Baining Guo.
\newblock Swin transformer: Hierarchical vision transformer using shifted windows.
\newblock In \emph{Proceedings of the IEEE/CVF international conference on computer vision}, pages 10012--10022, 2021.

\bibitem[Blair(1985)]{blair1985problem}
Charles Blair.
\newblock Problem complexity and method efficiency in optimization (as nemirovsky and db yudin).
\newblock \emph{Siam Review}, 27\penalty0 (2):\penalty0 264, 1985.

\bibitem[Vavasis(1995)]{vavasis1995complexity}
Stephen~A Vavasis.
\newblock Complexity issues in global optimization: a survey.
\newblock \emph{Handbook of global optimization}, pages 27--41, 1995.

\bibitem[D{\'e}fossez et~al.(2022)D{\'e}fossez, Bottou, Bach, and Usunier]{defossez2022a}
Alexandre D{\'e}fossez, Leon Bottou, Francis Bach, and Nicolas Usunier.
\newblock A simple convergence proof of adam and adagrad.
\newblock \emph{Transactions on Machine Learning Research}, 2022.
\newblock ISSN 2835-8856.
\newblock URL \url{https://openreview.net/forum?id=ZPQhzTSWA7}.

\bibitem[Li and Orabona(2019)]{li2019convergence}
Xiaoyu Li and Francesco Orabona.
\newblock On the convergence of stochastic gradient descent with adaptive stepsizes.
\newblock In \emph{The 22nd international conference on artificial intelligence and statistics}, pages 983--992. PMLR, 2019.

\bibitem[Ward et~al.(2020)Ward, Wu, and Bottou]{ward2020adagrad}
Rachel Ward, Xiaoxia Wu, and Leon Bottou.
\newblock Adagrad stepsizes: Sharp convergence over nonconvex landscapes.
\newblock \emph{The Journal of Machine Learning Research}, 21\penalty0 (1):\penalty0 9047--9076, 2020.

\bibitem[Zou et~al.(2018)Zou, Shen, Jie, Sun, and Liu]{zou2018weighted}
Fangyu Zou, Li~Shen, Zequn Jie, Ju~Sun, and Wei Liu.
\newblock Weighted adagrad with unified momentum.
\newblock \emph{arXiv preprint arXiv:1808.03408}, 2018.

\bibitem[Drori and Shamir(2020)]{drori2020complexity}
Yoel Drori and Ohad Shamir.
\newblock The complexity of finding stationary points with stochastic gradient descent.
\newblock In \emph{International Conference on Machine Learning}, pages 2658--2667. PMLR, 2020.

\bibitem[Wang et~al.(2021)Wang, Magn{\'u}sson, and Johansson]{wang2021convergence}
Xiaoyu Wang, Sindri Magn{\'u}sson, and Mikael Johansson.
\newblock On the convergence of step decay step-size for stochastic optimization.
\newblock \emph{Advances in Neural Information Processing Systems}, 34:\penalty0 14226--14238, 2021.

\bibitem[Duchi et~al.(2011)Duchi, Hazan, and Singer]{JMLR:v12:duchi11a}
John Duchi, Elad Hazan, and Yoram Singer.
\newblock Adaptive subgradient methods for online learning and stochastic optimization.
\newblock \emph{Journal of Machine Learning Research}, 12\penalty0 (61):\penalty0 2121--2159, 2011.
\newblock URL \url{http://jmlr.org/papers/v12/duchi11a.html}.

\bibitem[Hinton et~al.(2012)Hinton, Srivastava, and Swersky]{rmsprop}
Geoffrey Hinton, Nitish Srivastava, and Kevin Swersky.
\newblock Lecture 6e rmsprop: Divide the gradient by a running average of its recent magnitude, 2012.
\newblock URL \url{https://www.cs.toronto.edu/~tijmen/csc321/slides/lecture_slides_lec6.pdf}.

\bibitem[Vaswani et~al.(2017)Vaswani, Shazeer, Parmar, Uszkoreit, Jones, Gomez, Kaiser, and Polosukhin]{NIPS2017_3f5ee243}
Ashish Vaswani, Noam Shazeer, Niki Parmar, Jakob Uszkoreit, Llion Jones, Aidan~N Gomez, \L~ukasz Kaiser, and Illia Polosukhin.
\newblock Attention is all you need.
\newblock In I.~Guyon, U.~Von Luxburg, S.~Bengio, H.~Wallach, R.~Fergus, S.~Vishwanathan, and R.~Garnett, editors, \emph{Advances in Neural Information Processing Systems}, volume~30. Curran Associates, Inc., 2017.
\newblock URL \url{https://proceedings.neurips.cc/paper_files/paper/2017/file/3f5ee243547dee91fbd053c1c4a845aa-Paper.pdf}.

\bibitem[Paszke et~al.(2019{\natexlab{a}})Paszke, Gross, Massa, Lerer, Bradbury, Chanan, Killeen, Lin, Gimelshein, Antiga, et~al.]{paszke2019pytorch}
Adam Paszke, Sam Gross, Francisco Massa, Adam Lerer, James Bradbury, Gregory Chanan, Trevor Killeen, Zeming Lin, Natalia Gimelshein, Luca Antiga, et~al.
\newblock Pytorch: An imperative style, high-performance deep learning library.
\newblock \emph{Advances in neural information processing systems}, 32, 2019{\natexlab{a}}.

\bibitem[Loshchilov and Hutter(2019)]{loshchilov2018decoupled}
Ilya Loshchilov and Frank Hutter.
\newblock Decoupled weight decay regularization.
\newblock In \emph{International Conference on Learning Representations}, 2019.
\newblock URL \url{https://openreview.net/forum?id=Bkg6RiCqY7}.

\bibitem[Vahdat and Kautz(2020)]{vahdat2020nvae}
Arash Vahdat and Jan Kautz.
\newblock Nvae: A deep hierarchical variational autoencoder.
\newblock \emph{Advances in neural information processing systems}, 33:\penalty0 19667--19679, 2020.

\bibitem[Radford et~al.(2019)Radford, Wu, Child, Luan, Amodei, Sutskever, et~al.]{radford2019language}
Alec Radford, Jeffrey Wu, Rewon Child, David Luan, Dario Amodei, Ilya Sutskever, et~al.
\newblock Language models are unsupervised multitask learners.
\newblock \emph{OpenAI blog}, 1\penalty0 (8):\penalty0 9, 2019.

\bibitem[Karpathy(2022)]{nanogpt}
Andrej Karpathy.
\newblock \text{NanoGPT}.
\newblock \url{https://github.com/karpathy/nanoGPT}, 2022.

\bibitem[Gokaslan and Cohen(2019)]{Gokaslan2019OpenWeb}
Aaron Gokaslan and Vanya Cohen.
\newblock Openwebtext corpus.
\newblock \url{http://Skylion007.github.io/OpenWebTextCorpus}, 2019.

\bibitem[Hendrycks et~al.(2021)Hendrycks, Burns, Basart, Zou, Mazeika, Song, and Steinhardt]{hendrycks2021measuring}
Dan Hendrycks, Collin Burns, Steven Basart, Andy Zou, Mantas Mazeika, Dawn Song, and Jacob Steinhardt.
\newblock Measuring massive multitask language understanding.
\newblock In \emph{International Conference on Learning Representations}, 2021.
\newblock URL \url{https://openreview.net/forum?id=d7KBjmI3GmQ}.

\bibitem[Haochen and Sra(2019)]{haochen2019random}
Jeff Haochen and Suvrit Sra.
\newblock Random shuffling beats sgd after finite epochs.
\newblock In \emph{International Conference on Machine Learning}, pages 2624--2633. PMLR, 2019.

\bibitem[Paszke et~al.(2019{\natexlab{b}})Paszke, Gross, Massa, Lerer, Bradbury, Chanan, Killeen, Lin, Gimelshein, Antiga, Desmaison, Kopf, Yang, DeVito, Raison, Tejani, Chilamkurthy, Steiner, Fang, Bai, and Chintala]{NEURIPS2019_9015}
Adam Paszke, Sam Gross, Francisco Massa, Adam Lerer, James Bradbury, Gregory Chanan, Trevor Killeen, Zeming Lin, Natalia Gimelshein, Luca Antiga, Alban Desmaison, Andreas Kopf, Edward Yang, Zachary DeVito, Martin Raison, Alykhan Tejani, Sasank Chilamkurthy, Benoit Steiner, Lu~Fang, Junjie Bai, and Soumith Chintala.
\newblock Pytorch: An imperative style, high-performance deep learning library.
\newblock In \emph{Advances in Neural Information Processing Systems 32}, pages 8024--8035. Curran Associates, Inc., 2019{\natexlab{b}}.
\newblock URL \url{http://papers.neurips.cc/paper/9015-pytorch-an-imperative-style-high-performance-deep-learning-library.pdf}.

\bibitem[Haarnoja et~al.(2018)Haarnoja, Zhou, Abbeel, and Levine]{pmlr-v80-haarnoja18b}
Tuomas Haarnoja, Aurick Zhou, Pieter Abbeel, and Sergey Levine.
\newblock Soft actor-critic: Off-policy maximum entropy deep reinforcement learning with a stochastic actor.
\newblock In Jennifer Dy and Andreas Krause, editors, \emph{Proceedings of the 35th International Conference on Machine Learning}, volume~80 of \emph{Proceedings of Machine Learning Research}, pages 1861--1870. PMLR, 10--15 Jul 2018.
\newblock URL \url{https://proceedings.mlr.press/v80/haarnoja18b.html}.

\bibitem[Todorov et~al.(2012)Todorov, Erez, and Tassa]{6386109}
Emanuel Todorov, Tom Erez, and Yuval Tassa.
\newblock Mujoco: A physics engine for model-based control.
\newblock In \emph{2012 IEEE/RSJ International Conference on Intelligent Robots and Systems}, pages 5026--5033, 2012.
\newblock \doi{10.1109/IROS.2012.6386109}.

\bibitem[Raffin et~al.(2021)Raffin, Hill, Gleave, Kanervisto, Ernestus, and Dormann]{raffin2021stable}
Antonin Raffin, Ashley Hill, Adam Gleave, Anssi Kanervisto, Maximilian Ernestus, and Noah Dormann.
\newblock Stable-baselines3: Reliable reinforcement learning implementations.
\newblock \emph{The Journal of Machine Learning Research}, 22\penalty0 (1):\penalty0 12348--12355, 2021.

\end{thebibliography}
